%% file: main_arxiv_20260619.tex
\documentclass{article}

\input{llm_history_allen/llm_history_preamble}

\usepackage{hyperref}

\usepackage[accepted]{icml2026}

\input{math_commands.tex}

\usepackage{algorithm}
\usepackage{algorithmic}
\usepackage{amsmath}
\usepackage{amssymb}
\usepackage{mathtools}
\usepackage{amsthm}

\usepackage[capitalize]{cleveref}
\crefformat{equation}{Eq.~#2#1#3}
\crefname{section}{Sec.}{Secs.}
\Crefname{section}{Sec.}{Secs.}

\usepackage{xcolor, soul}         %
\sethlcolor{pink}
\usepackage{graphicx}
\usepackage{wrapfig}
\usepackage{subcaption}
\usepackage{enumitem}
\usepackage{titletoc} %
\usepackage{float}
\usepackage{booktabs}
\usepackage{minted}
\usepackage{comment}
\usepackage{array}

\usepackage{hhline}
\usepackage{tabularx,array}
\usepackage{enumitem}
\usepackage{etoc}
\usepackage{verbatim}

\setlist[enumerate]{itemsep=1pt, topsep=1pt}

\setlist[itemize]{itemsep=1pt, topsep=1pt}

\theoremstyle{plain}
\newtheorem{theorem}{Theorem}[section]



\newcolumntype{Y}{>{\centering\arraybackslash}X}
\icmltitlerunning{ModelSMC -- A Probabilistic Framework for LLM-Based Model Discovery}

\begin{document}

\twocolumn[
  \icmltitle{A Probabilistic Framework for LLM-Based Model Discovery}

  \icmlsetsymbol{equal}{*}

  \begin{icmlauthorlist}
    \icmlauthor{Stefan Wahl}{mls,aicenter}
    \icmlauthor{Raphaela Schenk}{bi}
    \icmlauthor{Ali Farnoud}{bi}
    \icmlauthor{Jakob H. Macke}{mls,aicenter,mpiis}
    \icmlauthor{Daniel Gedon}{mls,aicenter}
  \end{icmlauthorlist}

  \icmlaffiliation{mls}{Machine Learning in Science, University of Tübingen, Tübingen, Germany}
  \icmlaffiliation{aicenter}{Tübingen AI Center, Tübingen, Germany}
  \icmlaffiliation{mpiis}{Department Empirical Inference, Max Planck Institute for Intelligent Systems, Tübingen, Germany}
  \icmlaffiliation{bi}{Boehringer Ingelheim, Biberach, Germany}

  \icmlcorrespondingauthor{}{\{firstname.lastname\}@uni-tuebingen.de}

  \icmlkeywords{LLM, model discovery, SMC, Bayesian inference, SBI, scientific discovery} 

  \vskip 0.3in
]

\printAffiliationsAndNotice{}%

\begin{abstract}
Automated methods for discovering mechanistic simulator models from observational data offer a promising path toward accelerating scientific progress. 
Such methods often take the form of agentic-style iterative workflows that repeatedly propose and revise candidate models by imitating human discovery processes. However, existing LLM-based approaches typically implement such workflows via hand-crafted heuristic procedures, without an explicit probabilistic formulation.
We recast model discovery as probabilistic inference, i.e., as sampling from an unknown distribution over mechanistic models capable of explaining the data. This perspective provides a unified way to reason about model proposal, refinement, and selection within a single inference framework.
As a concrete instantiation of this view, we introduce ModelSMC, an algorithm based on Sequential Monte Carlo sampling. ModelSMC represents candidate models as particles which are iteratively proposed and refined by an LLM, and weighted using likelihood-based criteria.
Experiments on real-world scientific systems illustrate that this formulation discovers models with interpretable mechanisms and improves posterior predictive checks. More broadly, this perspective provides a probabilistic lens for understanding and developing LLM-based approaches to model discovery.
\end{abstract}

\section{Introduction} 

Scientific discoveries---from the expanding universe \citep{hubble1929relation} to gravitational waves \citep{abbott2016observation} and neuronal excitability \citep{hodgkin1952}---often rely on formal models that integrate theory with data to enable interpretation and prediction.
As scientific problems grow in experimental and conceptual complexity, computational simulators that encode explicit mechanistic hypotheses have become central tools for discovery %
\citep{frigg2009philosophy,winsberg2019scienceintheageofsimulation,lavin2022simulationintelligencenewgeneration}.

Simulators are only approximations of reality, and their fidelity depends on how accurately they  capture the structure of the underlying system \citep{box1987empirical}. Constructing scientific simulators that are simultaneously interpretable, tractable, and realistic can be challenging. This modeling process requires substantial domain expertise and iterative manual refinement, rendering it labor-intensive and difficult to scale.

\begin{figure*}[t]
    \centering
    \includegraphics{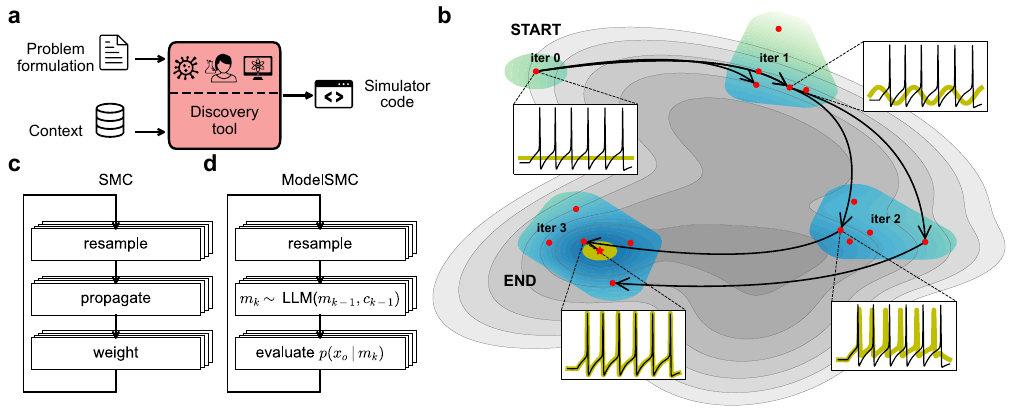}
    \caption{
    \textbf{Overview of ModelSMC for automated LLM-based model discovery.}
    \textbf{(a)} Given a textual problem formulation and context data, we infer a simulator model implemented in code.
    \textbf{(b)} ModelSMC iteratively refines an initial model to sample high-density regions of the model posterior $p(m | \vx_o)$, approaching the unknown data-generating process (red star) in high-density regions. 
    \textbf{(c)} ModelSMC is inspired by SMC, which approximates evolving distributions via weighted particles by iterative resampling, propagation, and weighting.
    \textbf{(d)} In ModelSMC, models are propagated by LLM sampling and weighted by likelihood evaluation.
    }
    \label{fig:overview}
\end{figure*}

A range of methods have been proposed to partially automate model discovery, but existing approaches typically trade off expressive model classes against mechanistic interpretability or require substantial manual specification \citep{udrescu2020ai,Donnelly2024pinnsurrogate,schroder2024sbmi}.
LLMs offer a qualitatively different point in this design space: trained on extensive corpora of scientific text and source code, they can generate and revise executable simulator implementations directly from textual problem formulations and feedback \citep{austin2021program,chen2021evaluating}.

However, most LLM-based discovery systems are defined operationally, i.e., via prompts, agent roles, and interaction protocols, rather than as inference procedures for an explicit target distribution.
As a result, these approaches 
(i)~do not explicitly state the underlying inference problem, which hinders interpretability, obscures failure modes, and prevents principled extensions or connections to existing theory;
(ii)~do not naturally admit theoretical analysis or convergence properties; and
(iii)~rely on loosely specified agent roles and interactions, making it difficult to reason formally about the necessity or sufficiency of components. %

Here, we formulate automated model discovery as a probabilistic inference problem. Given observational data $\vx_o$, the goal is to sample from the unknown distribution $p(m | \vx_o)$ over mechanistically interpretable models~$m$---here, simulator programs implemented in code---capable of explaining the data~$\vx_o$. Building on this formulation, we introduce ModelSMC, an approach based on Sequential Monte Carlo (SMC) methods, which approximate latent distributions using collections of weighted particles \citep{johansen2009tutorial}. In ModelSMC, candidate simulator models are represented as particles whose code is iteratively proposed and modified by an LLM, evaluated using likelihood-based criteria, and resampled to focus on promising (i.e., high-marginal likelihood) regions of the model space (\cref{fig:overview}).

Formulating model discovery as probabilistic inference within an SMC framework closely mirrors the iterative workflows employed by human experts, yielding conceptual and practical advantages:
First, it provides an explicit mathematical formulation of the refinement process, clarifying assumptions and failure modes. 
Second, under standard SMC conditions, it enables principled guarantees on stability and convergence. 
Third, it allows formal reasoning about the role and necessity of individual system components, facilitating systematic ablations and extensions.
Our main contributions can be summarized as follows:
\begin{itemize}
    \item We cast automated LLM-based discovery of models as probabilistic inference over executable programs.
    \item We propose ModelSMC, an SMC-style population method with LLM-based proposals of programs.
    \item We demonstrate competitive performance on synthetic tasks and show that ModelSMC discovers accurate and mechanistically interpretable models on real-world neuroscience and pharmacology datasets.
    \item We show how existing SMC theory can be used to analyze automated model discovery with LLMs.
\end{itemize}

\section{Related Work}

\textbf{Classic mechanistic model discovery.}
Automated discovery of mechanistic models has traditionally focused on symbolic regression, equation discovery, and program synthesis, including genetic programming \citep{Schmidt2009,udrescu2020ai,kamienny2022end} and sparse regression over predefined libraries of candidate terms \citep{brunton2016discovery,rudy2017data}, as well as grammar- or DSL-based program synthesis \citep{gulwani2017program,parisotto2017neurosymbolic,ellis2021dreamcoder}. Other approaches combine mechanistic structure with statistical learning or physics-informed models \citep{raissi2019physics,rackauckas2020universal,karniadakis2021physics,Donnelly2024pinnsurrogate}. These methods rely on carefully engineered search spaces and have so far focused on low-complexity models. In contrast, ModelSMC explores complex simulator programs spanning hundreds of lines of code in an open search space without hand-crafted restrictions.

\textbf{LLM-based methods.}
More recently, several works have used LLMs for open-ended automated model discovery. One class of approaches combines LLMs with symbolic regression or optimization loops, where the LLM proposes candidate equations or programs and iterative evaluation or gradient-based refinement is used to improve fit to data \citep{Ma2024SGA,Li2024Boxlm,shojaee2025llmsr,holt2025gsim}. Another line builds on evolutionary search, exemplified by FunSearch \citep{romera-paredes_mathematical_2024,Castro2025cogfunsearch} or the closed-source AlphaEvolve \citep{novikov2025alphaevolve}. These frameworks have been applied to domain-specific problems in cognitive science \citep{rmus2025generating,jha2025modeling}, language modeling \citep{cheng2025language}, pharmaceutical applications \citep{holt2024data}, neuroscience \citep{Tilbury2025}, or combinatorial optimization \citep{liu2024evolution}.
Model proposal, refinement, and selection are typically defined through heuristic system designs, such as fixed agent roles, optimization loops, or evolutionary operators. Selection in particular relies on general-purpose metrics such as MMD, MSE, or Wasserstein distance \citep{gretton2012kernel,arjovsky2017wasserstein}, rather than derived from an explicit mathematical objective.

Closest to our work are methods combining LLMs with Bayesian reasoning in scientific and probabilistic settings, such as Large Language Bayes \citep{domke2025large} and AutoDiscovery \citep{agarwal2025autodiscovery}. 
Large Language Bayes generates probabilistic programs from informal descriptions and performs approximate Bayesian inference via likelihood-weighted model averaging. 
AutoDiscovery explores nested hypotheses using Bayesian surprise within a Monte Carlo tree search, without explicitly defining a posterior over mechanistic models or maintaining a population of executable models.
In contrast, we frame mechanistic model discovery as inference over models.%

\textbf{Probabilistic model inference.}
Probabilistic approaches to model discovery and comparison are traditionally grounded in Bayesian model selection and averaging, where candidate models are evaluated via marginal likelihoods or Bayes factors \citep{kass1995bayes,morey2016bayes}, as well as Bayesian program induction \citep{lake2015human}. With simulators as tools for scientific inquiry, this has motivated simulation-based and likelihood-free inference methods, including synthetic likelihoods and neural density estimation \citep{Cranmer2020,deistler2025simulation}. It also includes SMC samplers for static or trans-dimensional targets, where explicit likelihood tempering bridges from prior to posterior to mitigate weight degeneracy \citep{del2006smcsampler,neal2001annealed}. Recent work, such as simulation-based model inference \citep{schroder2024sbmi}, extends this view to simultaneous model and parameter inference over graph-structured components using amortized inference, but assumes a fixed compositional prior without iterative refinement. In contrast, ModelSMC treats full model programs as latent variables and uses LLMs as proposals for open-ended model spaces, progressively concentrating on high-likelihood regions without an explicit tempering schedule.

\section{Method}  

\subsection{Background: Sequential Monte Carlo}\label{sec:MethodSMC}

SMC methods \citep{johansen2009tutorial,naesseth2019elements} approximate a sequence of target probability distributions using a set of weighted samples. Consider a state space model with latent states $x_{1:T}$ and observations $y_{1:T}$, with joint density $p(x_{1:T}, y_{1:T})$.
The inference task is to approximate the filtering or target distribution $p(x_t  |  y_{1:t})$. At each time step~$t$, SMC represents the distribution using a finite set of~$N$ particles $\{x_t^i\}_{i=1}^N$ and associated weights $\{w_t^i\}_{i=1}^N$ that together define an empirical measure

\begin{equation}
    p(x_t  |  y_{1:t}) \approx \sum_{i=1}^N w_t^i \, \delta_{x_t^i},
\end{equation}

providing a discrete approximation of the target distribution. In practice $N$ is of the order $10^2$–$10^5$. Convergence results guarantee that, under mild conditions, the particle approximation converges to the true target distribution as~$N\to\infty$ \citep{moral2004feynman,johansen2009tutorial}.

As a specific instantiation of SMC, we consider a standard bootstrap particle filter with adaptive resampling \citep{doucet2001sequential}. Given particles at time $t-1$, the update at time $t$ consists of the following steps:
\begin{enumerate}
    \item \textbf{Resample.} Sample $N$ ancestor indices $a_t^i$ with systematic resampling \citep{Kitagawa1996} if adaptive resampling is triggered. Otherwise copy all particles, i.e. $a_t^i=i$.
    \item \textbf{Propagate.} Sample new particles from the transition model, $x_t^i \sim p(x_t  |  x_{t-1}^{a_t^i})$.
    \item \textbf{Weight.} Assign unnormalized weights using the likelihood, $\tilde w_t^i = p(y_t  |  x_t^i)$, and normalize $w_t^i = \frac{\tilde w_t^i}{\sum_{j=1}^N \tilde w_t^j}$.
\end{enumerate}

Repeating these steps for $t=1,\dots,T$ yields a particle-based approximation of the target distributions. SMC can also be used for static target distributions, i.e., $p(x_t) \equiv p(x)$ for all $t$, by defining an artificial sequence of intermediate target distributions that progressively transform a simple initial distribution into the desired target (e.g., via likelihood tempering) \citep{neal2001annealed,del2006smcsampler}.

\subsection{Problem Formulation}\label{sec:MethodProblemFormulation}

Scientific model discovery aims to infer a model $m$ that reproduces the observable behavior of a system of interest~$m^*$. Information about~$m^*$ is available only through an observed set $\vx_o = \{(x_o^j, c_o^j)\}_{j=1}^M$, where each observation~$x_o^j$ is generated by~$m^*$ under context~$c_o^j$. The context may include, e.g., initial conditions or experimental parameters. We assume the observed data pairs are independent and identically distributed (iid).

A model $m$ is an algorithm written in code that probabilistically maps context $c \in C$ and model parameters $\theta \in \Theta$ to data $x \in X$. The model parameters $\theta$ are unknown but fall within a prior distribution $p(\theta)$ (details in App. \ref{app:ExperimentalDetailsPriorChoices}). For notational simplicity, we abbreviate $m(c,\theta) \equiv m$.

From a probabilistic perspective, model discovery corresponds to inference over models conditioned on the observed data. The inference target is a distribution over models $p(m|\vx_o)$ that assigns higher probability to models whose simulated behavior matches the observed behavior of $m^*$ across contexts. This distribution is defined implicitly via a likelihood and a prior: the likelihood $p(x_o|m)$ is induced by executing the simulator model defined by $m$ with latent parameters $\theta\sim p(\theta)$. The prior $p(m)$ corresponds to the marginal distribution over model implementations induced by an LLM-based generative process. The distribution $p(m|\vx_o)$ is generally intractable due to model space size and the absence of an explicit likelihood.

\subsection{SMC-Based Model Discovery with ModelSMC}
\label{sec:Method_ModelSMC}

ModelSMC approximates the posterior $p(m | \vx_o)$ via a population of $N$ candidate models that is iteratively refined using SMC.
Since the true system $m^*$ is fixed, the SMC time index corresponds to iterative refinement rather than temporal evolution. Unlike classical SMC, each particle corresponds to a candidate model rather than a dynamical state.

Each particle represents a candidate model and is associated with a weight reflecting its posterior relevance. In addition to the model implementation, particles store auxiliary context such as evaluation feedback. A particle consists of the tuple of model, weight and context $(m, w, c)$.

ModelSMC proceeds for $K$ iterations. In each iteration~$k$, particles are resampled, propagated to generate new candidate models, and weighted according to their likelihood under the observed data. Each iteration consists of three steps:
(i) resampling,
(ii) propagation, and 
(iii) weighting.
\Cref{alg:modelsmc} summarizes the complete procedure.

\begin{algorithm}[ht]
\small
\caption{ModelSMC (for $i=1,\dots,N$)}
\label{alg:modelsmc}
\renewcommand{\algorithmicrequire}{\textbf{Input:}}
\renewcommand{\algorithmicensure}{\textbf{Output:}}
\begin{algorithmic}[1]
\REQUIRE Number of particles~$N$, number of iterations~$K$, cloning probability~$\alpha\in[0,1]$, ESS threshold $\tau_{\mathrm{ESS}}>0$, observed data~$\vx_o=\{(x_o^j, c_o^j)\}_{j=1}^M$, task prompt~$c_{\text{task}}$, initial particle and weight~$(m_0, w_0=1)$.

\FOR{$k=1,\dots,K$}

    \STATE \textit{$\triangleright$} \textbf{1. Resample}

    Compute effective sample size (ESS)
    \begin{equation}\label{eq:AlgoModelSMCESS}
        \mathrm{ESS}_k = \frac{1}{\sum_{i=1}^N(w_k^i)^2}.
    \end{equation}

    \STATE Sample ancestor indices via systematic resampling $\mathcal{R}$.
    \begin{equation}
        a_k^i \sim 
        \begin{cases}
            \delta_i  & \text{if }\;\mathrm{ESS}_k\geq\tau_{\mathrm{ESS}}, \\
        \mathcal{R}\left(\left\{w_{k-1}^j\right\}_{j=1}^N\right)& \text{else.}
        \end{cases}
        \label{eq:AlgoModelSMCAncestorIdx}
    \end{equation}

    \STATE \textit{$\triangleright$} \textbf{2. Propagate}
    \STATE Propagate each particle $m_k^i$. Sample $z_i\sim \gU[0,1]$, then
    \begin{equation}\label{eq:AlgoModelSMCPropagation}
    m_k^i \sim 
    \begin{cases}
       \delta_{m_{k-1}^{a_k^i}}, & \text{if } z_i<\alpha, \\
       p_{\mathrm{LLM}}(m \mid m_{k-1}^{a_k^i}, c_{k-1}^{a_k^i}, c_{\text{task}}), & \text{else}.
    \end{cases}
    \end{equation}

    \STATE \textit{$\triangleright$} \textbf{3. Weight}
    \STATE Compute surrogate likelihood $p_\phi(x\mid m_k^i,\theta,c)$.
    \STATE Compute marginal likelihood for each $(x_o^j,c_o^j)\in\vx_o$
    \begin{equation}
        p(x_o^j\mid m_k^i)
        =
        \int p(\theta)\,
        p_\phi(x_o^j\mid m_k^i,\theta,c_o^j)\,
        d\theta.
        \label{eq:MarginalLikelihoodSingleInstance}
    \end{equation}
    \STATE Update weights $\tilde w_k^i$ and normalize to obtain $w_k^i$
    \begin{equation}
        \tilde w_k^i
        =
        \prod_{j=1}^M p(x_o^j\mid m_k^i).
        \label{eq:ResamplingWeightModelSMCFactorization}
    \end{equation}
    \STATE Generate feedback $c_k^i=\mathrm{LLM}(m_k^i,w_k^i, c_{\text{prompt}})$.
\ENDFOR

\ENSURE Particles and weights $\{m_K^i, w_K^i\}_{i=1}^N$.
\end{algorithmic}
\end{algorithm}

\textbf{Resampling.}
The effective sample size (ESS, \cref{eq:AlgoModelSMCESS}), a measure of how many particles carry meaningful weight, is computed for the particles in the population. If the ESS is below a threshold $\tau_{\mathrm{ESS}}$, systematic resampling \citep{Kitagawa1996} is applied to rejuvenate the population, i.e. to rebalance the population by discarding low-weight particles. Otherwise, the population remains unaltered.

\textbf{Propagation.}
Each particle in the population is propagated via a mixture kernel (\cref{eq:AlgoModelSMCPropagation}). With probability $\alpha$ the particle is cloned unchanged, and with probability $1-\alpha$ we sample a new implementation by drawing from the LLM conditioned on 
(i)~the models along its ancestry, 
(ii)~their associated contextual feedback, and 
(iii)~a fixed task-specific prompt $c_{\text{task}}$ encoding domain knowledge, implementation constraints, and evaluation criteria. Simulating and evaluating new models is computationally expensive; cloned particles avoid this cost entirely, motivating the mixture kernel. The number of particles in the population is constant by construction of the propagation scheme, while each iteration generates $N\cdot(1-\alpha)$ new particles in expectation.

The context $c_{k-1}^{a}$ contains the information on which the LLM is conditioned in the generation of a new implementation. In addition to the implementations of the resampled particle's ancestors, it also contains LLM-based performance assessments of the ancestors and task-specific information (prompting details in Apps.~\ref{app:LLMDetails} and ~\ref{app:PromptExamples}). To control computational cost and prompt length, only a limited number of the most recent ancestry elements are included. 
The probability $\alpha$ controls the trade-off between exploration and exploitation: cloned particles require no new simulations or likelihood evaluations, preserving computational resources, while newly proposed particles explore the model space. Selection among clones is handled implicitly by resampling based on likelihood weights.

This iterative adaptation of the proposal plays a role analogous to likelihood tempering in classical SMC samplers \citep{del2006smcsampler}: early iterations explore the model space broadly, while later iterations concentrate proposals around high-likelihood candidates as ancestry and feedback accumulate. Unlike explicit tempering, this annealing occurs in the proposal distribution rather than the target, and without formal variance-control guarantees.

\textbf{Weighting.}
Each generated model $m^i_k$ is assigned a weight proportional to the likelihood of the observed data,
\begin{equation}
    \tilde w_k^i = p(\vx_o \mid m_k^i).
    \label{eq:ResamplingWeightModelSMCText}
\end{equation}
Assuming iid observed data, the likelihood factorizes across samples (\cref{eq:ResamplingWeightModelSMCFactorization}). Weighting models by likelihood is consistent with the posterior $p(m|\vx_o)$ as target (under idealized assumptions; derivation in App.~\ref{app:ModelSMC_WeightDerivation}).

The likelihood $p(\vx_o|  m^i_k)$ induced by a model is generally intractable. To obtain an estimate of the likelihood, we can sample from the likelihood by evaluating the model at different values of $(\theta,c)$. Specifically, we sample a synthetic dataset $\theta\sim p(\theta)$, $c\sim p(c)$, $x \sim p(x|  m^i_k(\theta, c))$ and train a conditional density estimator $p_\phi(x |  m^i_k, \theta, c)$. This approach is also known as Neural Likelihood Estimation \citep[NLE]{papamakarios2019sequential}. Given a trained likelihood surrogate, the marginal likelihood for each observed data instance is obtained by integrating out $\theta$ via Monte Carlo approximation (\cref{eq:MarginalLikelihoodSingleInstance}). We note that the trained likelihood estimation is an approximation of the true likelihood (details in App.~\ref{app:Method_WeightEstimation}).

For NLE, we adapt a likelihood-estimation variant derived from neural posterior estimation with prior-data fitted networks (NPE-PFN) \citep{vetter2025effortless}. This approach uses the pretrained tabular foundation model TabPFN as a density estimator, enabling training-free likelihood estimation \citep{hollmann2025tabpfn}. Crucially, this eliminates the need for a separate likelihood estimation training loop at each iteration of ModelSMC, which is especially beneficial in model discovery where repeated likelihood evaluations are required. The foundation model is particularly efficient with small simulation budgets.
Our likelihood estimation approach allows ModelSMC to generate non-differentiable models without any explicit parameter optimization.

ModelSMC differs from prior model discovery methods in two key aspects. 
First, the weighting step is derived from the model posterior $p(m | \vx_o)$, providing a principled probabilistic objective rather than relying on alternative discrepancy measures such as MMD, MSE, or Wasserstein distance \citep{gretton2012kernel,arjovsky2017wasserstein}. 
Second, ModelSMC marginalizes over model parameters $\theta$ during discovery instead of optimizing point estimates, thereby reducing sensitivity to parameter uncertainty.  
For downstream use of discovered models, we perform parameter inference using neural posterior estimation (NPE) \citep{Lueckmann2017FlexibleStatisticalInference,papamakarios2018fastepsilonfreeinferencesimulation} and NPE-PFN \citep{vetter2025effortless} (details in App.~\ref{app:Method_ParameterEstimation}).

\subsection{Theoretical Insights}
Since our framework is based on SMC, we connect ModelSMC to classical SMC theory. Specifically, we apply a standard SMC consistency result \citep{chopin2004central,moral2004feynman} to the model discovery setting. showing that an idealized ModelSMC sampler provides an asymptotically unbiased particle approximation of the model posterior $p(m | \vx_o)$. 

\begin{theorem}[Consistency of ModelSMC]\label{thm:ConsistencyProofModelSMC}
    Let $\pi(m) \propto p(\vx_o|  m)p(m)$ be a fixed target distribution over models, and let $q(m' |  m)$ denote an idealized proposal kernel induced by the propagation step. Assume 
    (i) \emph{support coverage}, i.e., any $m$ with $\pi(m)>0$ can be generated with non-zero probability; 
    (ii) uniformly bounded importance weights; 
    (iii) conditional independence of propagated particles given resampling; and
    (iv) \emph{exact likelihood evaluation}, i.e., the importance weights are computed using the true marginal likelihood $p(\vx_o | m)$.
    Let $\{(m^{i}, w^{i})\}_{i=1}^N$ be the particle population at a fixed iteration. Then, for any bounded test function~$\varphi$,
    \begin{equation}
        \sum_{i=1}^N w^{i} \varphi(m^{i}) \xrightarrow{\mathbb{P}} \mathbb{E}_{\pi}[\varphi(m)] \quad \text{as} \quad N \to \infty,
    \end{equation}
    with asymptotic variance $\mathcal{O}(1/N)$. The result holds in the presence of (adaptive) resampling.
\end{theorem}
The proof is provided in App.~\ref{app:ProofTheorem}.
\Cref{thm:ConsistencyProofModelSMC} characterizes an idealized reference sampler with exact likelihoods and a prior-matching proposal. The practical algorithm deviates from this setting through surrogate likelihoods and a context-conditioned LLM proposal. Thus, the theorem is best understood as a proof of concept that the probabilistic framing of model discovery admits formal analysis through Bayesian inference over models.

\section{Experiments}

\subsection{Validation of ModelSMC for Model Discovery}
\label{sec:experiments_ValidationNoLLM}

We first validate the core probabilistic inference mechanism underlying ModelSMC in a controlled setting where the ground-truth model is known, and the model proposal is not affected by the LLM. We consider an LLM-free variant of ModelSMC on a finite model space with 20 candidates. Each candidate is a 10-dimensional Gaussian mixture models with varying numbers of components and dense covariances (details in App.~\ref{app:ExperimentalDetails_MinimalExample}). In this setting, the LLM-based propagation kernel component is replaced by uniformly sampling a random model from the candidate set.

\Cref{fig:MinimalExampleMainText} shows that the target model dominates the particle population for all evaluated configurations. Furthermore, a characteristic oscillation of the proportion of the target model and a upper bound of approximately $80\%$ can be observed. This is the expected behaviour for the evaluated propagation kernel and hyper parameter choice: resampling concentrates the population on the target model, while the random replacement of particles in the propagation step gradually reduces its proportion until the next resampling is triggered (details in App.~\ref{app:llm_free_example_additional_results}).

This experiment isolates the inference component of ModelSMC and shows that, in a finite model space, the algorithm concentrates posterior mass on the true data-generating process, behaving as a consistent posterior sampler over models. 
Thus, once the ground-truth model is proposed by the LLM, ModelSMC reliably identifies it. These results verify the inference procedure of ModelSMC.

\subsection{Quantitative Results} 
\label{sec:quantitative_results}

\begin{figure}[t]
    \centering
    \includegraphics[trim={7 7 7 7},clip,width=\columnwidth]{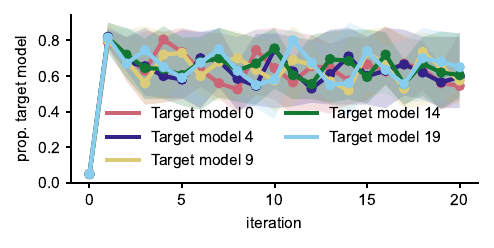}
    \caption{\textbf{Proportion of the target model in the population for LLM-free ModelSMC.} Median proportion of the target model in the population, averaged over ten random seeds, for five target models. Shaded regions depict the 100\% confidence interval of the median, estimated using 10,000 bootstrap replicates.}
    \label{fig:MinimalExampleMainText}
\end{figure}

\begin{figure*}[t]
    \centering
    \includegraphics{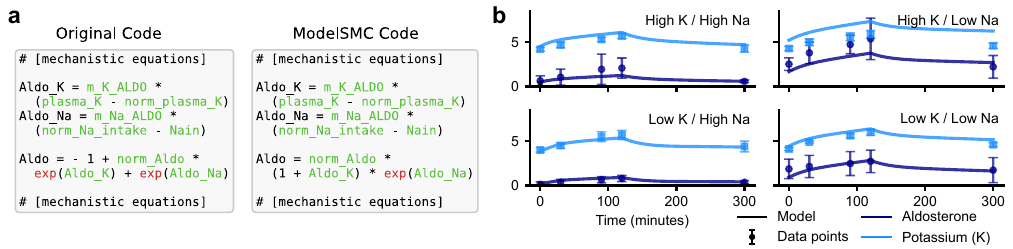}
    \caption{\textbf{Systems pharmacology kidney model with experimental data.} 
    \textbf{(a)} Original code snippet in R for the aldosterone mechanism (left) and one instance of the inferred model code (right). 
    \textbf{(b)} Posterior predictive for the code instance in (a) with real-world data points. 
    }
    \label{fig:results-kidney}
\end{figure*}

We perform a quantitative evaluation of ModelSMC across three model discovery tasks: a synthetic epidemiological system with known ground-truth dynamics (SIR; App.~\ref{app:ExperimentalDetails_SIR}), a pharmacological simulator of kidney aldosterone regulation implemented in R (\cref{sec:exp-kidney}), and a neuroscience Hodgkin-Huxley (HH) neuron simulator (\cref{sec:exp-allen}). ModelSMC targets model discovery as mechanistic equations for complex simulator models. We therefore compare against two baselines matching this scope: (i)~FunSearch+ a modification of FunSearch \citep{romera-paredes_mathematical_2024} that optimizes free model parameters and uses likelihood-based weights \citep{Castro2025cogfunsearch}, and (ii)~ModelSMC with a single particle ($N$=1) where the most recent particle is propagated ($\alpha = 0.0$), which removes the population-based inference of the full algorithm. 
Together, these two baselines represent the two dominant paradigms in LLM-based model discovery: evolutionary search (FunSearch+) and sequential non-probabilistic model exploration (ModelSMC $N$=1), the latter being representative of related sequential methods \citep{holt2025gsim,Ma2024SGA,Li2024Boxlm,shojaee2025llmsr}. Crucially, ModelSMC $N$=1 differs from ModelSMC only in removing the SMC framework, isolating the contribution of population-based probabilistic inference.
FunSearch+ is Python-specific, which prevents direct application to the R-based kidney simulator. We do not compare with classical symbolic regression because it is typically limited to low-complexity functions \citep{Schmidt2009}, and black-box predictive models, such as transformers, because they lack mechanistic interpretability.

\begin{table}[t]
\centering
\caption{\textbf{Performance evaluation.}
We evaluate the median metrics for the best particles over ten discovery runs. Results including confidence intervals are reported in Appendix \cref{tab:appendix_quantitative_results}.}
\setlength{\tabcolsep}{0pt}
\label{tab:quantitative_results}
\begin{tabularx}{\linewidth}{>{\centering\arraybackslash}p{1.0cm} *{3}{Y}ccc}
\toprule
 & \multicolumn{1}{c}{FunSearch+} & \multicolumn{1}{c}{ModelSMC} & \multicolumn{1}{c}{ModelSMC $N=1$} \\
\midrule
 & \multicolumn{3}{c}{$-\log p(\vx_o|m)~(\downarrow)$} \\
\midrule
SIR & -48794.12 & -45834.78 & -57028.36 \\
HH & 254.58 & 256.00 & 243.27 \\
Kidney & N/A & 42.85 & 36.54 \\
\bottomrule
\end{tabularx}
\end{table}

Across all tasks, ModelSMC achieves comparable quantitative performance (\cref{tab:quantitative_results}). On both synthetic and real-world tasks, it attains competitive marginal and conditional likelihoods at comparable computational budgets (\cref{fig:appendix_token_usage_vs_performance}), remaining on par with baseline methods. Unlike FunSearch+, ModelSMC handles any coding language. 

These results indicate that framing model discovery as an inference problem does not compromise quantitative performance. Instead, ModelSMC discovers competitive mechanistic models while providing the probabilistic foundation that enables additional interpretability and analysis.

\subsection{Pharmacological Kidney Model}
\label{sec:exp-kidney}

We evaluate ModelSMC on a real-world quantitative systems pharmacology (QSP) problem.
Specifically, we adapt a simulator model of plasma potassium regulation in the kidney mediated by aldosterone, a key steroid hormone \citep{Maddah2022}. The original model is implemented in R using RxODE2 \citep{wang2016tutorial,Rcore2025} and is derived from first principles of pharmacokinetics and pharmacodynamics, capturing renal excretion, tubular reabsorption, and hormonal feedback. The model's aldosterone mechanism is empirically fitted rather than biophysically derived~\citep{Maddah2022}.

We intentionally replace the aldosterone mechanism with a constant term (\cref{fig:results-kidney}a, left), thereby introducing a targeted structural misspecification. 
ModelSMC is tasked with inferring the missing regulatory dynamics from the data.
The inference problem is particularly challenging due to extremely limited data: four patient intake scenarios (high\,/\, low K $\times$ high\,/\,low Na), with a single measurement of plasma potassium and aldosterone at five time points per scenario \citep{dluhy1972studies} (details in App.~\ref{app:ExperimentalDetails_kidney}).

Across all seeds, ModelSMC assigns high posterior weight to a narrow family of aldosterone feedback mechanisms whose functional form closely matches the originally proposed model (\cref{fig:results-kidney}a, right, for one model). 
Although individual particles differ in structural details, models encoding qualitatively different regulatory hypotheses receive negligible posterior weight.
Posterior predictive simulations from high-probability models closely track observed plasma potassium and aldosterone measurements, despite the limited data (\cref{fig:results-kidney}b).

From an inference perspective, this experiment illustrates strong posterior concentration over a mechanistically coherent class rather than the selection of a single best-fitting model.
The absence of competing high-weight alternatives indicates that the available data are sufficient to rule out broad classes of hypotheses, diagnosing structural inadequacy rather than parameter uncertainty.
Compared to methods producing a single point estimate, ModelSMC yields a posterior over low-complexity mechanisms that both generalize in posterior predictive checks and quantify uncertainty over plausible mechanisms.

\subsection{Hodgkin-Huxley Model}
\label{sec:exp-allen}

\begin{figure*}[t]
    \centering
    \includegraphics{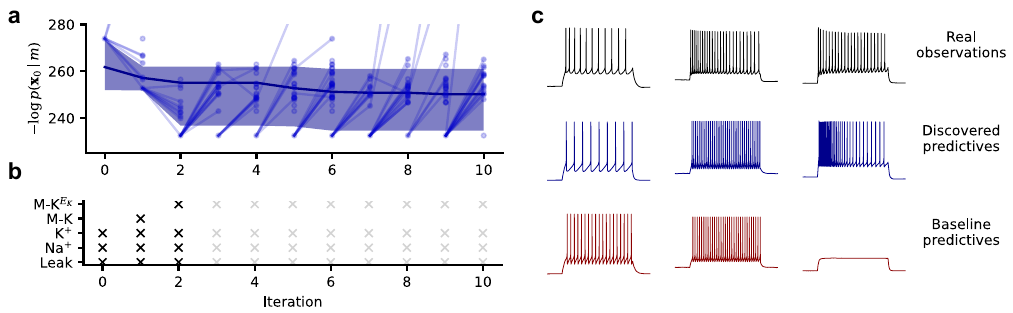}
    \caption{\textbf{Hodgkin--Huxley model with Allen cell types database data.} 
    Results are shown for one representative random seed from the 10 runs, used consistently across panels.
    \textbf{(a)} ModelSMC convergence across runs, showing the mean (solid line) and 95\% percentiles (shaded). With lower opacity, we overlay the full model ancestry of the selected run.
    \textbf{(b)} For the selected run, highlighting of inferred ion-channel mechanisms at successive improvement stages.
    \textbf{(c)} Posterior predictive voltage simulations for real observations from the Allen database, comparing the baseline HH model to the best model identified in the selected run.
    }
    \label{fig:results-allen}
\end{figure*}

Finally, we evaluate ModelSMC on a neuroscience simulator of neuronal membrane dynamics. Specifically, we consider the single-compartment Hodgkin-Huxley~(HH) model of action potential generation \citep{hodgkin1952,pospischil2008minimal}, which describes membrane voltage dynamics via nonlinear ion-channel conductances and gating variables.
As baseline, we consider the simulator from \citet{gonccalves2020training}, consisting of eight biophysical parameters and explicit ion channels for sodium, potassium, and leak current. While this model provides a strong description of neuronal dynamics, it can exhibit systematic discrepancies when fitted to real recordings \citep{gao2023generalized,bernaerts2023combinedhodgkinhuxley}. 
Observations consist of membrane voltage traces summarized by seven standard summary statistics \citep{gonccalves2020training}. We consider 10 real voltage recordings from the Allen Cell Types Database \citep{allen_cell_types_2015}.

In contrast to prior work focusing exclusively on parameter inference \citep{gonccalves2020training,vetter2025effortless}, we treat the HH model as a structural inference problem. Taking the baseline HH model, ModelSMC is tasked with inferring additional ion-channel mechanisms that improve model specification and reduce systematic mismatch to data.
This setting is particularly challenging due to the limited number of real-world observations, the already high fidelity of the baseline model, and the need to detect subtle but consequential structural deficiencies rather than gross misspecification (details in App. \ref{app:ExperimentalDetails_allen}).

ModelSMC shows rapid posterior concentration while continuing to explore multiple high-likelihood structural variants over successive iterations (\cref{fig:results-allen}a). 
For a representative run, particle ancestry reveals a diverse sampling of the model space with distinct ion-channel extensions of the baseline HH model (\cref{fig:results-allen}a, opaque).
Across runs, high posterior weight is consistently assigned to models introducing an M-type slow potassium current, also paired with a secondary channel such as persistent sodium (I$_{NaP}$), A-type, or HCN (I$_h$).
Within these families, refinements such as a corrected potassium reversal potential (E$_K$) receive increased posterior weight and improve posterior predictive agreement with recorded voltage traces, particularly for spike-frequency adaptation (\cref{fig:results-allen}b,c).

Rather than identifying a single revised HH model, ModelSMC yields a posterior over mechanistically interpretable extensions, with comparable explanatory power.
The posterior exposes residual non-identifiability in ion-channel composition while isolating robust conclusions, i.e., the extension with additional slow currents.
Thus, posterior mass serves as a diagnostic: low-weight regions indicate genuine structural mismatch, while clusters of high-probability models reveal symmetries and ambiguities that are invisible to single-model discovery.

\subsection{Posterior Analysis}
\label{sec:exp-posterior}

\begin{figure*}[t]
    \centering
    \includegraphics[trim={0 0 0 17}, clip, width=\textwidth]{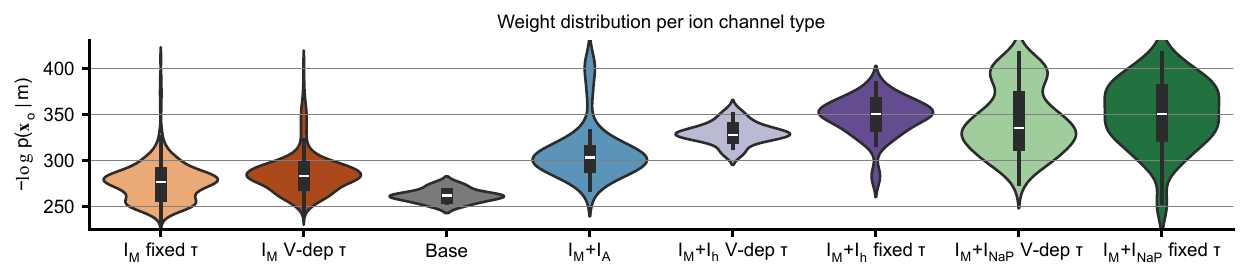}
    \caption{\textbf{Posterior mass analysis.}
    Distribution of weights per ion-channel subtype, pooled across all 1,445 particles and 10 seeds of the Hodgekin-Huxley experiment.
    Channel types are sorted and colored by family.
    }
    \label{fig:results-posterior}
\end{figure*}

Particle weights in ModelSMC represent a posterior distribution, enabling evaluation of entire families of modelling hypotheses beyond single-model selection.
We demonstrate this for the ion-channel extensions from \cref{sec:exp-allen}, showing which structural choices are supported by the data and which are ruled out.

We apply a three-step LLM-assisted pipeline to all $1{,}445$ particles across 10 seeds.
First, a stratified sample of $\approx15\%$ of particles is used to extract free-form descriptions of the added channels, and consolidated into a two-level taxonomy of 11 structural subtypes across families (I$_M$, I$_{NaP}$, I$_h$, I$_A$, and multi-channel combinations).
Second, all particles are classified against this taxonomy.
Third, weights are compared across subtypes (details in App.~\ref{app:PosteriorAnalysis_allen}).

Among the added channels, I$_M$ variants (M-type K$^+$ currents) consistently achieve the lowest weights across subtypes (\cref{fig:results-posterior}), while I$_M$ with a second channel including I$_{NaP}$, I$_h$, and I$_A$ variants receive substantially higher weights.
The unextended base model itself receives comparably low weights, reflecting its already high fidelity, but is consistently outperformed by I$_M$ extensions.
This ranking is stable across independent runs, with I$_M$ placing first in 9 out of 10 seeds (App.~\cref{fig:results-posterior-appendix}), demonstrating that ModelSMC arrives at reproducible posterior conclusions rather than spreading mass uniformly across alternatives (details in App.~\ref{app:PosteriorAnalysis_allen}).

\subsection{Ablation Study}
\label{sec:exp-ablation}

\begin{table}[h]
\centering
\caption{\textbf{Ablation study summary on the Hodgkin-Huxley task.}
Median $-\!\log p(\vx_o|m)$ with $90\%$ confidence intervals for each variant relative to the default (first row). The ablation results are based on 5 runs. Underlined cells indicate configurations outperforming the default. The confidence intervals are estimated using 10,000 bootstrap replicates.
}
\label{tab:ablation_main}
\setlength{\tabcolsep}{4pt}
\begin{tabular}{llr}
\toprule
Axis & Ablation&$-\log p(\vx_o|m)~(\downarrow)$\\
\midrule
Default&&256.00 {\scriptsize [255.05, 262.03]}\\
\midrule
Weighting&MSE weights&264.65 {\scriptsize [250.90, 341.82]}\\
\midrule
Simulations&1000&270.07 {\scriptsize [250.34, 292.69]}\\
&500&263.89 {\scriptsize [254.03, 284.50]}\\
&200&266.91 {\scriptsize [264.11, 292.58]}\\
\midrule
Prompt&Reduced&262.18 {\scriptsize [249.96, 282.15]}\\
\midrule
Feedback&LLM + metrics&\underline{254.82} {\scriptsize [233.71, 262.89]}\\
&Metrics only &258.30 {\scriptsize [245.49, 262.89]}\\
\midrule
Pool size&$N=150,K=5$&267.01 {\scriptsize [260.03, 278.91]}\\
&$N=10,K=75$&258.25 {\scriptsize [254.10, 277.81]}\\
&$N=5,K=150$&\underline{248.02} {\scriptsize [237.97, 267.01]}\\
\midrule
LLM&GPT-5 mini&\underline{254.49} {\scriptsize [242.22, 256.04]}\\
\bottomrule
\end{tabular}
\end{table}

We evaluate the sensitivity of ModelSMC to its core design choices on the Hodgkin-Huxley task (\cref{sec:exp-allen}), systematically varying six axes while holding all others fixed at the default configuration.
We vary the \emph{weighting scheme} (marginal likelihood vs.\ MSE) to test whether the principled Bayesian weights are necessary.
We reduce the \emph{simulation budget} for likelihood estimation ($5{,}000$ to $200$--$1{,}000$) to understand the compute--quality trade-off. 
We simplify the \emph{prompt} to assess how much prompt engineering contributes.
We change the \emph{feedback format} (LLM narrative, metrics, or both) to determine whether richer feedback improves LLM proposals. 
We vary \emph{pool size} $N$ and number of iterations $K$ at a fixed total budget of $K{\cdot}(1{-}\alpha){\cdot}N = 150$ new particles to study the depth-vs-diversity trade-off. 
Finally, we replace the \emph{LLM backbone} with GPT-5 mini to verify robustness across models.
Overall, no single variation causes a significant degradation in performance (\cref{tab:ablation_main}), indicating that ModelSMC is robust to these design choices (full results in App.~\ref{app:ablation_study_allen}).

\section{Discussion}

We formulated automated discovery of mechanistic models as Bayesian inference over programs. We introduced ModelSMC, an SMC-based algorithm that uses LLMs as proposal distributions. Unlike prior agentic systems, ModelSMC directly approximates the posterior $p(m|\vx_o)$ by marginalizing over parameters. This approach produces a weighted population of models, revealing which hypotheses concentrate posterior mass, which remain uncertain, and which are non-identifiable. Across synthetic and real-world systems, this enables diagnosing which mechanistic hypotheses are strongly supported, ruled out, or remain as plausible alternatives under limited data.

A natural question is whether automated discovery systems generate genuinely novel scientific insights or primarily rediscover known mechanisms. We emphasize that our primary contribution is methodological: we introduce a general inference framework in which such questions can be meaningfully posed. Whether novel mechanisms emerge depends on the application, data, and modeling assumptions, and is orthogonal to the validity of the framework itself.
Practically, ModelSMC applies to settings where candidate models can be executed and admit reliable likelihood approximation. In such domains, probabilistic reasoning over models is well-defined and informative.

\textbf{Limitations.}
We introduce an inference framework for model discovery, instantiated by ModelSMC. Limitations arise at algorithmic and framework levels:
First, ModelSMC is computationally expensive due to repeated simulation and likelihood estimation for each candidate model. This could be alleviated by multi-fidelity simulation, adaptive particle budgets, or early-rejection strategies from SMC literature. 
Second, posterior weights rely on surrogate likelihoods, which may introduce bias under misspecification. Incorporating uncertainty-aware or robust likelihood estimation is an important direction.
Third, the discovery process is bottlenecked by the LLM's ability to propose relevant model structures. Improvements in LLM capability, domain-specific pretraining, or retrieval-augmented generation will directly enhance discovery performance \citep{lewis2020retrieval}.
Finally, while posterior weights enable principled ranking and clustering of hypotheses (\cref{sec:exp-posterior}), models remain high-dimensional programs without an explicit geometry. Developing notions of similarity or clustering over model space would enable richer posterior analysis and enhance inference with semantic structure over program spaces.

\textbf{Conclusion.}
By reframing LLM-driven model discovery as probabilistic inference, we provide a framework that unifies proposal, refinement, and selection under a single probabilistic objective. 
ModelSMC provides one concrete instantiation of this view, combining LLM-based proposals with likelihood-based weighting to approximate a posterior over mechanistic models.
Crucially, this connection enables direct use of established Monte Carlo methodology, including particle rejuvenation moves \citep{gordon1993novel}, SMC samplers for static targets \citep{del2006smcsampler}, adaptive resampling schemes \citep{doucet2001sequential,johansen2009tutorial,del2012adaptive}, and annealed or tempered objectives \citep{neal2001annealed}. 
More broadly, this perspective situates recent LLM-based discovery systems within a well-studied inference framework, providing a common language for comparison, analysis, and systematic development of future methods.

\section*{Acknowledgements}
We thank Manuel Gloeckler for feedback on the manuscript, members of the Mackelab for discussions, and Julian Schwab for domain expertise on the QSP experiment. We thank Anthropic for API credits provided through the AI for Science Program. This work was funded by the German Research Foundation (DFG) under Germany’s Excellence Strategy – EXC number 2064/1 – 390727645 and SFB 1233 `Robust Vision' (276693517), and the Boehringer Ingelheim AI \& Data Science Fellowship Program. SW is a member of the International Max Planck Research School for Intelligent Systems (IMPRS-IS).

\section*{Impact Statement}
This paper presents work whose goal is to advance the field of Machine Learning. There are many potential societal consequences of our work, none of which we feel must be specifically highlighted here.

\bibliography{references}
\bibliographystyle{icml2026}

\newpage

\appendix

\onecolumn

\newcommand{\appsection}[1]{%
  \section{#1}%
  \edef\appsec@anchor{appsec.\thesection}%
  \hypertarget{\appsec@anchor}{}%
  \addtocontents{toc}{\protect\contentsline{section}%
    {\protect\numberline{\thesection}#1}{\thepage}{\appsec@anchor}}%
}
\newcommand{\appsubsection}[1]{%
  \subsection{#1}%
  \edef\appsec@anchor{appsec.\thesubsection}%
  \hypertarget{\appsec@anchor}{}%
  \addtocontents{toc}{\protect\contentsline{subsection}%
    {\protect\numberline{\thesubsection}#1}{\thepage}{\appsec@anchor}}%
}
\etocsettocstyle{\section*{Appendix Contents}\vspace{-6pt}}{}
\etocsetnexttocdepth{subsection}
\localtableofcontents
\newpage

\setcounter{section}{0}
\renewcommand{\thesection}{\Alph{section}}

\renewcommand{\thefigure}{\Alph{section}-\arabic{figure}}
\setcounter{figure}{0}
\renewcommand{\thetable}{\Alph{section}-\arabic{table}}
\setcounter{table}{0}
\renewcommand{\theequation}{\Alph{section}-\arabic{equation}}
\setcounter{equation}{0} %

\appsection{Software and Computational Resources}\label{app:SoftwareComputationalResources}

All methods are implemented in Python. For general computations, NumPy \citep{harris2020array} and PyTorch \citep{paszke2019pytorchimperativestylehighperformance} are used. Plotting is performed using Matplotlib \citep{Hunter2007Matplotlib}. To store and handle experimental results, we use pandas \citep{reback2020pandas}. NLE and NPE are performed using SBI's \citep{BoeltsDeistler_sbi_2025} implementation of the corresponding methods. In the kidney experiment (\cref{sec:exp-kidney}), the models are implemented in R \citep{Rcore2025}. To parallelize the evaluation of multiple particles in ModelSMC, we use Ray \citep{Moritz2018Ray}. LLMs are accessed via DSPy \citep{khattab2022demonstrate,khattab2024dspy}, which automatically assembles prompts and returns structured outputs with named fields. Unless stated otherwise, we run all experiments with \texttt{Claude Sonnet 4.6} \citep{Anthropic2026ClaudeSonnet46intro} as language model and have an ablation using \texttt{GPT-5-mini} \citep{OpenAI2025chatgpt5intro}. Within individual discovery runs, the same language model is used for propagation and feedback generation (App.~\ref{app:LLMDetails}). Code is available at \url{https://github.com/mackelab/ModelSMC}.

NPE-PFN (App.~\ref{app:Method_ParameterEstimation}) and NLE-PFN (App.~\ref{app:Method_WeightEstimation}) are run on an A100 GPU with 40GB of memory with associated CPUs. NLE and NPE are run on the CPUs only. The particle propagation and evaluation in ModelSMC can be parallelized. In our experiments, we run up to five particles in parallel. In contrast, the baseline method ModelSMC $N$=1 and FunSearch+ are run sequentially.

\appsection{Derivation of ModelSMC Weights from a Static Bayesian Target}
\label{app:ModelSMC_WeightDerivation}

This section derives the weighting rule used in ModelSMC from a Bayesian posterior over model programs.
The derivation makes explicit how likelihood-based weights arise from an SMC construction, with an underlying static inference problem.

\paragraph{Inference target.}
Given observed data $\vx_o = \{(x_o^j, c_o^j)\}_{j=1}^M$, the target posterior over models is
\begin{align}
    p(m \mid \vx_o)
    \;\propto\;
    p(\vx_o \mid m)\, p(m),
    \label{eq:PosteriorModel_App}
\end{align}
with marginal likelihood
\begin{align}
    p(\vx_o \mid m)
    =
    \prod_{j=1}^M p(x_o^j \mid m),
    \qquad
    p(x_o^j \mid m)
    =
    \int p(\theta)\, p(x_o^j \mid m, \theta, c_o^j)\, d\theta.
    \label{eq:MarginalLikelihood_App}
\end{align}
The prior $p(m)$ is implicit and corresponds to the marginal distribution over models induced by the LLM-based generative process in the absence of conditioning on observed data. 

For notational simplicity, we drop the explicit conditioning on the observed context $c_o^j$, including particle-specific feedback variables, for the remainder of the derivation.
This corresponds to an idealized setting in which the model proposal distribution depends only on the previously generated model.
In practice, the proposal is conditioned on additional contextual information, which leads to approximate rather than exact importance weights (see Remarks below at the end of this section).

\paragraph{Artificial SMC target distribution.}
To apply SMC, we introduce an artificial sequence of latent variables
$(m_1,\dots,m_K)$, where each $m_k$ represents a candidate model.
We define the unnormalized target distribution on the extended space as
\begin{align}
    \tilde\pi(m_{1:K})
    :=
    p(m_{1:K}, \vx_o)
    =
    \prod_{k=1}^K
    p(m_k \mid m_{1:k-1})\,
    p(\vx_o \mid m_k).
    \label{eq:ExtendedTarget_App}
\end{align}
The corresponding normalized target is
\begin{align}
    \pi(m_{1:K})
    =
    \frac{\tilde\pi(m_{1:K})}{Z},
    \qquad
    Z = p(\vx_o)^K.
\end{align}
Marginalizing $\pi(m_{1:K})$ over $m_{1:k-1}$ yields the posterior over models under the LLM-induced prior, $p_{\text{LLM}}(m_k \mid \vx_o)$ for every $k$.

Thus, the sequence is purely algorithmic and does not represent a temporal evolution of the system.

\paragraph{Proposal distribution.}
At iteration $k$, ModelSMC generates a new model using the same LLM-based generative mechanism.
In practice, this proposal is conditioned on the previous model, particle-specific feedback, and task context, yielding
\begin{align}
    q_k(m_k \mid m_{1:k-1})
    :=
    p_{\mathrm{LLM}}(m_k \mid m_{1:k-1}, c_{1:k-1}, c_{\text{task}}).
\end{align}
For the purpose of the derivation, we absorb all contextual variables into the conditioning and consider an idealized proposal that depends only on $m_{1:k-1}$.

\paragraph{Incremental importance weights.}
The standard SMC incremental weight is
\begin{align}
    \tilde w_k
    &=
    \frac{
        \tilde\pi(m_{1:k})
    }{
        \tilde\pi(m_{1:k-1})\,
        q_k(m_k \mid m_{1:k-1})
    }.
\end{align}
Using \cref{eq:ExtendedTarget_App}, we obtain
\begin{align}
    \tilde w_k
    &=
    \frac{
        \prod_{t=1}^k p(m_t \mid m_{1:t-1})\, p(\vx_o \mid m_t)
    }{
        \Big(\prod_{t=1}^{k-1} p(m_t \mid m_{1:t-1})\, p(\vx_o \mid m_t)\Big)\,
        q_k(m_k \mid m_{1:k-1})
    } \\
    &=
    \frac{
        p(m_k \mid m_{1:k-1})\, p(\vx_o \mid m_k)
    }{
        q_k(m_k \mid m_{1:k-1})
    }.
\end{align}
We refer to this as a prior-matching proposal when the proposal coincides with the conditional generative distribution defining the implicit model prior in an idealized setting, i.e.,
\begin{align}
    q_k(m_k \mid m_{1:k-1})
    =
    p(m_k \mid m_{1:k-1}).
\end{align}
Under this assumption, the transition terms cancel exactly, yielding the idealized incremental weight
\begin{align}
    \tilde w_k
    =
    p(\vx_o \mid m_k).
    \label{eq:IncrementalWeightFinal_App}
\end{align}
In practice, the proposal is conditioned on additional contextual information (e.g., data-dependent feedback, task context, or finite context windows), and the exact incremental weight would take the form
\begin{align}
    \tilde w_k^{\mathrm{true}}
    =
    p(\vx_o \mid m_k)\,
    \frac{p(m_k \mid m_{1:k-1})}
         {q_k(m_k \mid m_{1:k-1}, \vx_o, c_{1:k-1})}.
\end{align}
Because this context-conditioned proposal density is implicit and intractable for large language models, ModelSMC omits the correction factor and uses $p(\vx_o | m_k)$ as a practical, approximate importance weight.

\paragraph{Resulting weighting rule.}
Thus, under a prior-matching proposal, the incremental SMC weight at iteration $k$
is precisely the marginal likelihood of the newly proposed model.
This yields the weighting rule used in ModelSMC as the likelihood
\begin{align}
    \tilde w(m)
    \propto
    p(\vx_o \mid m).
\end{align}

\paragraph{Remarks.}
\begin{itemize}
    \item \textbf{Implicit model prior.}  
    The prior $p(m)$ is implicit and defined by the marginal distribution over models induced by the LLM generative process. The proposal distribution corresponds to its conditional transitions.
    ModelSMC therefore performs inference under this implicit prior rather than an explicit analytic one.
    \item \textbf{Approximate prior matching.}
    The derivation assumes an idealized prior-matching proposal that ignores contextual conditioning.
    In practice, the LLM proposal depends on feedback and task-specific context, so the exact importance ratio for the SMC weight would additionally include a proposal correction term.
    Because this proposal density is implicit and intractable, ModelSMC uses likelihood-based weights as a practical approximation, corresponding to inference under an implicit LLM-induced prior.
    \item \textbf{Static target.}  
    The sequence $(m_1,\dots,m_K)$ is artificial and introduced solely to enable SMC;
    all marginals correspond to the same static posterior $p(m \mid \vx_o)$.
    \item \textbf{Tempering extension.}
    The current formulation uses the full likelihood $p(\vx_o \mid m_k)$ at every iteration. This can be generalized by introducing a schedule $\beta_1 \leq \cdots \leq \beta_K = 1$ and replacing the target at iteration $k$ with $p(\vx_o \mid m_k)^{\beta_k}$, recovering a standard tempered SMC sampler \citep{del2006smcsampler}. The existing temperature parameter $\tau$ (\cref{eq:TenmperatureScaledResamplingWeight}) could be adapted to serve this role.
\end{itemize}

\appsection{Proof of ModelSMC Convergence}\label{app:ProofTheorem}

\begin{theorem}[restated \cref{thm:ConsistencyProofModelSMC}]
    Let $\pi(m) \propto p(\vx_o|  m)p(m)$ be a fixed target distribution over models, and let $q(m' |  m)$ denote an idealized proposal kernel induced by the propagation step.
    Assume 
    (i) \emph{support coverage}, i.e., any $m$ with $\pi(m)>0$ can be generated with non-zero probability; 
    (ii) uniformly bounded importance weights; 
    (iii) conditional independence of propagated particles given resampling; and
    (iv) \emph{exact likelihood evaluation}, i.e., the importance weights are computed using the true marginal likelihood $p(\vx_o | m)$.
    Let $\{(m^{i}, w^{i})\}_{i=1}^N$ be the particle population at a fixed iteration. Then, for any bounded test function~$\varphi$,
    \begin{align}
        \sum_{i=1}^N w^{i} \varphi(m^{i}) \xrightarrow{\mathbb{P}} \mathbb{E}_{\pi}[\varphi(m)] \quad \text{as} \quad N \to \infty,
    \end{align}
    with asymptotic variance $\mathcal{O}(1/N)$. The result holds in the presence of (adaptive) resampling.
\end{theorem}

\paragraph{Discussion of assumptions.}
\begin{itemize}
    \item \textbf{Assumption~(i)} (support coverage) requires that the proposal induced by the LLM assigns non-zero probability to any model with non-zero posterior mass. This assumption cannot be verified formally and depends critically on the expressive capacity and training of the underlying LLM. In practice, we assume that a large class of scientifically meaningful models, including those commonly used in the literature, their combinations, and moderate extrapolations thereof, lie within the effective support of the proposal distribution. The use of textual feedback and task-specific prompts allows the proposal to be steered toward relevant regions of model space, but does not guarantee full coverage. If this assumption is violated, ModelSMC converges to the posterior restricted to the subset of models reachable by the proposal, resulting in asymptotic bias.
    \item \textbf{Assumption~(ii)} (uniformly bounded importance weights) ensures finite variance of the Monte Carlo estimator.
    While difficult to verify formally, it is encouraged in ModelSMC by finite surrogate likelihood outputs. 
    \item \textbf{Assumption~(iii)} is standard in SMC analyses and is satisfied by resampling followed by conditionally independent propagation \citep{moral2004feynman}. In ModelSMC, propagation additionally depends on particle context, so strict conditional independence is not fully met in practice. The theorem therefore characterizes an idealized setting that serves as a reference point for the algorithm’s behavior.
    \item \textbf{Assumption~(iv)} formalizes an idealized setting in which the marginal likelihood $p(\vx_o | m)$ is available exactly. In ModelSMC, this likelihood is approximated using a surrogate density estimator using NLE \citep{papamakarios2019sequential} or NLE-PFN (details in App.~\ref{app:Method_WeightEstimation}). The theorem therefore characterizes the asymptotic behavior of an idealized sampler and serves as a reference point for the practical algorithm, as is standard in analyses of likelihood-free SMC methods \citep{del2006smcsampler}.
\end{itemize}

The following argument establishes consistency for an idealized importance sampling or SMC scheme in which the proposal distribution $q_k$ is known and used to compute exact importance weights.
In ModelSMC, the proposal is implicit and context-dependent, so the theorem applies to the corresponding idealized sampler that motivates the practical algorithm.

\begin{proof}
The proof follows classical importance sampling arguments and their extension to SMC samplers for static target distributions, adapted to the ModelSMC setting \citep{robert1999monte,moral2004feynman,del2006smcsampler}.

\paragraph{Step 1: Target expectation and proposal identity.}
Fix an iteration $k$. Conditional on the resampling step at iteration $k-1$, the propagation step induces a proposal distribution
\begin{align}
    q_k(m) := p_{\mathrm{LLM}}(m \mid m_{k-1}),
\end{align}
where $m_{k-1}$ denotes a resampled ancestor. By assumption~(i), the support of $q_k$ contains the support of $\pi$.

For any bounded test function $\varphi$, we have
\begin{align}
    \mathbb E_{\pi}[\varphi(m)]
    = \int \varphi(m)\,\pi(m)\,dm 
    = \int \varphi(m)\,\frac{\pi(m)}{q_k(m)}\,q_k(m)\,dm 
    = \mathbb E_{q_k}\!\left[\frac{\pi(m)}{q_k(m)}\,\varphi(m)\right].
\end{align}

\paragraph{Step 2: Importance sampling estimator.}
Let $m^1,\dots,m^N$ be iid samples from $q_k(m)$.
Define (unnormalized) importance weights
\begin{align}
    \tilde w^i = \frac{\pi(m^i)}{q_k(m^i)}, 
    \qquad
    w^i = \frac{\tilde w^i}{\sum_{j=1}^N \tilde w^j}.
\end{align}
The corresponding estimator is
\begin{align}
    \hat I^N := \sum_{i=1}^N w^i\,\varphi(m^i).
\end{align}

By the strong law of large numbers,
\begin{align}
    \sum_{i=1}^N  w^i\,\varphi(m^i)
    &\xrightarrow{\text{a.s.}}
    \mathbb E_{q_k}\!\left[\frac{\pi(m)}{q_k(m)}\,\varphi(m)\right]
    = \mathbb E_{\pi}[\varphi(m)], \\
    \sum_{i=1}^N  w^i
    &\xrightarrow{\text{a.s.}}
    \mathbb E_{q_k}\!\left[\frac{\pi(m)}{q_k(m)}\right]
    = \int \pi(m)\,dm = 1.
\end{align}
Hence $\hat I^N \to \mathbb E_{\pi}[\varphi(m)]$ almost surely.

\paragraph{Step 3: Variance decomposition.}
Define the random variables
\begin{align}
    Z^i := \frac{\pi(m^i)}{q_k(m^i)}\,\varphi(m^i),
    \qquad m^i \sim q_k \text{iid}
\end{align}
Then
\begin{align}
    \frac{1}{N}\sum_{i=1}^N Z^i
    \;\xrightarrow{\text{a.s.}}\;
    \mathbb E_{\pi}[\varphi(m)].
\end{align}
Moreover,
\begin{align}
    \operatorname{Var}\!\left(\hat I^N\right)
    = \frac{1}{N}\operatorname{Var}(Z^1).
\end{align}

\paragraph{Step 4: Variance goes to zero.}
By Assumption~(ii) (uniformly bounded importance weights) and boundedness of $\varphi$, there exists $C<\infty$ such that
\begin{align}
    \left|\frac{\pi(m)}{q_k(m)}\,\varphi(m)\right| \le C
    \quad \text{for all } m.
\end{align}
Hence $\operatorname{Var}(Z^1) \leq C^2$, and therefore
\begin{align}
    \operatorname{Var}\!\left(\hat I^N\right)
    = \operatorname{Var}\!\left(\frac{1}{N}\sum_{i=1}^N Z^i\right)
    \leq \frac{C^2}{N}
    = \mathcal O(1/N) \xrightarrow{} 0.
\end{align}

\paragraph{Step 5: Convergence in probability (Chebyshev).}
For any $\varepsilon>0$, Chebyshev’s inequality yields
\begin{align}
    \mathbb P\!\left(
    \left|
    \hat I^N - \mathbb E_{\pi}[\varphi(m)]
    \right| \ge \varepsilon
    \right)
    &\le
    \frac{\operatorname{Var}\!\left(\hat I^N\right)}{\varepsilon^2} 
    \le \frac{C^2}{N\varepsilon^2},
\end{align}
which converges to zero as $N\to\infty$.
This establishes convergence in probability of the importance sampling estimator to $\mathbb E_{\pi}[\varphi(m)]$.

\paragraph{Step 6: Effect of resampling.}
The resampling step replaces the weighted particle system by an equally weighted one whose empirical measure is unbiased.
Under Assumption~(iii) (conditional independence of propagated particles given resampling) and bounded importance weights, resampling does not change the limit of the estimator and introduces additional variance of order $\mathcal O(1/N)$ only.
Therefore, resampling preserves convergence in probability and the Monte Carlo rate.
When adaptive resampling is used, i.e., resampling is triggered only when $\mathrm{ESS} < \tau_{\mathrm{ESS}}$, the argument extends directly. On iterations where resampling is skipped, particles retain their current weights and ancestry, which is part of the path-space reduction in Step~7 without introducing additional bias.

\paragraph{Step 7: Importance sampling vs.\ SMC (path-space argument).}
ModelSMC performs SMC updates with incremental weights of the form
\begin{align}
    w_k^i
    = \frac{\pi_k(m_k^i)}
           {\pi_{k-1}(m_{k-1}^i)\,q_k(m_k^i \mid m_{k-1}^i)},
\end{align}
where $\{\pi_k\}$ is an artificial sequence of target distributions whose marginal at iteration $k$ is $\pi$.
Viewing SMC as importance sampling on the extended path space $(m_0,\dots,m_k)$, the product of incremental weights along a trajectory equals the standard importance weight $\pi(m_k)/q_k(m_k)$ for the marginal distribution.
Hence, analyzing the corresponding importance sampling estimator suffices to establish consistency of the SMC estimator.

\paragraph{Step 8: Conclusion.}
Combining the above steps, we obtain that
\begin{align}
    \sum_{i=1}^N w^i\,\varphi(m^i)
    \;\xrightarrow{\mathbb P}\;
    \mathbb E_{\pi}[\varphi(m)]
    \quad\text{as } N\to\infty,
\end{align}
with asymptotic variance $\mathcal O(1/N)$.
Thus, ModelSMC provides a consistent Monte Carlo approximation of the posterior $p(m\mid\vx_o)$.
\end{proof}

\textbf{Remarks.}
\begin{itemize}
    \item \textbf{Mixture proposal.} The mixture kernel in \cref{alg:modelsmc} clones a fraction of particles while propagating the remainder via the LLM. Since cloned particles retain their previous weights and the overall kernel satisfies the support coverage assumption as $N\to\infty$, this does not affect the validity of the convergence argument.
    \item \textbf{Idealized setting.} The theorem characterizes the asymptotic behavior of an idealized ModelSMC sampler with exact importance weights. The implemented algorithm uses approximate weights due to implicit, context-conditioned LLM proposals, and should therefore be interpreted as an approximation to this idealized sampler.
    \item \textbf{Prompt engineering and proposal quality.}
    Prompt design affects the proposal distribution implicitly by shaping the inductive biases of the LLM.
    From the perspective of the theorem, prompt engineering serves as a practical mechanism to improve support coverage and reduce weight variance, analogous to hand-crafted proposal kernels in traditional Monte Carlo methods.
    It influences finite-sample efficiency but does not affect the asymptotic consistency result.
\end{itemize}

\appsection{Method Details}

\appsubsection{Note on the Definition of Prior Distributions}
\label{app:ExperimentalDetailsPriorChoices}

The dimensionality of the parameter vectors $\theta$ and the interpretation of the individual entries may vary across models. Therefore, the prior $p(\theta)$ is model- and task-specific, and is defined by the scientist using domain knowledge to identify meaningful parameter ranges, defaulting to uniform priors over physically plausible bounds, or incorporating knowledge from literature. While this is the approach taken in our experiments, an alternative is to jointly generate models and priors in the discovery experiment. For example, \citet{holt2025gsim} synthesize the bounds of uniform priors as part of the models.

The experimental conditions $c$ (e.g., the stimulus current and initial voltage in Hodgkin-Huxley (\cref{sec:exp-allen}) or the initial population sizes in the SIR task (App.~\ref{app:ExperimentalDetails_SIR})) are observable by definition. Therefore, a categorical distribution over observed training values is always available as a minimal choice for $p(c)$. Additional information about the experimental setting can provide more sophisticated distributions.

\appsubsection{Prompting Strategy}%
\label{app:LLMDetails}

In the propagation step of ModelSMC (\cref{alg:modelsmc} and \cref{sec:Method_ModelSMC}), an LLM is used as a conditional sampler for implementations of models, implicitly synthesizing new models. To obtain meaningful implementations, we condition the LLM on text-based conditions incorporating information about previously generated models and the task to solve.

Passing information about models generated earlier in the discovery process, especially their implementations, to the LLM, is a well-established strategy in LLM-based model discovery \citep{holt2025gsim, Castro2025cogfunsearch, romera-paredes_mathematical_2024, agarwal2025autodiscovery}. We follow this general scheme by providing the ancestry of implementations $m_{k-1}^{a_k^{i}}$ and the ancestry of performance feedback $c_{k-1}^{a_k^{i}}$ to the LLM. Both rely on the ancestry of the resampled particle. Injecting ancestral information in the propagation step enables the LLM to reason about the effect of changes to the models on their performance. %

Furthermore, we pass a task-specific context $c_\text{task}$ to the LLM, which does not change over the course of the discovery run. The task-specific context $c_\text{task}$ contains all information about the task. By passing $c_\text{task}$ to the LLM, we aim at triggering the LLM's broad knowledge base to further improve its ability to provide implementations that solve the given task. ModelSMC structures $c_\text{task}$ into three parts:
\begin{enumerate}
    \item \textbf{System description.} We include all prior information about the task to solve. This includes information about the initial solution provided to ModelSMC or the metrics used to evaluate the generated models. The description also contains information about the scientific domain and about the data used to test the synthesized models.
    \item \textbf{Signature description.} We collect all information about the framework in which the generated models are evaluated. This includes details such as the interpretation and structure of the model's input arguments and outputs.
    \item \textbf{Task description.} This part summarizes the aim of the discovery run and describes how the tasks should be approached (i.e., to reason from previous attempts, use coding standards, or specific coding instructions).
\end{enumerate}
The different parts of $c_\text{task}$ can vary across different applications, depending on the amount of prior knowledge, the environment used to evaluate the code, the scientific domain, and the specifics of the task. For the full prompts, we refer to the accompanying code.

The generation of contextual feedback $c_{k}^{i}$ of a model $m_k^i$ in \cref{alg:modelsmc} is performed after model simulation, and numerical evaluation. In a single LLM call, the system description, the simulator code, and the numerical evaluation metrics are jointly provided to the LLM, which is instructed to simultaneously diagnose failure modes and propose concrete improvements. Any exceptions arising during evaluation (invalid code, numerical instability, non-terminating simulations, or timeouts) are captured and included in the performance metrics string passed to this call. The LLM returns structured feedback comprising a concise main diagnosis and up to two prioritized issues, each annotated with a severity level, a code location, and a targeted fix suggestion. The resulting feedback is designed to guide subsequent LLM propagation steps by highlighting concrete directions for model revision.

During model simulation and evaluation, exceptions may arise due to invalid code, numerical instability, non-terminating simulations, or timeouts. Such exceptions are explicitly caught and recorded. When an exception occurs, a structured description of the failure mode (e.g., error messages, timeout events, or invalid outputs) is included in the first-stage diagnosis and propagated to the LLM as part of the contextual feedback. This allows models that failed to receive targeted corrective feedback to be revised if they are later resampled as ancestors, rather than being irreversibly discarded. \Cref{fig:appendix_figure_prompting_strategy} visualizes the structure of the prompts used in the propagation and feedback steps of ModelSMC. An instantiation of the prompts is provided in App.~\ref{app:PromptExamples}.\newpage

\begin{figure}[H]
    \centering
    \includegraphics[width=0.75\textwidth]{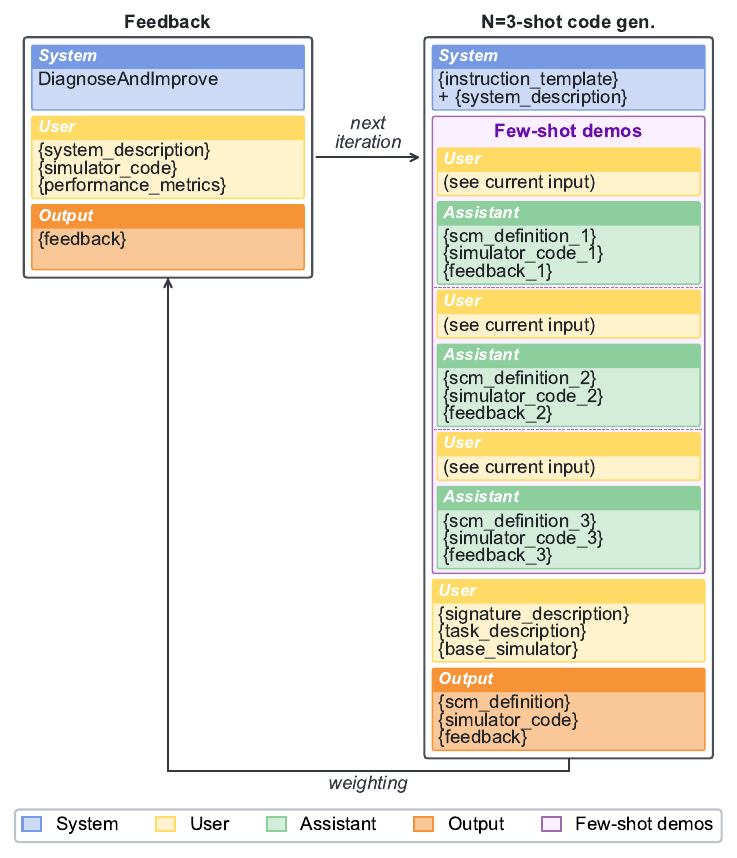}
    \caption{
        \textbf{End-to-end prompting example for ModelSMC.} Based on the numeric performance metrics computed in the weighting step, the implementation of the model, general instructions, and task-specific information, we obtain LLM-based feedback about the model (\textbf{left}). Based on general instructions, task-specific information, few-shot examples consisting of the implementation of particles together with their feedback, the structural causal model, and the corresponding user prompt, we generate a new model implementation and the corresponding structural causal model in the propagation step. The propagation step also yields a text-based feedback field for consistency with the few-shot example structure, though this feedback is not used. To avoid duplicating static text, we do not include the signature description, task description, and base implementation in the few-shot example messages, but instead include references to the final user message (\textbf{right}).
    }
    \label{fig:appendix_figure_prompting_strategy}
\end{figure}

\appsubsection{Parameter Estimation}
\label{app:Method_ParameterEstimation}

ModelSMC does not require the value of the external parameters $\theta$ for the model discovery process. Since many downstream applications require parameter estimation, we estimate the parameter values $\hat\theta$ that best explain the observed data. For an observation $x_o$ and associated context $c_o$, we estimate $\hat\theta_o$ by maximizing the posterior $p(\theta|  x, c, m)$ with respect to~$\theta$ as
\begin{equation}\label{eq:ParameterEstimationMaxLogPosterior}
    \hat\theta_o = \arg\max_{\theta\in\Theta} \log p\left(\theta\mid x_o, c_o, m\right).\\
\end{equation}
To simplify notation, we drop the conditioning on the discovered model for the remainder of this section.
The key ingredient for estimating $\hat\theta_o$ via \cref{eq:ParameterEstimationMaxLogPosterior} is access to the posterior $p\left(\theta| x, c\right)$, which in practice is usually intractable. Therefore, we use a surrogate model
\begin{equation}
    p_\psi(\theta\mid x, c)\approx p(\theta\mid x, c).
\end{equation}
One approach for obtaining $p_\psi$ is Neural Posterior Estimation \citep[NPE]{Lueckmann2017FlexibleStatisticalInference,papamakarios2018fastepsilonfreeinferencesimulation} which trains a density estimator $p_\psi$ from scratch using a training set of $n_{\text{sim}}$ synthetic samples
\begin{equation}\label{eq:TrainingSetNLENPE}
    \mathcal{D} = \{\left(\theta_j\sim p(\theta), c_j\sim p(c), x_j = m(c_j,\theta_j) \sim p(x\mid m(\theta, c)\right)\}_{j = 1}^{n_{\text{sim}}}.
\end{equation}
Once $p_\psi$ is trained, the expression in \cref{eq:ParameterEstimationMaxLogPosterior} is estimated by drawing $n_\text{post}$ samples from $p_\psi$ and selecting the sample that maximizes $p_\psi(\theta| x, c)$ 
\begin{equation}\label{eq:MaxLogPosteriorMAPEstimationSampleBased}
    \hat\theta_o = \arg\max_{\theta\in \Theta_o}\; \log p_\psi\left(\theta\mid x_o, c_o\right), \qquad \text{with} \quad \Theta_o = \left\{\theta_l\sim p_\psi\left(\theta\mid x_o, c_o\right)\right\}_{l = 1}^{n_\text{post}}.
\end{equation}
An alternative to NPE is NPE-PFN \citep{vetter2025effortless}. NPE-PFN employs TabPFN as a conditional density estimator. TabPFN \citep{hollmann2025tabpfn} is a Prior-data Fitted Network \citep{muller2022transformers} for tabular data, which provides one-dimensional density estimates and can be autoregressively extended for higher-dimensional densities. The key difference between NPE and NPE-PFN is that NPE-PFN does not train a new surrogate model but uses $\mathcal{D}$ (see, \cref{eq:TrainingSetNLENPE}) as an in-context dataset for the pretrained TabPFN model. Hence, NPE-PFN allows for training-free posterior estimation.

Following \citet{vetter2025effortless} and adapting to the case of additional external context $c_o$, the posterior can be decomposed into a product over the parameter dimensions as 
\begin{equation}\label{eq:FactorizationPosterior}
    p(\theta|x_o, c_o) \approx \prod_{j = 1}^{d_{\theta}}q_{\psi}\left(\theta_j\mid\theta_{<j}, x_o, c_o, \mathcal{D}_{<j}\right).
\end{equation}
For a $d_\theta$-dimensional parameter space, $\theta_j$ is the $j$-th component of $\theta$, and $\theta_{<j}$ the components of $\theta$ up to, but not including, the $j$-th. Analogously, $\mathcal{D}_{<j}$ is defined as $\mathcal{D}_{<j} = \left\{\theta_j^i,\left[\theta_{<j}^i, x^i, c^i\right]: \left(\theta^i, x^i, c^i\right)\in\mathcal{D}\right\}$. In the setting of NPE-PFN $q_{\psi}$ is a one-dimensional density estimate by TabPFN. To sample from the approximated posterior \cref{eq:FactorizationPosterior}, the individual dimensions of $\theta$ are sampled recursively.

The parameter estimation scheme \cref{eq:ParameterEstimationMaxLogPosterior} used for NPE can also be applied with NPE-PFN. For high-dimensional parameter spaces, however, the recursive sampling and density estimation (\cref{eq:FactorizationPosterior}) for $n_\text{post}$ posterior samples can slow down the discovery process. To reduce the computational costs of the parameter estimation, we exploit NPE-PFNs factorization of the posterior \cref{eq:FactorizationPosterior} and introduce a new strategy for estimating $\hat\theta_o$ which is based on temperature-scaling the posterior.

Temperature-scaling refers to exponentiating a density with an exponent of $\frac{1}{T}$, with the temperature $T>0$. Applying this transformation to a density leaves the positions of the modes, i.e., the local maxima, unchanged. Furthermore, for $T\ll 1$ the temperature-scaled density becomes localized around the mode with the highest density value in the original density. Sampling from the temperature-scaled distribution accordingly produces samples that have a high density under the original distribution. 

Sampling from the temperature-scaled posterior at $T\ll 1$ instead of the posterior itself can significantly reduce the necessary amount of posterior samples $n_\text{post}$ in \cref{eq:MaxLogPosteriorMAPEstimationSampleBased} as all samples are close to the posterior mode with the highest density value. Therefore, we temperature-scale the posterior to obtain
\begin{equation}\label{eq:FactorizationPosteriorTemperatureScaled}
    p(\theta|x_o, c_o)^{\frac{1}{T}} \approx \prod_{j = 1}^{d_{\theta}}q_{\psi}\left(\theta_j\mid\theta_{<j}, x_o, c_o, \mathcal{D}_{<j}\right)^{\frac{1}{T}}.
\end{equation}
Crucially, \cref{eq:FactorizationPosteriorTemperatureScaled} only requires to power-scale the one-dimensional factors in the product rather than the full, potentially high-dimensional, posterior itself. In our experiments, we draw one sample from the temperature-scaled posterior and use it as $\hat\theta_o$. Drawing only one sample from the temperature-scaled posterior reduces the computational cost of the parameter estimation from $n_\text{post}$ sampling and density estimation steps to a single sampling step. Therefore, this parameter estimation approach exploits the factorization of the posterior in NPE-PFN and significantly reduces the computational costs of the parameter estimation.

\appsubsection{Estimation of Resampling Weights: Likelihood Estimation}%
\label{app:Method_WeightEstimation}

\cref{eq:MarginalLikelihoodSingleInstance} defines the resampling weight of a particle as the likelihood of the observed data $x_o$ under the model $m$, marginalized over the external parameters $\theta$. In practice, the likelihood $p(x| \theta,c,m)$ is intractable. Therefore, we use a surrogate model $p_{\phi}(x| \theta,c,m) \approx p(x| \theta,c,m)$ in the computation of \cref{eq:MarginalLikelihoodSingleInstance}. For simplicity, we drop the conditioning on $m$ for the remainder of this section. Throughout, Monte Carlo samples are drawn iid unless otherwise noted. 

One approach to learn a surrogate model is to apply Neural Likelihood Estimation \citep[NLE]{papamakarios2019sequential}. NLE uses the training set $\mathcal{D}$ (\cref{eq:TrainingSetNLENPE}) to train the surrogate model $p_{\phi}(x| \theta,c,m)$ via maximum likelihood training. An alternative to NLE is to adapt NPE-PFN \citep{vetter2025effortless} to likelihood estimation. We denote this approach as NLE-PFN. Compared to \cref{eq:FactorizationPosterior}, NLE-PFN effectively interchanges $\theta$ and $x$, yielding 
\begin{equation}\label{eq:FactorizationLikelihood}
    p(x|\theta, c) \approx \prod_{\ell = 1}^{d_{x}}q_{\psi}\left(x_\ell\mid x_{<\ell}, \theta, c, \mathcal{B}_{<\ell}\right),
\end{equation}
where $d_{x}$ denotes the dimensionality of the observation space, $x_\ell$ the $\ell$-th component of the observation, and $x_{<\ell}$ the components of $x$ up to, but not including, the $\ell$-th. The in-context dataset $\mathcal{B}$ is defined as $\mathcal{B}_{<j} = \left\{x_\ell^i,\left[x_{<\ell}^i, \theta^i, c^i\right]: \left(\theta^i, x^i, c^i\right)\in\mathcal{D}\right\}$. As in NPE-PFN, each factor in \cref{eq:FactorizationLikelihood} is estimated by a one-dimensional density estimate from TabPFN \citep{hollmann2025tabpfn}.

With the likelihood surrogate model $p_{\phi}(x\mid \theta,c,m)$, \cref{eq:MarginalLikelihoodSingleInstance} can be estimated by the Monte Carlo estimate 
\begin{equation}\label{eq:appendixMCEstimateMarginal}
    p(x_o^j\mid m^i_k)\approx\frac{1}{B}\sum_{b = 1}^B p(x_o^j\mid m_k^i, \theta_b,c_o^j), \qquad \theta_b\overset{\mathrm{iid}}{\sim}p\left(\theta\right), \qquad \forall b = 1, ..., B.
\end{equation}

If a (neural) posterior estimate $p_{\psi}\left(\theta\mid x\right)$ is available, $p(x_o^j\mid m^i_k)$ can also be approximated by using standard importance weighting and samples drawn from the posterior distribution rather than from the prior 

\begin{equation}\label{eq:appendixMCEstimateMarginalPosterior}
    p(x_o^j\mid m^i_k)\approx\frac{1}{B}\sum_{b = 1}^B\frac{p_{\psi}\left(\theta_b\mid x_o^j\right)}{p(\theta_b)}p(x_o^j\mid m_k^i, \theta_b,c_o^j), \qquad \theta_b\overset{\mathrm{iid}}{\sim}p_{\psi}\left(\theta\mid x_o^j\right), \qquad \forall b = 1, ..., B.
\end{equation}

We find that for small $B$, \cref{eq:appendixMCEstimateMarginalPosterior} has a lower variance than \cref{eq:appendixMCEstimateMarginal}. This reduced variance however, comes at the cost of requiring a pretrained NPE and additional density evaluations to compute the importance weights for each parameter sample.

For high-dimensional observations, we can use summary statistics $s: X\rightarrow \mathbb{R}^d$ with $d\ll d_x$ \citep{Lueckmann2017FlexibleStatisticalInference}. Assuming that $s(x)$ retains most of the relevant information about $x$, the approximation of the likelihood changes to 
\begin{equation}
    p(x|\theta, c) \approx p_\phi(s(x)|\theta, c).
\end{equation}
Evaluating \cref{eq:appendixMCEstimateMarginal} for each observed instance $(x_o^j,c_o^j)\in\vx_o$ allows us to compute the unnormalized resampling weight $\tilde w_k^i$ (\cref{eq:ResamplingWeightModelSMCFactorization}) of a model $m_k^i$. In practice, we use the slightly adapted definition \cref{eq:TenmperatureScaledResamplingWeight} of $\tilde w_k^i$ of the form
\begin{equation}\label{eq:TenmperatureScaledResamplingWeight}
    \tilde w_k^i \propto \exp\left(\frac{1}{\tau}\log\left(\prod_{j=1}^M p(x_o^j\mid m_k^i)\right)\right).
\end{equation}
The temperature parameter $\tau$ in \cref{eq:TenmperatureScaledResamplingWeight} controls the bias of the resampling step. For $\tau\ll1$, particles explaining the data well are favored, whereas for $\tau\gg1$, the selection becomes more uniform across all particles. Unless stated otherwise, the temperature is set to $\tau=1$ in all experiments. Importantly, the temperatures $T$ and $\tau$ used in \cref{eq:FactorizationPosteriorTemperatureScaled} and \cref{eq:TenmperatureScaledResamplingWeight} are two distinct hyperparameters.

\appsection{Baseline Methods}

\appsubsection{FunSearch+}
\label{app:baselines_FunSearchPlus}

FunSearch \citep{romera-paredes_mathematical_2024} is an LLM-based evolutionary method to evolve function implementations to improve their performance with respect to an evaluation function. The program database used by FunSearch to store previously generated functions consists of several populations of functions, referred to as islands. Periodically, the worst-performing islands are discarded and reinitialized with a function from the remaining well-performing islands. When generating the prompt to synthesize a new function, FunSearch randomly selects an island and samples $k$ functions from this island with sampling probabilities based on their score and length. The selected functions are sorted by their scores and inserted into the program skeleton to provide ranked examples to the LLM. Based on this prompt, the LLM is tasked to implement a new function with the goal of generating a higher-scoring function. Newly generated functions are added to the same island from which the functions in the prompt were selected. Therefore, the islands evolve independently between reinitialization steps. A key difference between ModelSMC and FunSearch is that FunSearch selects the functions in the prompt based on the score but does not consider ancestral relations. 

As originally proposed for problems from extremal combinatorics, FunSearch does not allow the use of external model parameters $\theta$. We therefore include the ability to use and estimate model parameters in FunSearch and denote this extended version as FunSearch+. Including external model parameters into FunSearch is conceptually inspired by \citet{Castro2025cogfunsearch}. \citet{Castro2025cogfunsearch} applies gradient descent-based optimization on a likelihood-based criterion to estimate the values of the model parameters. The score of a discovered model is given by the validation performance of this criterion, given the optimized parameters. Instead of using gradient-based parameter optimization, we use the parameter estimation approach based on NPE and NPE-PFN (details in App.~\ref{app:Method_ParameterEstimation}). The advantage of using NPE and NPE-PFN to estimate $\hat\theta$ over gradient-based approaches is to obtain posterior estimates for the parameters without requiring differentiable models. Inspired by the use of a likelihood-based score in \citet{Castro2025cogfunsearch}, we also use a likelihood-based score in FunSearch+, which we define as
\begin{equation}
    s =\log\prod_{j=1}^M p_{\phi}(x_o^j\mid m,c_o^j, \hat\theta^j).\label{eq:appendix_FunSearch_likelihood_score}
\end{equation}
In \cref{eq:appendix_FunSearch_likelihood_score}, $\hat\theta^j$ denotes the parameter estimate for the observed instance $(x_o^j,c_o^j)\in\vx_o$ and $p_{\phi}(x|m,c,\theta)$ is a surrogate likelihood (described in App.~\ref{app:Method_WeightEstimation})for the unknown true likelihood $p(x|m,c,\theta)$. 

We base our implementation of FunSearch+ on the openly available FunSearch implementation provided by \citet{RayZhhh2024funsearchimplementation}. Our main modifications include evolving class methods rather than standalone functions and including our parameter estimation (App.~\ref{app:Method_ParameterEstimation}) and scoring schemes (\cref{eq:appendix_FunSearch_likelihood_score}). \citet{RayZhhh2024funsearchimplementation} relies on string-based modifications of functions and is therefore tightly coupled to Python code. As a result, unlike ModelSMC, it cannot be readily extended to other programming languages.

\appsubsection{ModelSMC $N$=1}
\label{app:baselines_ModelSMConeparticle}

This baseline method represents alternative LLM-based approaches for model discovery that explore model space along a single trajectory \citep{holt2025gsim,Ma2024SGA,Li2024Boxlm,shojaee2025llmsr}. ModelSMC $N$=1 uses the same implementation as ModelSMC, but employs a population size of one and always propagates the most recently generated particle (i.e. setting $\alpha=0$ in \cref{eq:AlgoModelSMCPropagation}).

\appsection{Experimental Details}\label{app:ExperimentalDetails}

Here, we provide details on the exact configurations used in the experiments. For a complete list of all settings and details on the implementation, we refer to the accompanying code.

\appsubsection{Validation of ModelSMC for Model Discovery}\label{app:ExperimentalDetails_MinimalExample}

The experiment presented in \cref{sec:experiments_ValidationNoLLM} tests the probabilistic inference mechanism underlying ModelSMC without using LLMs to synthesize new models. To implement this, we restrict ModelSMC to select models from a finite model space consisting of 20 predefined candidate models with observation dimension $d_x=10$.

\subsubsection{Candidate models}

The candidate models used in this experiment are Gaussian Mixture Models (GMMs). Given the number of mixture components $n_{\text{comp}}\in\mathbb{N}$, a set of covariance matrices $\left\{\Sigma^i\in\mathbb{R}^{d_x\times d_x}\right\}_{i =1}^{n_{\text{comp}}}$, mean vectors $\left\{\mu^i\in\mathbb{R}^{d_x}\right\}_{i =1}^{n_{\text{comp}}}$, and weights $\left\{w^i\in\mathbb{R}_+\right\}_{i =1}^{n_{\text{comp}}}$ with $\sum_i w^i = 1$, we define a GMM as
\begin{equation}\label{eq:appendix_GMMDefinition}
    p_{\text{GMM}}(x\mid s, u) = \sum_{i = 1}^{n_{\text{comp}}}w^{i}\cdot\frac{1}{\left(2\pi\right)^{d_x/2}\det\left(s^2\Sigma^i\right)^{1/2}}\exp\left(-\frac{1}{2s^2}\left(x-\left(s\mu^i+u\right)\right)^T\left(\Sigma^i\right)^{-1}\left(x-\left(s\mu^i+u\right)\right)\right).
\end{equation}
Here, $s$ and $u$ are the external parameters of the model: the scale parameter $s$ is scalar and follows a uniform prior over the interval $[0.1,2]$. The shift parameter $u$ is a $d_x$-dimensional vector, where all but four components are set to zero. Each of the four non-zero components follows a uniform prior over the interval $[-2,2]$. The non-zero components are sampled uniformly at random when initializing the candidate model and remain unaltered over the discovery run. The external parameters allow to scale and to shift the GMM along selected components to fit it to the observed data. For each candidate model, the configuration of the underlying GMM is randomly initialized as
\begin{align}
    n_{\text{comp}} &\sim \text{Unif}\left(\left\{1, 2, ..., 10\right\}\right),\\
    \tilde w^i &\sim \text{Unif}\left(\left[0,1\right]\right), \qquad w^i = \frac{\tilde w^i}{\sum_{j=1}^{n_{\text{comp}}}\tilde w^j}, \qquad \forall i =1,...,n_{\text{comp}},\\
    \mu^i_j&\sim\text{Unif}\left(\left[-5,5\right]\right), \qquad \forall j=1,...,d_x, \qquad \forall i =1,...,n_{\text{comp}},\\
    \tilde \Sigma^i_{jk}&\sim\text{Unif}\left(\left[-2,2\right]\right), \qquad \forall j,k=1,...,d_x, \qquad \forall i =1,...,n_{\text{comp}},\\
    \Sigma^i &= \alpha^i\cdot\tilde\Sigma^i\left(\tilde\Sigma^i\right)^T, \qquad \alpha^i\sim\text{Unif}\left(\left[1,2\right]\right), \qquad \forall i =1,...,n_{\text{comp}}.
\end{align}
In each discovery run, one of the 20 candidate models is selected as the ground truth model $m^*$, and observed data is drawn from $m^*$ with $s = 1$ and $u = \mathbf{0}_{d_x}$, i.e., no scaling and shifting is applied.

\subsubsection{Experimental configuration}

Five different target models are selected from the set of 20 candidate models. For each target model, ten discovery runs are conducted, each with a different random seed. The set of candidate models is the same for all discovery runs.

ModelSMC's resampling weights are estimated using NLE \citep{papamakarios2019sequential} with a training set $\mathcal{D}$ of size 5,000 (App.~\ref{app:Method_WeightEstimation}). No additional context is used, i.e., a training instance consists of parameters $\theta = (s,u)$ drawn from the prior and a simulation $m(\theta)$. The observed set $\vx_o$ is of size 1,000 and is drawn from the target model. To compute the Monte Carlo estimate \cref{eq:appendixMCEstimateMarginal}, $B=2,000$ prior samples are used. We employ sbi's \citep{BoeltsDeistler_sbi_2025} default hyperparameters for NLE but limit the maximum number of training epochs to 500.

All discovery runs proceed for 20 iterations and the probability for copying a particle in the propagation step is $\alpha = 0.8$ (\cref{eq:AlgoModelSMCPropagation}). The LLM-based propagation kernel component is replaced by uniformly sampling a random model from the candidate set. The population contains $N=100$ particles. At initialization of the population, the resampling weights are computed for all 20 candidate models. The population is then initialized by copying each model five times. In addition to not using an LLM in the propagation step, no text-based performance feedback (App.~\ref{app:LLMDetails}) is generated.

\appsubsection{SIR Model}\label{app:ExperimentalDetails_SIR}

We apply ModelSMC to a synthetic dataset where the ground truth simulator model is known. In particular, we use a simple epidemiological SIR model \citep{Kermack1927SIR}. The SIR model considers three populations---Susceptible, Infected, Recovered---with interactions controlled by two external parameters, a base transmission rate $\beta$ and a recovery probability $\gamma$. In this experiment, we use a time-discretized adaptation of the classical SIR model from \citet{holt2025gsim}.

\subsubsection{Dataset}

We use the implementation of the ground truth model and prior distributions of \citet{holt2025gsim} to generate observed data. Synthetic data is generated by drawing the initial numbers of susceptible, infected, and recovered individuals from the prior distributions specified as
\begin{align}
    S_0\sim \text{Unif}\left(\left\{50, 51, ..., 99, 100\right\}\right), \qquad
    I_0\sim \text{Unif}\left(\left\{1, 2, ..., 4, 5\right\}\right), \qquad
    R_0= 0.
\end{align}
These initial population sizes define the context $c = (S_0, I_0, R_0)$. Given the external parameters $\beta$ and $\gamma$, the population sizes at time step $t \in \left\{1, ..., T\right\}$ are computed recursively from those at time $t-1$. Let $N = S_{t-1} + I_{t-1} + R_{t-1}$ and $p_{\text{inf}} = 1.0 - \exp(-\beta\cdot I_{t-1} / N)$. New infections and recoveries are drawn as
\begin{align}
    N_{\text{inf}}\sim\text{Binomial}\left(S_{t-1},p_{\text{inf}}\right), \qquad
    N_{\text{rec}}\sim\text{Binomial}\left(I_{t-1},\gamma\right),
\end{align}
and the state is updated via
\begin{align}
    S_t = S_{t-1} - N_{\text{inf}}, \qquad
    I_t = I_{t-1} + N_{\text{inf}} - N_{\text{rec}}, \qquad
    R_t = R_{t-1} + N_{\text{rec}}.
\end{align}
Because infections and recoveries are sampled from Binomial distributions, the transitions across time steps are stochastic.

In total, the simulation runs for $T=60$ time steps, yielding three time series of 61 time steps. The prior distributions for the external parameters are given by
\begin{align}
    \gamma \sim \text{Unif}\left([0,1]\right), \qquad
    \beta \sim \text{Unif}\left([0,2]\right).
\end{align}
In our experiments, these priors are fixed for all generated models and are not treated as part of the model. In contrast to \citet{holt2025gsim}, we do not use an LLM to synthesize the dimensionality of the parameter space and the bounds of the uniform priors. Following their NPE-based parameter estimation scheme, a single training trajectory drawn from the ground truth model is used to estimate $\beta$ and $\gamma$. For computations of evaluation metrics, including the resampling weights, we use $\vx_o$ containing 100 samples drawn from the same generative process.

\subsubsection{Experimental Configuration}

Parameter estimation for $\beta$ and $\gamma$ is performed using NPE  with a training set of 1,000 synthetic instances (App.~\ref{app:Method_ParameterEstimation}). Each synthetic training sample consists of a context $c=(S_0,I_0,R_0)$ drawn from the initial condition priors, parameters $\theta = (\gamma,\beta)$ drawn from their uniform priors, and an observation $x=m(\theta,c)$. The observation is a 183-dimensional vector formed by concatenating the simulated time series of susceptible, infected, and recovered populations.

After training, parameter values are inferred using $10^4$ samples drawn from the approximate posterior $p_{\psi}(\theta|x_o,c_o,m)$ using maximum log-posterior (MAP) criterion (\cref{eq:MaxLogPosteriorMAPEstimationSampleBased}). Here, $x_o$ denotes the single observed trajectory generated from the ground-truth data-generating model and $c_o$ the corresponding initial condition. The resulting parameter estimate for $\beta$ and $\gamma$ is fixed and reused for all downstream computations. No summary statistics are applied to the observations. 

This estimation procedure is inspired by \citet{holt2025gsim} but differs in multiple aspects: In our implementation, the initial condition is explicitly treated as a condition in NPE. Moreover, we normalize the observations and the initial conditions elementwise by the total population size before passing them to the density estimator. Finally, we use a MAP estimate based on posterior samples, instead of averaged posterior samples, which can lead to poor parameter estimates for multimodal posteriors. 

Resampling weights for ModelSMC are computed using NLE (App.~\ref{app:Method_WeightEstimation}). The training set for NLE is identical to the normalized training set used for NPE in the parameter estimation step. We approximate the marginal likelihood (\cref{eq:appendixMCEstimateMarginal}) using 5,000 prior samples. ModelSMC runs for 15 iterations. In the propagation kernel (\cref{eq:AlgoModelSMCPropagation}), existing particles are copied unaltered with a probability of $\alpha=0.8$. The particle population size is $N{=}50$. Therefore, in expectation $N \cdot (1-\alpha) = 10$ new particles are generated per iteration.

The prompts in $c_\text{task}$ (App.~\ref{app:LLMDetails}) are constructed by adapting the prompts used in \citet{holt2025gsim} to match the specifics of ModelSMC. We follow the same procedure for the program skeleton. Components predicting prior bounds are removed, and the initial simulator model is restricted to return constant population sizes.

For ModelSMC $N$=1 (App.~\ref{app:baselines_ModelSMConeparticle}), 150 models are generated in total, matching the expected number generated by ModelSMC across iterations. FunSearch+ (App.~\ref{app:baselines_FunSearchPlus}) uses the same evaluation pipeline as the other two approaches. Since we find that FunSearch+ requires less LLM tokens per generated model compared to ModelSMC and ModelSMC $N$=1, we increase the total number of generated models for FunSearch+ to 190 in order to match the token budget of the other methods. All methods provide the LLM with up to three previously generated implementations (i.e., ancestors). To compute the results in \cref{tab:quantitative_results}, all methods are run with ten different initial random seeds. The two baseline methods use the same parameter estimation scheme as ModelSMC. For implementation details, we refer to the accompanying code.

\appsubsection{Pharmacological Kidney Model}
\label{app:ExperimentalDetails_kidney}

\subsubsection{Dataset}
The dataset is originally from \citet{dluhy1972studies}; we use the version provided with the manuscript from \citet{Maddah2022}. The data contains measured plasma potassium and aldosterone responses to a standardized potassium chloride infusion in healthy humans under controlled dietary conditions. Subjects were assigned to four diet groups defined by low/high potassium (40 vs. 200 mEq/day) and low/high sodium (10 vs. 200 mEq/day) intake for 6–10 days, then received potassium chloride at 0.62 mEq/min for 120 min followed by 180 min recovery. Plasma potassium and aldosterone were sampled at baseline, during infusion (0, 30, 90, 120 min), and after recovery (300 min). This constitutes a total of 40 unique data points. There are no summary statistics used.

As context $c$ is passed to the model, we give the initial condition of the model. The initial condition is derived from steady-state conditions with normal sodium and potassium intake. The model parameters $\theta$ are drawn from uniform priors. For details about the bounds of the prior, we refer to the accompanying code.

\subsubsection{Experimental Configuration}
ModelSMC is run for 15 iterations with cloning probability $\alpha{=}0.8$ (\cref{eq:AlgoModelSMCPropagation}). The particle population contains $N{=}50$ particles. Parameter estimation is performed using NPE-PFN with temperature-scaled posteriors (App.~\ref{app:Method_ParameterEstimation}), using temperature $T=10^{-5}$ and 2,000 simulator calls per evaluation. Resampling weights are estimated using NLE-PFN (App.~\ref{app:Method_WeightEstimation}), and the marginal likelihood estimate (\cref{eq:appendixMCEstimateMarginalPosterior}) is approximated with $B=500$ posterior samples. In \cref{eq:appendixMCEstimateMarginalPosterior}, NPE-PFN is used for importance weighting and sampling proposals. The in-context dataset used for NPE-PFN is constructed from the same simulations as the in-context dataset of NLE-PFN. In each propagation step, up to three previously discovered models are provided as few-shot context to the LLM (App.~\ref{app:LLMDetails}).

For ModelSMC $N$=1, the same parameter and weight estimation settings are used. ModelSMC $N$=1 is run for 150 iterations. All experiments are performed using fixed, hand-crafted summary statistics and are evaluated using the negative log marginal likelihood estimated by NLE-PFN.

\appsubsection{Hodgkin-Huxley Model}
\label{app:ExperimentalDetails_allen}

\subsubsection{Dataset}

We consider data from the Allen cell types database \citep{allen_cell_types_2015}, which consists of 10 in-vitro recordings from a mouse cortex. This data was previously used for parameter inference \citep{gonccalves2020training,vetter2025effortless}. We use the same set of seven summary statistics as previously proposed in \citet{gonccalves2020training}: Spike count, mean, and standard deviation of the resting potential, and the first four voltage moments, i.e., mean, standard deviation, skew and kurtosis.

The additional context $c$ passed to the model is a two-dimensional vector containing the initial membrane voltage $V_0$ and the input current $I$ to the cell. The prior over the conditions is defined by uniformly drawing the context values from the observed set $\vx_o$. The two context dimensions are drawn independently, resulting in 100 possible combinations of initial voltage and input current. Model parameters $\theta$ are 10-dimensional vectors with components independently drawn from uniform priors, in the same way as provided in \citet{gonccalves2020training}. For details about the bounds of the prior, we refer to the accompanying code. 

\subsubsection{Experimental Configuration}

ModelSMC is run for 15 iterations with cloning probability $\alpha{=}0.8$ (\cref{eq:AlgoModelSMCPropagation}). The particle population has size of $N{=}50$. For parameter estimation, we employ NPE-PFN-based parameter estimation utilizing temperature-scaling of the posterior (App.~\ref{app:Method_ParameterEstimation}). The temperature used for temperature-scaling (\cref{eq:FactorizationPosteriorTemperatureScaled}) is $T = 10^{-5}$. NPE-PFN uses an in-context dataset~$\mathcal{D}$ (\cref{eq:TrainingSetNLENPE}) of size 5,000. Resampling weights, i.e. likelihood estimates, are obtained using NLE-PFN (App.~\ref{app:Method_WeightEstimation}), and \cref{eq:appendixMCEstimateMarginalPosterior} is estimated with $B=500$ posterior samples. In \cref{eq:appendixMCEstimateMarginalPosterior}, NPE-PFN is used for importance weighting and sampling proposals.  NLE-PFN and NPE-PFN are applied to seven standard summary statistics \citep{gonccalves2020training} rather than to the full voltage traces. In each propagation step, up to three ancestral samples are prompted to the LLM (App.~\ref{app:LLMDetails}).

For ModelSMC $N$=1 (App.~\ref{app:baselines_ModelSMConeparticle}), we use the same configuration for parameter and weight estimation as in ModelSMC. ModelSMC $N$=1 runs for 150 iterations. FunSearch+ (App.~\ref{app:baselines_FunSearchPlus}) generates 180 models and uses up to three previously synthesized models to construct the prompts used for propagation. Per prompt, four new models are generated. The program database consists of ten islands, and reinitialization of the worst-performing islands is performed every four hours. To estimate the parameters of the generated models and to compute the score (\cref{eq:appendix_FunSearch_likelihood_score}), NPE-PFN and NLE-PFN are used with the same settings as in ModelSMC and ModelSMC $N$=1. \Cref{eq:appendixMCEstimateMarginalPosterior} is estimated with $B=2,000$ posterior samples. All three methods are run with ten different initial random seeds to obtain the results presented in \cref{tab:quantitative_results}.

\appsubsection{Posterior Analysis of Hodgkin-Huxley Experiment}
\label{app:PosteriorAnalysis_allen}

To characterize posterior mass across all discovered models, we classify all $1{,}445$ usable particles (those with finite log-weight) from the Hodgkin-Huxley experiment (\cref{sec:exp-allen}) using the three-step LLM-assisted pipeline described in \cref{sec:exp-posterior}.

\paragraph{Step 1: Taxonomy construction.}
A stratified sample of approximately 15\% of all particles (across quality quartiles defined by weight) is used to extract free-form descriptions of the ion channels added beyond the three base Hodgkin-Huxley currents (fast Na$^+$, delayed-rectifier K$^+$, ohmic leak) via an LLM call.
A second LLM call consolidates all descriptions into a two-level taxonomy: multiple families (I$_M$, I$_{NaP}$, I$_h$, I$_A$, and multi-channel combinations) with 11 structural subtypes, split by activation gate power, presence of an inactivation gate, and time-constant implementation (voltage-dependent, fixed scalar, log-scaled, or instantaneous).

\paragraph{Step 2: Classification.}
All $1{,}445$ particles are classified in batches via an LLM call per batch, yielding a \texttt{family\_id} and \texttt{subtype\_id} for each particle.
The resulting family distribution is: I$_M$ = $1{,}084$ (75.0\%), I$_M$+I$_{NaP}$ = 216 (14.9\%), I$_M$+I$_h$ = 80 (5.5\%), I$_M$+I$_A$ = 35 (2.4\%), base only = 10 (0.7\%), unknown = 11 (0.8\%).

\paragraph{Step 3: Weight analysis.}
The distribution of weights per ion-channel subtype is shown in \cref{fig:results-posterior}. Note that we omit showing subtypes with less than 8 models belonging to that subtype, i.e. we result with 8 subtypes in \cref{fig:results-posterior}.
\Cref{fig:results-posterior-appendix} shows the within-seed rank of each ion-channel family across all 10 independent seeds.
I$_M$ ranks first in 9 out of 10 seeds, confirming that the preference for M-type channel extensions is a stable posterior conclusion and not an artefact of a single run.

\begin{figure}[ht]
    \centering
    \includegraphics{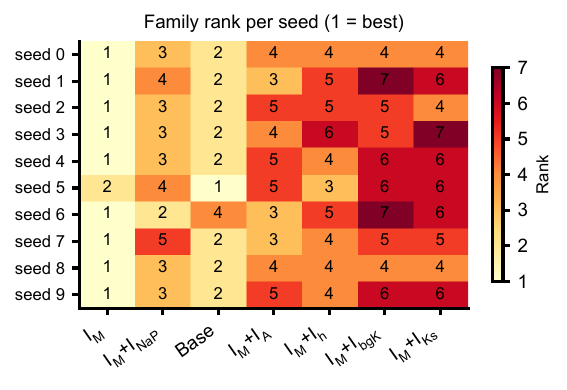}
    \caption{\textbf{Cross-seed rank stability of ion-channel families.}
    Each row is one seed and each column is an ion-channel family; cell colour and number give the within-seed rank of that family (1 = best; brighter = better).
    }
    \label{fig:results-posterior-appendix}
\end{figure}

\newpage
\appsubsection{Ablation Study of Hodgkin-Huxley Experiment}
\label{app:ablation_study_allen}

We systematically vary six design axes of ModelSMC on the Hodgkin-Huxley task (\cref{sec:exp-allen,sec:exp-ablation}). The default configuration uses Claude Sonnet 4.6, $N{=}50$ particles, $K{=}15$ iterations, $5{,}000$ simulations for posterior and likelihood estimation, marginal-likelihood weights, the full prompt, and LLM-only narrative feedback. Each ablation modifies exactly one component. The default is run with ten independent seeds; all variants use five seeds. The total budget in new particles is $K{\cdot}(1{-}\alpha){\cdot}N = 15{\cdot}0.2{\cdot}50 = 150$.

\paragraph{Weighting scheme.}
We replace the marginal-likelihood weights with MSE-based weights to test whether the principled Bayesian weighting is necessary for good performance of single models. Note that modifying the weighting removes the interpretatbility of the particles as valid posterior samples.

\paragraph{Simulation budget.} 
We reduce the number of simulations used to train the likelihood estimator from $5{,}000$ to $\{1{,}000, 500, 200\}$ to characterise the compute--quality trade-off.

\paragraph{Prompt complexity.}
We replace the full task description with a minimal prompt to assess how much prompt engineering contributes to discovery performance.

\paragraph{Feedback format.}
We vary how performance information is communicated to the LLM. I.e., LLM-based feedback text only (default), combined LLM feedback + numeric metrics, and numeric metrics only. This tests whether structured metric feedback helps or hurts proposal quality.

\paragraph{Pool size and number of iterations.}
We vary $N$ and $K$ while keeping the total budget of $150$ new particles fixed: $(N{=}150, K{=}5)$ favours population diversity, while $(N{=}5, K{=}150)$ favours deep iterative refinement.

\paragraph{LLM backbone.}
We replace Claude Sonnet 4.6 with GPT-5 mini to verify that ModelSMC is not tied to a specific language model.

\appsection{Additional Experiments}\label{app:AdditionalExperiments}

\appsubsection{Validation of ModelSMC for Model Discovery}\label{app:llm_free_example_additional_results}

In \cref{fig:MinimalExampleMainText}, the median proportion of the target model in the population, i.e., the model from which the data $\vx_o$ was drawn, converges to an upper bound of $\approx80\%$ and exhibits a distinct oscillatory behaviour. This oscillatory behaviour becomes even more prominent when evaluating the proportion of the target model in the population for a single run (\cref{fig:MinimalExampleAdditionalResults}~b).

If the ESS (\cref{eq:AlgoModelSMCESS}) is below $\tau_{\mathrm{ESS}}$, systematic resampling is applied to rejuvenate the particle population (\cref{sec:Method_ModelSMC}). For candidate models that differ from the target model, the resampling weights are vanishingly small (\cref{fig:MinimalExampleAdditionalResults}~a). Therefore, the population is expected to contain only the target model after resampling with these weights.  In the subsequent propagation step, each particle in the population is replaced by a randomly selected candidate model with a probability of $1-\alpha=20\%$. Given that there are 20 candidate models, in expectation $81\%$ ($(1-0.2)+0.2\cdot\frac{1}{20} = 0.81$) of the models in the population are the target model in this setting with 20 candidate models and choice of $\alpha$. This matches the observed upper bound of $\approx80\%$ (\cref{fig:MinimalExampleMainText}). 

The oscillatory behaviour can also be explained by the replacement of particles during the propagation step. Starting with $100\%$ target models after the systematic resampling step, the propagation step will in expectation yield $81\%$ target models as outlined in the previous paragraph. If no systematic resampling occurs in the next iteration, i.e., the effective sample size is still sufficiently large, the random replacement of particles during the propagation step will further reduce the proportion of the target model in the population. This reduction continues until the effective sample size is below the specified threshold, which will bring the proportion of the target model back to roughly $80\%$. The repeated interplay between the resampling and propagation steps produces the observed oscillatory behaviour in the proportion of the target model. This relation between effective sample size and proportion of the target model can clearly be observed (\cref{fig:MinimalExampleAdditionalResults}~b,c).

\begin{figure}[ht]
    \centering
    \includegraphics[width=\linewidth]{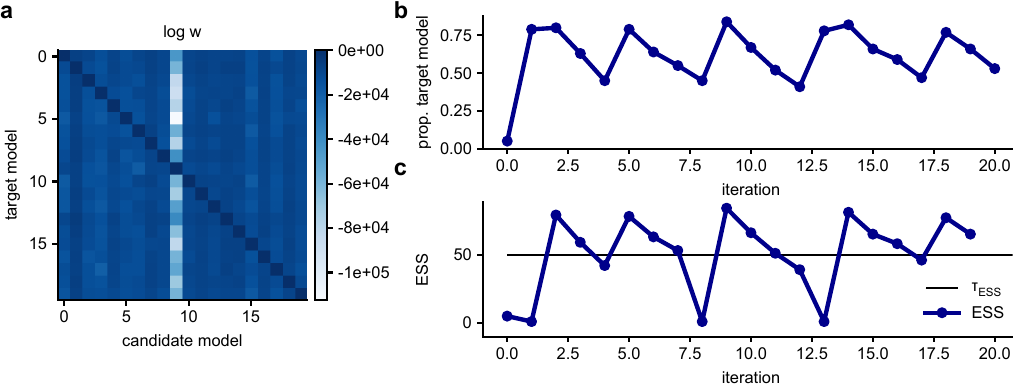}
    \caption{\textbf{(a)} Logarithm of normalized resampling weights, averaged over five random seeds, for different combinations of target and candidate models for the LLM-free experiment (\cref{sec:experiments_ValidationNoLLM},  App.~\ref{app:ExperimentalDetails_MinimalExample}). Data $\vx_o$ is drawn from the target model, and weights are evaluated using the ground truth density (\cref{eq:appendix_GMMDefinition}) of the candidate model. \textbf{(b)} Proportion of the target model in the population for a selected run from the experiment presented in \cref{fig:MinimalExampleMainText}. \textbf{(c)} Effective sample size (\cref{eq:AlgoModelSMCESS}) for the run presented in (b). The black horizontal line denotes the threshold $\tau_{\mathrm{ESS}}$ for applying systematic resampling (\cref{sec:Method_ModelSMC}).}
    \label{fig:MinimalExampleAdditionalResults}
\end{figure}

\subsection{Extended Version of \Cref{tab:quantitative_results}}
\begin{table}[H]
\centering
\caption{
\textbf{Extended version of \cref{tab:quantitative_results}}. Metrics are collected for the best-performing particle from each of the ten discovery runs. Best-performing particles are selected according to the metric which was applied to weight particles during the discovery process. We report the median metric together with a $90\%$ confidence interval. The confidence interval is estimated using a percentile bootstrap: $10^4$~bootstrap samples of size ten are drawn with replacement from the metric values of the selected particles. The median is computed for each sample, and the $5\%$ and $95\%$ quantiles of the resulting distribution are reported.
}

\label{tab:appendix_quantitative_results}
\setlength{\tabcolsep}{0pt}

\begin{tabularx}{\textwidth}{>{\centering\arraybackslash}p{1.0cm} *{3}{Y}ccc}
\toprule
 & \multicolumn{1}{c}{FunSearch+} & \multicolumn{1}{c}{ModelSMC} & \multicolumn{1}{c}{ModelSMC $N=1$} \\
\midrule
 & \multicolumn{3}{c}{$-\log p(\vx_o|m)~(\downarrow)$} \\
\midrule
SIR & -48794.12 {\scriptsize [-51910.18, -48293.60]} & -45834.78 {\scriptsize [-50184.76, -38023.12]} & -57028.36 {\scriptsize [-59613.61, -47127.31]} \\
HH & 254.58 {\scriptsize [250.46, 259.52]} & 256.00 {\scriptsize [255.05, 260.50]} & 243.27 {\scriptsize [232.39, 258.08]} \\
Kidney & N/A & 42.85 {\scriptsize [41.54, 45.62]} & 36.54 {\scriptsize [34.95, 38.31]} \\
\bottomrule
\end{tabularx}

\end{table}

\appsubsection{LLM Token Usage}
\label{app:tokens_vs_performance}

An evaluation criterion to assess the effectiveness of an LLM-based model discovery algorithm is how many tokens are passed to and generated by an LLM to generate a model with a given performance threshold on the observed data. To minimize computational costs, the number of tokens needed to reach a given performance threshold should be as small as possible. In this section, we examine how ModelSMC performs under this evaluation criterion compared to the baseline methods (Apps.~\ref{app:baselines_FunSearchPlus} and ~\ref{app:baselines_ModelSMConeparticle}).

The results in \cref{fig:appendix_token_usage_vs_performance} are based on the discovery runs evaluated in \cref{sec:quantitative_results}. For each discovery run we compute the running best metric value, i.e., the best value observed up to and including each iteration of the discovery run, and then take the median over all ten discovery runs per combination of task and metric. The $90\%$ confidence intervals are computed by randomly selecting ten of the trajectories with replacement and computing the median of the running best values over the selection at each iteration. This is done 10,000 times for each combination of task and method. The $90\%$ confidence interval is then defined by the $5\%$ and $95\%$ quantiles of the distribution of the bootstrap medians at each iteration. The position on the token axis corresponds to the median number of tokens used up to and including the current iteration, computed over all ten discovery runs.

\cref{fig:appendix_token_usage_vs_performance} shows that while ModelSMC and ModelSMC $N$=1 generate the same expected number of particles, ModelSMC $N$=1 consistently consumes more tokens across all three tasks. This is because ModelSMC $N$=1 always resamples the most recently generated particle, so the ancestry chain passed to the LLM grows with each iteration until reaching its maximum length after 3 iterations. In ModelSMC, particles from early iterations can persist into later ones, keeping ancestry chains shorter and reducing expected input tokens.
Compared to both ModelSMC variants, FunSearch+ requires significantly fewer tokens per particle. This is most clear on the Hodgkin-Huxley task, where the expected 150 ModelSMC particles consume approximately twice as many tokens as the 180 FunSearch+ particles. This is expected: ModelSMC and ModelSMC $N$=1 pass the task context $c_{\text{task}}$ and text-based performance evaluations (App.~\ref{app:LLMDetails}) alongside the resampled model ancestry, whereas FunSearch+ only provides the code skeleton and sampled functions from the program database. ModelSMC and ModelSMC $N$=1 also generate more output tokens, as both produce text-based performance feedback that FunSearch+ does not.

The median curve for ModelSMC $N$=1 falls below both other methods for most of the token budget, suggesting higher token efficiency. However, the confidence intervals overlap across all three tasks, so no statistically significant difference can be established from ten discovery runs per method. Formulating model discovery as probabilistic inference therefore does not significantly affect token efficiency relative to the baselines.

\begin{figure}[H]
    \centering
    \includegraphics[width=\linewidth]{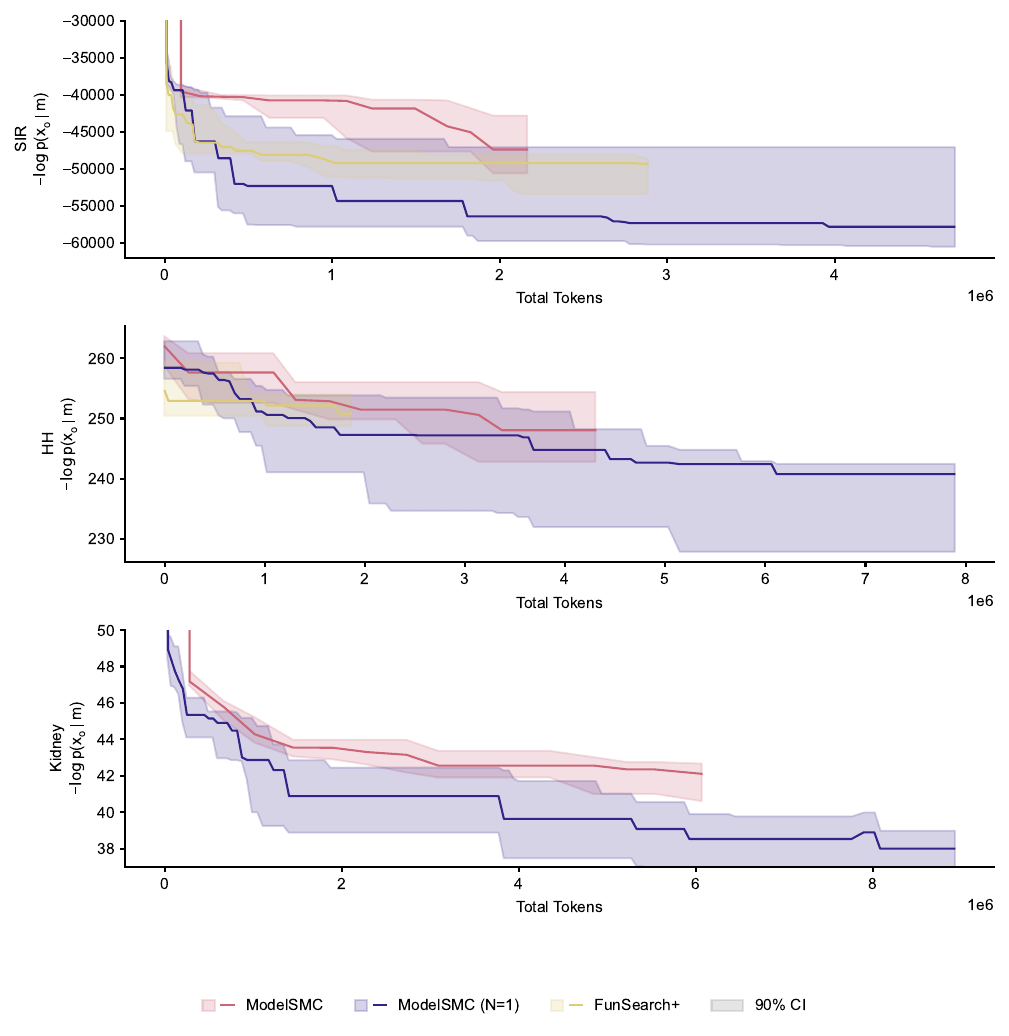}
    \caption{\textbf{Performance of the discovered models as a function of token usage}. Red: ModelSMC. Blue: ModelSMC $N$=1. Yellow: FunSearch+. The depicted performance metrics and token counts are the median of the ten discovery runs conducted for the experiments discussed in \cref{sec:quantitative_results}. The depicted metric is the negative log of ModelSMC's resampling weight (\cref{eq:ResamplingWeightModelSMCFactorization}). The $90\%$ confidence intervals of the medians are defined by the $5\%$ and $95\%$ quantiles of the empirical distribution of 10,000 bootstrap replicates of the median. First row: Synthetic data from a SIR epidemiological model (App.~\ref{app:ExperimentalDetails_SIR}). Second row: Hodgkin-Huxley Model (\cref{sec:exp-allen} and App.~\ref{app:ExperimentalDetails_allen}). Third row: Pharmacological kidney model (\cref{sec:exp-kidney} and App.~\ref{app:ExperimentalDetails_kidney}).}
    \label{fig:appendix_token_usage_vs_performance}
\end{figure}

\newpage
\appsection{Prompt Examples}
\label{app:PromptExamples}
The implementations contained in the messages presented in this section are based on and adapted from the implementation accompanying \cite{gonccalves2020training} (available at \url{https://github.com/mackelab/IdentifyMechanisticModels_2020}), Copyright 2020 Pedro J. Gonçalves, Jan-Matthis Lueckmann, Michael Deistler, Marcel Nonnenmacher, Kaan Öcal, Giacomo Bassetto, Chaitanya Chintaluri, William F. Podlaski, Sara A. Haddad, Tim P. Vogels, David S. Greenberg, Jakob H. Macke, used under the MIT license.

\appsubsection{Hodgkin-Huxley Experiment: LLM-History Propagation Step}
\label{app:LLMInteractionsHistoryPropagation}

The prompts passed to the LLM are assembled from different components (see App.~\ref{app:LLMDetails}) including specified task information and instructions, implementations of ancestral particles, and their LLM-generated performance assessments (App.~\ref{app:LLMDetails}). In this section, we provide a concrete example of the general interaction scheme visualized in \cref{fig:appendix_figure_prompting_strategy}. Specifically, we exemplify the interactions observed for a particle during the propagation step in a discovery run for the Hodgkin-Huxley experiment (\cref{sec:exp-allen}, App.~\ref{app:ExperimentalDetails_allen}). In the following visualization, each text box corresponds to a single message sent to or received from the LLM. The raw messages were reformatted for readability, with malformatted and \LaTeX{}-incompatible characters replaced throughout. The messages sent to the LLM are presented in section 'Prompt', followed by the LLM's response in section 'Response'.

\input{llm_history_allen/llm_history_propagation}

\newpage
\appsubsection{Hodgkin-Huxley Experiment: LLM-History Feedback Step }

This section presents a specific example of the LLM interactions in the feedback step. The messages sent to and received from the LLM are presented in the same format as the propagation step in App.~\ref{app:LLMInteractionsHistoryPropagation}.

\label{app:LLMInteractionsHistoryFeedback}

\input{llm_history_allen/llm_history_feedback}

\end{document}

%% file: llm_history_allen/llm_history_preamble.tex
\usepackage[T1]{fontenc}
\usepackage{upquote}

\usepackage{tcolorbox}
\tcbuselibrary{skins, breakable, listings}

\usepackage{listings}
\usepackage{xcolor}

\definecolor{SystemBg}      {HTML}{CDDAF5}
\definecolor{SystemFrame}   {HTML}{7799e4}

\definecolor{UserBg}        {HTML}{FFF3CD}
\definecolor{UserFrame}     {HTML}{ffdb66}

\definecolor{AssistantBg}   {HTML}{D4EDDA}
\definecolor{AssistantFrame}{HTML}{8ccf9c}

\definecolor{OutputBg}      {HTML}{FAC898}
\definecolor{OutputFrame}   {HTML}{f69337}

\definecolor{codebg}        {RGB}{248, 248, 248}
\definecolor{codeframe}     {RGB}{180, 180, 180}

\definecolor{kwcolor}       {RGB}{  0,   0, 195}
\definecolor{strcolor}      {RGB}{175,   0,   0}
\definecolor{cmcolor}       {RGB}{ 90, 120,  90}

\lstdefinestyle{pythonstyle}{
  language=Python,
  basicstyle=\scriptsize\ttfamily,
  keywordstyle=\color{kwcolor}\bfseries,
  stringstyle=\color{strcolor},
  commentstyle=\color{cmcolor}\itshape,
  numberstyle=\tiny\color{gray},
  numbers=left,
  numbersep=7pt,
  backgroundcolor=\color{codebg},
  frame=single,
  framesep=4pt,
  rulecolor=\color{codeframe},
  breaklines=true,
  breakatwhitespace=false,
  breakautoindent=true,
  breakindent=0.5em,
  postbreak=\mbox{\textcolor{gray}{$\hookrightarrow$}\space},
  keepspaces=true,
  showspaces=false,
  showstringspaces=false,
  tabsize=2,
  xleftmargin=16pt,
  numbersep=8pt,
  xrightmargin=4pt,
  morekeywords={self,True,False,None,as,with,yield,lambda,assert,
                del,except,finally,global,nonlocal,pass,raise,
                return,try,while,async,await},
}

%% file: math_commands.tex
\usepackage{amsmath,amsfonts,bm}

\def\eqref#1{equation~\ref{#1}}

\def\1{\bm{1}}

\def\vx{{\bm{x}}}

\DeclareMathAlphabet{\mathsfit}{\encodingdefault}{\sfdefault}{m}{sl}
\SetMathAlphabet{\mathsfit}{bold}{\encodingdefault}{\sfdefault}{bx}{n}

\def\gU{{\mathcal{U}}}

%% file: llm_history_allen/llm_history_propagation.tex
%

\subsubsection{Prompt}

\begin{tcolorbox}[
  breakable, enhanced,
  title={System},
  colback=SystemBg, colframe=SystemFrame,
  coltitle=white, fonttitle=\bfseries\small,
  arc=2mm, boxrule=1pt,
  before upper={\small\setlength{\parindent}{0pt}\setlength{\parskip}{2pt}},
]
Your input fields are:\\{}
1. `signature\_description` (str): Simulator I/O description\\{}
2. `task\_description` (str): Task considerations\\{}
3. `base\_simulator` (str): Base code structure\\{}
Your output fields are:\\{}
1. `scm\_definition` (str): Compact SCM summary as a bullet list (max 5 lines): key variables, their causal dependencies, and core assumptions. No equations or verbose descriptions.\\{}
2. `simulator\_code` (str): Improved mechanistic simulator with substantial structural differences from prior examples.\\{}
3. `feedback` (Union[str, NoneType]): Optional feedback. Only generate if requested.\\{}
All interactions will be structured in the following way, with the appropriate values filled in.
\smallskip
\par\noindent [[ \#\# signature\_description \#\# ]]\\
\noindent
\{signature\_description\}
\smallskip
\par\noindent [[ \#\# task\_description \#\# ]]\\
\noindent
\{task\_description\}
\smallskip
\par\noindent [[ \#\# base\_simulator \#\# ]]\\
\noindent
\{base\_simulator\}
\smallskip
\par\noindent [[ \#\# scm\_definition \#\# ]]\\
\noindent
\{scm\_definition\}
\smallskip
\par\noindent [[ \#\# simulator\_code \#\# ]]\\
\noindent
\{simulator\_code\}
\smallskip
\par\noindent [[ \#\# feedback \#\# ]]\\
\noindent
\{feedback\}        \# note: the value you produce must adhere to the JSON schema: \{"anyOf": [\{"type": "string"\}, \{"type": "null"\}]\}
\smallskip
\par\noindent [[ \#\# completed \#\# ]]\\
\noindent
In adhering to this structure, your objective is: \\{}
        Given the fields `signature\_description`, `task\_description`, `base\_simulator`, produce the fields `scm\_definition`, `simulator\_code`, `feedback`.\\{}
        \\{}
        You are an expert in mechanistic modeling with differential equations and simulator optimization.\\{}
        \\{}
        Core Requirements:\\{}
        - Create a scientific simulator that captures system dynamics accurately\\{}
        - Write fully functional, executable code with no placeholders\\{}
        - Use tabs for indentation and mathematically correct formulations\\{}
        - Add variable shapes in every line of code\\{}
        - Achieve optimal validation metrics\\{}
        \\{}
        Implementation Guidelines:\\{}
        - Start with mechanistic components, add stochastic elements where needed\\{}
        - Follow physical/biological constraints and include prior knowledge\\{}
        - Use any operations (log, exp, power, etc.) and non-differentiable operations\\{}
        - Ensure computational efficiency for many simulations\\{}
        - Keep the code as simple as possible\\{}
        \\{}
        Iterative Optimization Rules:\\{}
        - You are part of an iterative process - make only 1-2 targeted improvements per iteration\\{}
        - Never copy code verbatim - preserve working components while making precise, localized changes\\{}
        - Make creative but sensible and simple modifications that are structurally different from prior examples\\{}
        - Focus on changes that build incrementally toward better validation metrics\\{}
        - Ensure executable code without syntax or shape errors\\{}
        \\{}
        System description:\\{}
        \# HODGKIN-HUXLEY NEURON SIMULATOR WITH MISSING ION CHANNELS\\{}
        \\{}
        \#\# OBJECTIVE\\{}
        Extend an existing Hodgkin-Huxley neuron simulator by identifying and implementing\\{}
        missing ion channels necessary to reduce discrepancies between simulations and\\{}
        experimental voltage recordings. The goal is to improve agreement across multiple\\{}
        electrophysiological summary statistics while maintaining model parsimony.\\{}
        \\{}
        \#\# BASE MODEL\\{}
        Classic Hodgkin-Huxley formulation with standard three-gating-variable structure:\\{}
        - **Leak current**: conductance g\_leak, reversal potential E\_leak\\{}
        - **Sodium (Na$^{+}$) current**: conductance gNa, gating variables m (activation) and h (inactivation), reversal E\_Na\\{}
        - **Potassium (K$^{+}$) current**: conductance gK, gating variable n (activation), reversal E\_K\\{}
        - **Dynamics**: First-order gating kinetics with voltage-dependent time constants and steady states\\{}
        - **Input**: Externally applied current I\_inj(t) with optional stochastic components\\{}
        - **Stochasticity**: Optional random seed for reproducible simulations\\{}
        \\{}
        \#\# EXTENSIBILITY\\{}
        Two pre-allocated channel slots (X1, X2) are available for additional mechanisms.\\{}
        Each slot provides two tunable parameters (see signature\_description for details).\\{}
        Channels beyond these two slots must use fixed parameters, not inferred ones.\\{}
        \\{}
        \#\# EVALUATION METRICS\\{}
        Model performance assessed on held-out voltage recordings via these summary statistics:\\{}
        1. Number of spikes during stimulation\\{}
        2. Mean resting potential (pre-stimulation)\\{}
        3. Standard deviation of resting potential\\{}
        4. Mean voltage during stimulation\\{}
        5. Variance of voltage\\{}
        6. Skewness of voltage\\{}
        7. Kurtosis of voltage\\{}
        \\{}
        \#\# DATA CONTEXT\\{}
        Intracellular voltage recordings from a single neuron exhibiting tonic spiking\\{}
        behavior across different stimulus intensities.\\{}
        \\{}
        **Important data characteristics**:\\{}
        - No burst firing patterns (no clusters of rapid spikes separated by quiescence)\\{}
        - No prolonged or sustained high-frequency firing (no sustained high-frequency firing)\\{}
        - Regular, tonic spiking activity during periods of input current stimulation\\{}
        - Non-spiking (quiescent) behavior in the absence of input current\\{}
        - Evenly-spaced action potentials without burst patterns during the activation period\\{}
        \\{}
        Therefore, avoid adding channels specifically designed to produce bursting or sustained\\{}
        high-frequency dynamics.
\smallskip
\end{tcolorbox}
\newpage

\begin{tcolorbox}[
  breakable, enhanced,
  title={User},
  colback=UserBg, colframe=UserFrame,
  coltitle=white, fonttitle=\bfseries\small,
  arc=2mm, boxrule=1pt,
  before upper={\small\setlength{\parindent}{0pt}\setlength{\parskip}{2pt}},
]
This is an example of the task, though some input or output fields are not supplied.
\smallskip
\par\noindent [[ \#\# signature\_description \#\# ]]\\
\noindent
(see current input)
\smallskip
\par\noindent [[ \#\# task\_description \#\# ]]\\
\noindent
(see current input)
\smallskip
\par\noindent [[ \#\# base\_simulator \#\# ]]\\
\noindent
(see current input)
\smallskip
\end{tcolorbox}
\vspace{0.35em}

\begin{tcolorbox}[
  breakable, enhanced,
  title={Assistant},
  colback=AssistantBg, colframe=AssistantFrame,
  coltitle=white, fonttitle=\bfseries\small,
  arc=2mm, boxrule=1pt,
  before upper={\small\setlength{\parindent}{0pt}\setlength{\parskip}{2pt}},
]
\par\noindent [[ \#\# scm\_definition \#\# ]]\\
\noindent
None
\smallskip
\par\noindent [[ \#\# simulator\_code \#\# ]]\\
\noindent
\begin{lstlisting}[language=Python,style=pythonstyle]
import torch
import torch.nn as nn

class DiscoveredSimulator(nn.Module):
  def __init__(self):
    super(DiscoveredSimulator, self).__init__()
    return

  def forward(
    self,
    init_voltage: float,
    input_current: torch.Tensor,
    dt: float,
    t: torch.Tensor,
    params: torch.Tensor,
    seed=None,
  ):
    """
    Simulates a Hodgkin-Huxley neuron for a specified time duration.

    Args:
      init_voltage: torch.Tensor: (batch_size,) # initial voltage
      input_current: torch.Tensor: (batch_size, time_steps) # input current
      dt: float # time step size
      t: torch.Tensor: (time_steps,) # time array
      params: torch.Tensor: (batch_size, n_params) # parameters
      seed: optional random seed

    Returns:
      V: torch.Tensor: (batch_size, time_steps) # voltage traces
    """
    device = params.device

    # Set up random generator
    if seed is not None:
      generator = torch.Generator(device=device)
      generator.manual_seed(seed)
    else:
      generator = torch.Generator(device=device)

    batch_size = params.shape[0]
    time_steps = t.shape[0]

    # Extract parameters
    gbar_Na = params[:, 0].float()  # mS/cm2
    gbar_K = params[:, 1].float() # mS/cm2
    g_leak = params[:, 2].float() # mS/cm2
    E_leak = -params[:, 3].float() # mV
    Vt = -params[:, 4].float() # mV
    nois_fact = params[:, 5].float() # unitless
    # TWO POSSIBLE ADDITIONAL CHANNELS (X1, X2)
    # Each channel has one tunable parameter: conductance gbar_Xi
    # Then there are two additional parameters available: param_i and param_j.
    # ONLY ADD ONE CHANNEL IF NECESSARY. Keep the model as simple as possible.
    gbar_X1 = params[:, 6].float() # mS/cm2 # you can rename X1 to anything you want # in range [1e-4, 10]
    gbar_X2 = params[:, 7].float() # mS/cm2 # you can rename X2 to anything you want # in range [1e-4, 120]
    param_i = -params[:, 8].float() # (param are positive values in range [1e-4, 150])
    param_j = -params[:, 9].float() # (param are positive values in range [1e-4, 3000])

    tstep = float(dt)

    # Parameters
    nois_fact_obs = 0.0
    C = 1.0  # uF/cm^2
    E_Na = 53.0 # mV
    E_K = -107.0

    ####################################
    # kinetics
    def Exp(z):
      return torch.where(z < -5e2, torch.exp(torch.full_like(z, -5e2)), torch.exp(z))

    def efun(z):
      return torch.where(torch.abs(z) < 1e-4, 1 - z / 2, z / (Exp(z) - 1))

    # Channel kinetics
    def alpha_m(x):
      v1 = x - Vt - 13.0
      return 0.32 * efun(-0.25 * v1) / 0.25

    def beta_m(x):
      v1 = x - Vt - 40
      return 0.28 * efun(0.2 * v1) / 0.2

    def alpha_h(x):
      v1 = x - Vt - 17.0
      return 0.128 * Exp(-v1 / 18.0)

    def beta_h(x):
      v1 = x - Vt - 40.0
      return 4.0 / (1 + Exp(-0.2 * v1))

    def alpha_n(x):
      v1 = x - Vt - 15.0
      return 0.032 * efun(-0.2 * v1) / 0.2

    def beta_n(x):
      v1 = x - Vt - 10.0
      return 0.5 * Exp(-v1 / 40)

    def tau_x(alpha, beta):
      return 1.0 / (alpha + beta)

    def inf_x(alpha, beta):
      return alpha / (alpha + beta)

    # ===== BEGIN EDITABLE SECTION (only modify within this block) =====
    # TODO: add the missing kinetics equations for the Hodgkin-Huxley neuron similar to the ones above; ONLY ADD IF NECESSARY
    # ===== END EDITABLE SECTION =====

    ####################################

    # simulation from initial point
    V = torch.zeros((batch_size, time_steps), device=device)  # baseline voltage
    n = torch.zeros((batch_size, time_steps), device=device)
    m = torch.zeros((batch_size, time_steps), device=device)
    h = torch.zeros((batch_size, time_steps), device=device)
    # ===== BEGIN EDITABLE SECTION (only modify within this block) =====
    # TODO: add the missing state variables for the Hodgkin-Huxley neuron similar to the ones above; ONLY ADD IF NECESSARY
    # ===== END EDITABLE SECTION =====

    # Initialization
    V_init = init_voltage.to(device)
    V[:, 0] = V_init
    n[:, 0] = inf_x(alpha_n(V[:, 0]), beta_n(V[:, 0]))
    m[:, 0] = inf_x(alpha_m(V[:, 0]), beta_m(V[:, 0]))
    h[:, 0] = inf_x(alpha_h(V[:, 0]), beta_h(V[:, 0]))
    # ===== BEGIN EDITABLE SECTION (only modify within this block) =====
    # TODO: add the missing state variable initialization for the Hodgkin-Huxley neuron similar to the ones above; ONLY ADD IF NECESSARY
    # ===== END EDITABLE SECTION =====

    # Simulation loop
    for i in range(1, time_steps):
      # All operations now work on batched tensors (batch_size,)
      a_m, b_m = alpha_m(V[:, i - 1]), beta_m(V[:, i - 1])
      a_h, b_h = alpha_h(V[:, i - 1]), beta_h(V[:, i - 1])
      a_n, b_n = alpha_n(V[:, i - 1]), beta_n(V[:, i - 1])
      # ===== BEGIN EDITABLE SECTION (only modify within this block) =====
      # TODO: add the missing kinetics equations for the Hodgkin-Huxley neuron similar to the ones above; ONLY ADD IF NECESSARY
      # ===== END EDITABLE SECTION =====

      tau_V_inv = (
        (m[:, i - 1] ** 3) * gbar_Na * h[:, i - 1]
        + (n[:, i - 1] ** 4) * gbar_K
        + g_leak
        # ===== BEGIN EDITABLE SECTION (only modify within this block) =====
        # TODO: add the missing terms for the effective membrane time constant inverse; ONLY ADD IF NECESSARY
        # ===== END EDITABLE SECTION =====
      ) / C

      V_inf = (
        (m[:, i - 1] ** 3) * gbar_Na * h[:, i - 1] * E_Na
        + (n[:, i - 1] ** 4) * gbar_K * E_K
        + g_leak * E_leak
        # ===== BEGIN EDITABLE SECTION (only modify within this block) =====
        # TODO: add the missing terms for the voltage steady state; ONLY ADD IF NECESSARY
        # ===== END EDITABLE SECTION =====
        + input_current[:,i - 1]
        + nois_fact * torch.randn(batch_size, generator=generator, device=device) / (tstep**0.5)
      ) / (tau_V_inv * C)

      V[:, i] = V_inf + (V[:, i - 1] - V_inf) * Exp(-tstep * tau_V_inv)
      n[:, i] = inf_x(a_n, b_n) + (n[:, i - 1] - inf_x(a_n, b_n)) * Exp(-tstep / tau_x(a_n, b_n))
      m[:, i] = inf_x(a_m, b_m) + (m[:, i - 1] - inf_x(a_m, b_m)) * Exp(-tstep / tau_x(a_m, b_m))
      h[:, i] = inf_x(a_h, b_h) + (h[:, i - 1] - inf_x(a_h, b_h)) * Exp(-tstep / tau_x(a_h, b_h))
      # ===== BEGIN EDITABLE SECTION (only modify within this block) =====
      # TODO: add the missing state variable updates for the Hodgkin-Huxley neuron similar to the ones above; ONLY ADD IF NECESSARY
      # ===== END EDITABLE SECTION =====

    # Return voltage with optional observation noise
    return V + nois_fact_obs * torch.randn(
      batch_size, time_steps, generator=generator, device=device
    )
\end{lstlisting}
\smallskip
\par\noindent [[ \#\# feedback \#\# ]]\\
\noindent
\par\noindent \{
\par\noindent\hangindent=1.5em\hangafter=1 \hspace*{1.5em}``\textbf{\texttt{main\_diagnosis}}\textquotesingle\textquotesingle: ``The two additional channel slots (X1, X2) have conductance parameters and auxiliary parameters extracted (`gbar\_X1`, `gbar\_X2`, `param\_i`, `param\_j`) but zero kinetic equations, state variables, or current contributions are implemented anywhere in the simulation loop. The model is running as a plain 3-channel HH model regardless of what values the inference assigns to those four parameters, which fully explains the high NLE (\textasciitilde{}26 per observation).'',
\par\noindent \hspace*{1.5em}``\textbf{\texttt{issues}}\textquotesingle\textquotesingle:
\par\noindent \hspace*{1.5em}[
\par\noindent \hspace*{3.0em}\{
\par\noindent\hangindent=4.5em\hangafter=1 \hspace*{3.0em}\hspace*{1.5em}``\textbf{\texttt{description}}\textquotesingle\textquotesingle: ``gbar\_X1, gbar\_X2, param\_i, and param\_j are extracted from params but are never referenced in kinetics, state-variable arrays, tau\_V\_inv, V\_inf, or the integration loop. Their values have no effect on the output voltage trace, so the likelihood surface over those dimensions is flat and the model cannot reduce its NLE below what a plain HH achieves.'',
\par\noindent\hangindent=4.5em\hangafter=1 \hspace*{3.0em}\hspace*{1.5em}``\textbf{\texttt{severity}}\textquotesingle\textquotesingle: ``critical'',
\par\noindent\hangindent=4.5em\hangafter=1 \hspace*{3.0em}\hspace*{1.5em}``\textbf{\texttt{location}}\textquotesingle\textquotesingle: ``All four `\# ===== BEGIN EDITABLE SECTION =====` blocks inside forward()'',
\par\noindent\hangindent=4.5em\hangafter=1 \hspace*{3.0em}\hspace*{1.5em}``\textbf{\texttt{suggestion}}\textquotesingle\textquotesingle: ``Add an M-type slow potassium current (I\_KM) using gbar\_X1 as conductance and E\_K as reversal. This single-gating-variable current (activation gate `p`, no inactivation) is the canonical addition for converting standard HH into a model capable of regular, tonic, non-bursting firing. Implement: (1) kinetics block: `def inf\_p(x): return 1.0/(1+Exp(-(x-param\_i)/10.0))` and `def tau\_p(x): return (-param\_j)/(3.3*(Exp((x-param\_i)/20)+Exp(-(x-param\_i)/20))+1e-7)` (note param\_i is already negated in the extraction so it will be \textasciitilde{}-35 mV at typical posterior values; -param\_j gives a positive timescale up to 3000 ms). (2) State-variable array: `p = torch.zeros(...)`, initialized with `p[:,0] = inf\_p(V[:,0])`. (3) In the loop: compute `inf\_p\_val = inf\_p(V[:,i-1])` and `tau\_p\_val = tau\_p(V[:,i-1])`; add `gbar\_X1` contribution to `tau\_V\_inv` (`+ gbar\_X1*p[:,i-1]/C`) and to `V\_inf` (`+ gbar\_X1*p[:,i-1]*E\_K`); update `p[:,i] = inf\_p\_val + (p[:,i-1]-inf\_p\_val)*Exp(-tstep/tau\_p\_val)`. Keep gbar\_X2 zero or unused unless metrics remain poor after this fix.''
\par\noindent \hspace*{3.0em}\},
\par\noindent \hspace*{3.0em}\{
\par\noindent\hangindent=4.5em\hangafter=1 \hspace*{3.0em}\hspace*{1.5em}``\textbf{\texttt{description}}\textquotesingle\textquotesingle: ``The summary statistics include resting membrane potential mean and SD, plus voltage kurtosis and skewness. A plain HH model produces a resting potential determined only by g\_leak and E\_leak without subthreshold resonance, often mismatching experimental resting statistics. Adding even a small M-current contribution at rest shifts and stabilises the resting potential and changes subthreshold fluctuation shape.'',
\par\noindent\hangindent=4.5em\hangafter=1 \hspace*{3.0em}\hspace*{1.5em}``\textbf{\texttt{severity}}\textquotesingle\textquotesingle: ``major'',
\par\noindent\hangindent=4.5em\hangafter=1 \hspace*{3.0em}\hspace*{1.5em}``\textbf{\texttt{location}}\textquotesingle\textquotesingle: ``tau\_V\_inv and V\_inf accumulation expressions in the simulation loop'',
\par\noindent\hangindent=4.5em\hangafter=1 \hspace*{3.0em}\hspace*{1.5em}``\textbf{\texttt{suggestion}}\textquotesingle\textquotesingle: ``After implementing the M-current gate `p` as described above, ensure that the M-current term is included in both `tau\_V\_inv` (adds `gbar\_X1 * p[:, i-1] / C`) and `V\_inf` (adds `gbar\_X1 * p[:, i-1] * E\_K` in the numerator). Because `p` is near its steady-state at rest but small, the effect is a mild hyperpolarising bias that pulls the resting voltage toward E\_K, correcting the resting-potential mismatch without inducing bursting. If after this change the spike-count or resting-SD metrics still mismatch, consider activating gbar\_X2 as a small persistent sodium current (I\_NaP) with fixed half-activation, but only as a secondary step to preserve parsimony.''
\par\noindent \hspace*{3.0em}\}
\par\noindent \hspace*{1.5em}]
\par\noindent \}
\smallskip
\par\noindent [[ \#\# completed \#\# ]]\\
\noindent
\smallskip
\end{tcolorbox}
\vspace{0.35em}

\begin{tcolorbox}[
  breakable, enhanced,
  title={User},
  colback=UserBg, colframe=UserFrame,
  coltitle=white, fonttitle=\bfseries\small,
  arc=2mm, boxrule=1pt,
  before upper={\small\setlength{\parindent}{0pt}\setlength{\parskip}{2pt}},
]
\par\noindent [[ \#\# signature\_description \#\# ]]\\
\noindent
(see current input)
\smallskip
\par\noindent [[ \#\# task\_description \#\# ]]\\
\noindent
(see current input)
\smallskip
\par\noindent [[ \#\# base\_simulator \#\# ]]\\
\noindent
(see current input)
\smallskip
\end{tcolorbox}
\vspace{0.35em}

\begin{tcolorbox}[
  breakable, enhanced,
  title={Assistant},
  colback=AssistantBg, colframe=AssistantFrame,
  coltitle=white, fonttitle=\bfseries\small,
  arc=2mm, boxrule=1pt,
  before upper={\small\setlength{\parindent}{0pt}\setlength{\parskip}{2pt}},
]
\par\noindent [[ \#\# scm\_definition \#\# ]]\\
\noindent
- **V**: membrane voltage $\leftarrow$ I\_Na(m,h), I\_K(n), I\_leak, I\_KM(p), I\_inj, noise\\{}
- **m, h**: Na$^{+}$ activation/inactivation $\leftarrow$ V, Vt (fast kinetics)\\{}
- **n**: K$^{+}$ activation $\leftarrow$ V, Vt (delayed rectifier)\\{}
- **p**: M-current (slow K$^{+}$) activation $\leftarrow$ V, param\_i (half-activation), param\_j (time-scale)\\{}
- **Assumptions**: single-compartment, tonic spiking, M-current provides adaptation without bursting
\smallskip
\par\noindent [[ \#\# simulator\_code \#\# ]]\\
\noindent
\begin{lstlisting}[language=Python,style=pythonstyle]
import torch
import torch.nn as nn

class DiscoveredSimulator(nn.Module):
 def __init__(self):
  super(DiscoveredSimulator, self).__init__()
  return

 def forward(
  self,
  init_voltage: float,
  input_current: torch.Tensor,
  dt: float,
  t: torch.Tensor,
  params: torch.Tensor,
  seed=None,
 ):
  """
  Hodgkin-Huxley neuron with an added M-type slow potassium current (I_KM).

  Physiological rationale for I_KM addition:
  - M-current is a slowly activating, non-inactivating K+ current active at
   subthreshold voltages (~-35 to -60 mV range)
  - It provides spike-frequency adaptation: slightly reduces firing rate at
   sustained stimulation without producing burst or quiescence
  - It stabilises regular tonic firing and corrects resting potential statistics
   (mean, SD) by adding a mild hyperpolarising conductance at rest
  - It does NOT produce bursting or high-frequency clusters --- exactly what
   the target data requires

  Args:
   init_voltage: torch.Tensor: (batch_size,) # initial voltage [mV]
   input_current: torch.Tensor: (batch_size, time_steps) # injected current [uA/cm2]
   dt: float # time step [ms]
   t: torch.Tensor: (time_steps,) # time array [ms]
   params: torch.Tensor: (batch_size, 10) # biophysical parameters
   seed: optional random seed

  Returns:
   V: torch.Tensor: (batch_size, time_steps) # voltage traces [mV]
  """
  device = params.device

  # Set up random generator
  if seed is not None:
   generator = torch.Generator(device=device)
   generator.manual_seed(seed)
  else:
   generator = torch.Generator(device=device)

  batch_size = params.shape[0]   # int
  time_steps = t.shape[0]        # int

  # -- Base parameters ------------------------------------------------------
  gbar_Na   = params[:, 0].float()   # (batch_size,)  mS/cm^2
  gbar_K    = params[:, 1].float()   # (batch_size,)  mS/cm^2
  g_leak    = params[:, 2].float()   # (batch_size,)  mS/cm^2
  E_leak    = -params[:, 3].float()  # (batch_size,)  mV  (sign applied)
  Vt        = -params[:, 4].float()  # (batch_size,)  mV  (sign applied)
  nois_fact = params[:, 5].float()   # (batch_size,)  unitless

  # -- M-current channel (X1 slot) ------------------------------------------
  # gbar_KM  : maximal conductance  [mS/cm^2], range [1e-4, 10]
  # V_half_p : half-activation voltage [mV], = param_i = -params[:,8]
  #            params[:,8] in [1e-4, 150] -> V_half_p in [-150, -1e-4]
  #            Physiologically ~-35 mV for M-current
  # tau_max_p: time-constant scale [ms],     = -param_j = params[:,9]
  #            params[:,9] in [1e-4, 3000] -> tau_max_p in [1e-4, 3000]
  #            M-current kinetics are slow: typically 50--500 ms
  gbar_KM   = params[:, 6].float()   # (batch_size,)  mS/cm^2  --- M-current conductance
  # params[:, 7] (gbar_X2) intentionally unused --- parsimony principle
  V_half_p  = -params[:, 8].float()  # (batch_size,)  mV      --- half-activation of p gate
  tau_max_p =  params[:, 9].float()  # (batch_size,)  ms      --- max time constant of p gate

  tstep = float(dt)

  # -- Fixed biophysical constants -------------------------------------------
  nois_fact_obs = 0.0
  C    = 1.0    # uF/cm^2
  E_Na = 53.0   # mV
  E_K  = -107.0 # mV  (also reversal for M-current --- same ion)

  # -- Numerical helpers -----------------------------------------------------
  def Exp(z):
   # Numerically safe exponential --- clamp at -500 to avoid underflow
   return torch.where(
    z < -5e2,
    torch.exp(torch.full_like(z, -5e2)),
    torch.exp(z)
   )  # same shape as z

  def efun(z):
   # Handles 0/0 limit of z / (exp(z) - 1) via L'Hopital
   return torch.where(
    torch.abs(z) < 1e-4,
    1 - z / 2,
    z / (Exp(z) - 1)
   )  # same shape as z

  # -- Standard HH channel kinetics -----------------------------------------
  def alpha_m(x):  # x: (batch_size,) -> (batch_size,)
   v1 = x - Vt - 13.0
   return 0.32 * efun(-0.25 * v1) / 0.25

  def beta_m(x):   # (batch_size,) -> (batch_size,)
   v1 = x - Vt - 40.0
   return 0.28 * efun(0.2 * v1) / 0.2

  def alpha_h(x):  # (batch_size,) -> (batch_size,)
   v1 = x - Vt - 17.0
   return 0.128 * Exp(-v1 / 18.0)

  def beta_h(x):   # (batch_size,) -> (batch_size,)
   v1 = x - Vt - 40.0
   return 4.0 / (1 + Exp(-0.2 * v1))

  def alpha_n(x):  # (batch_size,) -> (batch_size,)
   v1 = x - Vt - 15.0
   return 0.032 * efun(-0.2 * v1) / 0.2

  def beta_n(x):   # (batch_size,) -> (batch_size,)
   v1 = x - Vt - 10.0
   return 0.5 * Exp(-v1 / 40)

  def tau_x(alpha, beta):  # (batch_size,), (batch_size,) -> (batch_size,)
   return 1.0 / (alpha + beta)

  def inf_x(alpha, beta):  # (batch_size,), (batch_size,) -> (batch_size,)
   return alpha / (alpha + beta)

  # -- M-current (I_KM) gate kinetics ---------------------------------------
  # Boltzmann steady-state: p_inf = 1 / (1 + exp(-(V - V_half_p) / 10))
  # Activation slope k=10 mV is standard for M-current literature
  def inf_p(x):  # x: (batch_size,) -> (batch_size,)
   return 1.0 / (1.0 + Exp(-(x - V_half_p) / 10.0))

  # Voltage-dependent time constant with bell shape centred at V_half_p
  # tau_p = tau_max_p / (3.3 * 2*cosh((V - V_half_p)/20))
  # At V_half_p: tau_p = tau_max_p / (3.3*2) ~= tau_max_p / 6.6
  # Away from V_half_p: tau_p decreases (faster kinetics at extreme voltages)
  def tau_p(x):  # x: (batch_size,) -> (batch_size,)
   dv = (x - V_half_p) / 20.0                      # (batch_size,)
   denom = 3.3 * (Exp(dv) + Exp(-dv)) + 1e-7       # (batch_size,) clamped >0
   return tau_max_p / denom                         # (batch_size,)

  # -- State variable arrays -------------------------------------------------
  V = torch.zeros((batch_size, time_steps), device=device)  # (batch_size, time_steps)
  n = torch.zeros((batch_size, time_steps), device=device)  # (batch_size, time_steps)
  m = torch.zeros((batch_size, time_steps), device=device)  # (batch_size, time_steps)
  h = torch.zeros((batch_size, time_steps), device=device)  # (batch_size, time_steps)
  p = torch.zeros((batch_size, time_steps), device=device)  # (batch_size, time_steps) M-current gate

  # -- Initial conditions ----------------------------------------------------
  V_init = init_voltage.to(device)  # (batch_size,)
  V[:, 0] = V_init                                              # (batch_size,)
  n[:, 0] = inf_x(alpha_n(V[:, 0]), beta_n(V[:, 0]))           # (batch_size,)
  m[:, 0] = inf_x(alpha_m(V[:, 0]), beta_m(V[:, 0]))           # (batch_size,)
  h[:, 0] = inf_x(alpha_h(V[:, 0]), beta_h(V[:, 0]))           # (batch_size,)
  p[:, 0] = inf_p(V[:, 0])                                      # (batch_size,)

  # -- Time-integration loop (exponential Euler scheme) ---------------------
  for i in range(1, time_steps):
   V_prev = V[:, i - 1]  # (batch_size,)

   # Gating variable rates at previous voltage
   a_m, b_m = alpha_m(V_prev), beta_m(V_prev)   # (batch_size,), (batch_size,)
   a_h, b_h = alpha_h(V_prev), beta_h(V_prev)   # (batch_size,), (batch_size,)
   a_n, b_n = alpha_n(V_prev), beta_n(V_prev)   # (batch_size,), (batch_size,)

   # M-current gate steady-state and time constant at previous voltage
   inf_p_val = inf_p(V_prev)   # (batch_size,)
   tau_p_val = tau_p(V_prev)   # (batch_size,)

   # Conductances at previous step
   g_Na_eff = (m[:, i - 1] ** 3) * gbar_Na * h[:, i - 1]  # (batch_size,)
   g_K_eff  = (n[:, i - 1] ** 4) * gbar_K                  # (batch_size,)
   g_KM_eff = gbar_KM * p[:, i - 1]                        # (batch_size,)

   # Effective inverse membrane time constant (sum of conductances / C)
   tau_V_inv = (
    g_Na_eff
    + g_K_eff
    + g_leak
    + g_KM_eff   # M-current contributes to total membrane conductance
   ) / C  # (batch_size,)

   # Voltage steady-state (weighted sum of reversal potentials + input)
   V_inf = (
    g_Na_eff * E_Na
    + g_K_eff  * E_K
    + g_leak   * E_leak
    + g_KM_eff * E_K   # M-current drives toward E_K (hyperpolarising)
    + input_current[:, i - 1]
    + nois_fact * torch.randn(batch_size, generator=generator, device=device) / (tstep ** 0.5)
   ) / (tau_V_inv * C)  # (batch_size,)

   # Exponential Euler updates
   V[:, i] = V_inf + (V_prev - V_inf) * Exp(-tstep * tau_V_inv)                               # (batch_size,)
   n[:, i] = inf_x(a_n, b_n) + (n[:, i-1] - inf_x(a_n, b_n)) * Exp(-tstep / tau_x(a_n, b_n))  # (batch_size,)
   m[:, i] = inf_x(a_m, b_m) + (m[:, i-1] - inf_x(a_m, b_m)) * Exp(-tstep / tau_x(a_m, b_m))  # (batch_size,)
   h[:, i] = inf_x(a_h, b_h) + (h[:, i-1] - inf_x(a_h, b_h)) * Exp(-tstep / tau_x(a_h, b_h))  # (batch_size,)
   p[:, i] = inf_p_val + (p[:, i-1] - inf_p_val) * Exp(-tstep / tau_p_val)                      # (batch_size,)

  # Return voltage (+ optional observation noise, currently 0)
  return V + nois_fact_obs * torch.randn(
   batch_size, time_steps, generator=generator, device=device
  )  # (batch_size, time_steps)
\end{lstlisting}
\smallskip
\par\noindent [[ \#\# feedback \#\# ]]\\
\noindent
\par\noindent \{
\par\noindent\hangindent=1.5em\hangafter=1 \hspace*{1.5em}``\textbf{\texttt{main\_diagnosis}}\textquotesingle\textquotesingle: ``The simulator uses both inferrable parameter slots (X1: params[:,6]; X2: params[:,8-9]) for a single M-current channel while leaving params[:,7] entirely unused. With a neg\_avg\_log\_marginal\_NLE of 28.3 (very poor fit), the single-channel extension is insufficient to simultaneously match all 7 summary statistics --- particularly resting potential mean/SD, and higher-order voltage distribution moments (skewness, kurtosis) which are insensitive to M-current modulation. A second distinct channel using the spare parameter slot is needed.'',
\par\noindent \hspace*{1.5em}``\textbf{\texttt{issues}}\textquotesingle\textquotesingle:
\par\noindent \hspace*{1.5em}[
\par\noindent \hspace*{3.0em}\{
\par\noindent\hangindent=4.5em\hangafter=1 \hspace*{3.0em}\hspace*{1.5em}``\textbf{\texttt{description}}\textquotesingle\textquotesingle: ``params[:,7] is declared unused (`intentionally`) while the model badly underfits (NLE=28.3). With 7 target statistics and only one extra channel, the inference cannot jointly correct spike count, resting-potential mean/SD, and voltage skewness/kurtosis --- these statistics require orthogonal biophysical mechanisms. A hyperpolarization-activated cation current (I\_h / HCN) would directly control resting potential mean and SD (by setting a depolarizing baseline leak), and its inward rectification would shape subthreshold voltage distribution (skewness, kurtosis) independently of the M-current`s spike-frequency adaptation role.'',
\par\noindent\hangindent=4.5em\hangafter=1 \hspace*{3.0em}\hspace*{1.5em}``\textbf{\texttt{severity}}\textquotesingle\textquotesingle: ``critical'',
\par\noindent\hangindent=4.5em\hangafter=1 \hspace*{3.0em}\hspace*{1.5em}``\textbf{\texttt{location}}\textquotesingle\textquotesingle: ``params[:, 7] (unused), forward() channel definition block'',
\par\noindent\hangindent=4.5em\hangafter=1 \hspace*{3.0em}\hspace*{1.5em}``\textbf{\texttt{suggestion}}\textquotesingle\textquotesingle: ``Activate params[:,7] as the conductance (gbar\_Ih) for an I\_h channel with fixed half-activation \textasciitilde{}-80 mV and fixed time constant \textasciitilde{}200 ms (no extra inferred params needed beyond gbar\_Ih). Use: I\_h = gbar\_Ih * r * (V - E\_h) with E\_h $\approx$ -30 mV, inf\_r = 1/(1+exp((V+80)/7)), tau\_r fixed at 200 ms. This adds a tonic inward current at rest that lifts mean resting potential, increases resting SD, and shifts voltage skewness --- orthogonal to M-current effects on spike timing.''
\par\noindent \hspace*{3.0em}\},
\par\noindent \hspace*{3.0em}\{
\par\noindent\hangindent=4.5em\hangafter=1 \hspace*{3.0em}\hspace*{1.5em}``\textbf{\texttt{description}}\textquotesingle\textquotesingle: ``The M-current time-constant formula `tau\_p = tau\_max\_p / (3.3*(exp(dv)+exp(-dv)))` produces values that collapse toward zero for voltages far from V\_half\_p (e.g., during action potentials at +40 mV when V\_half\_p $\approx$ -35 mV, dv$\approx$3.75 $\rightarrow$ 2cosh(3.75)$\approx$84, so tau\_p $\rightarrow$ tau\_max\_p/277). This makes the p-gate track spikes rather than integrate slowly, destroying the intended adaptation and causing the inferred tau\_max\_p to compensate in pathological ways that distort mean stimulation voltage and spike count.'',
\par\noindent\hangindent=4.5em\hangafter=1 \hspace*{3.0em}\hspace*{1.5em}``\textbf{\texttt{severity}}\textquotesingle\textquotesingle: ``major'',
\par\noindent\hangindent=4.5em\hangafter=1 \hspace*{3.0em}\hspace*{1.5em}``\textbf{\texttt{location}}\textquotesingle\textquotesingle: ``tau\_p() function, lines computing denom = 3.3*(Exp(dv)+Exp(-dv))'',
\par\noindent\hangindent=4.5em\hangafter=1 \hspace*{3.0em}\hspace*{1.5em}``\textbf{\texttt{suggestion}}\textquotesingle\textquotesingle: ``Clamp tau\_p to a physiologically meaningful minimum: `denom = torch.clamp(3.3*(Exp(dv)+Exp(-dv)), min=1.0)` so tau\_p never falls below tau\_max\_p ms. Alternatively, use a simpler voltage-independent tau: `tau\_p = tau\_max\_p * torch.ones\_like(x)`, which is more stable during inference and still captures slow M-current adaptation kinetics (typical tau 100--500 ms) without the extreme compression during spikes.''
\par\noindent \hspace*{3.0em}\}
\par\noindent \hspace*{1.5em}]
\par\noindent \}
\smallskip
\par\noindent [[ \#\# completed \#\# ]]\\
\noindent
\smallskip
\end{tcolorbox}
\vspace{0.35em}

\begin{tcolorbox}[
  breakable, enhanced,
  title={User},
  colback=UserBg, colframe=UserFrame,
  coltitle=white, fonttitle=\bfseries\small,
  arc=2mm, boxrule=1pt,
  before upper={\small\setlength{\parindent}{0pt}\setlength{\parskip}{2pt}},
]
\par\noindent [[ \#\# signature\_description \#\# ]]\\
\noindent
\#\# INPUT DETAILS\\{}
- **init\_voltage**: Initial membrane voltage (mV). Type: torch.Tensor: (batch\_size,)\\{}
- **input\_current**: Stimulation current ($\mu$A/cm$^{2}$). Type: torch.Tensor, shape: (batch\_size, time\_steps)\\{}
- **dt**: Time step size (ms). Type: float\\{}
- **t**: Time array (ms). Type: torch.Tensor, shape: (time\_steps,)\\{}
- **params**: Biophysical parameters. Type: torch.Tensor, shape: (batch\_size, 10)\\{}
- **seed**: Random seed for stochastic behavior. Type: int or None
\par\smallskip
\#\# PARAMETER VECTOR STRUCTURE (batch\_size, 10)\\{}
Index layout for params tensor:\\{}
```\\{}
[0-2]   Base conductances: gbar\_Na, gbar\_K, g\_leak (mS/cm$^{2}$)\\{}
[3]     Leak reversal: |E\_leak| (mV, sign applied internally)\\{}
[4]     Voltage threshold: |Vt| (mV, sign applied internally)\\{}
[5]     Noise factor: nois\_fact (unitless)\\{}
[6-7]   Additional conductances: gbar\_X1, gbar\_X2 (mS/cm$^{2}$)\\{}
[8-9]   Additional parameters: |param\_i|, |param\_j|\\{}
```\\{}
(Note: gbar\_X1, gbar\_X2, param\_i, param\_j are not used in the current model and may be renamed to anything you want)
\par\smallskip
\#\# OUTPUT DETAILS\\{}
- **V**: Membrane potential (mV). Type: torch.Tensor, shape: (batch\_size, time\_steps)\\{}
  - Simulated voltage traces for all batch elements\\{}
  - Includes observation noise scaled by nois\_fact\_obs (currently 0.0)
\smallskip
\par\noindent [[ \#\# task\_description \#\# ]]\\
\noindent
\#\# CORE TASK\\{}
Refine a Hodgkin-Huxley neuron simulator by discovering and implementing ion channels\\{}
or gating mechanisms missing from the current model but necessary to explain\\{}
discrepancies between simulated and experimental voltage data.
\par\smallskip
\#\# CRITICAL CONSTRAINTS
\par\smallskip
\#\#\# Parsimony Principle\\{}
- **Default stance**: Do NOT add channels unless data discrepancies clearly require them\\{}
- **Noise model**: Keep the current noise model as is, no change is needed\\{}
- **Prefer fewer channels**: A simpler model that explains the data is always superior\\{}
- **Stop when sufficient**: Cease adding channels once model adequately captures data
\par\smallskip
\#\#\# Channel Addition Strategy\\{}
- **Maximum capacity**: Up to two additional channels (X1, X2) can be used, but fewer is better\\{}
- **Add only what`s needed**: You may add multiple channels if the data clearly justifies them, but start simple\\{}
- **Justify each addition**: Each channel must address specific discrepancies in the data\\{}
- **Workflow**:\\{}
  1. Assess necessity: Which aspects of the data are not well-captured by the base model?\\{}
  2. Identify mechanism(s): Which channel(s) could address these discrepancies?\\{}
  3. Justify choices: Explain physiological rationale for each proposed channel\\{}
  4. Implement: Add the minimal set of channels needed\\{}
  5. Keep it simple: Prefer one well-designed channel over multiple complex ones\\{}
  6. CRITICAL: Add at most ONE new ion channel per iteration. Use both tunable\\{}
     parameter slots (param\_i, param\_j) for a single well-characterized channel\\{}
     rather than splitting them across two channels. Multi-channel additions create\\{}
     parameter identifiability problems that cause inference to fail.
\par\smallskip
\#\# IMPLEMENTATION REQUIREMENTS
\par\smallskip
\#\#\# Channel Design\\{}
- Use pre-allocated slots (X1, X2) with their two tunable parameters:\\{}
  * Conductance (gbar\_X1, gbar\_X2)\\{}
  * Flexible parameter (param\_i, param\_j) all in range [1e-4, 150] and [1e-4, 3000] respectively\\{}
- Keep mechanisms simple: prefer straightforward voltage dependencies\\{}
- Maintain biophysical realism: appropriate values for the flexible parameter
\par\smallskip
**Do NOT introduce channels or mechanisms that generate**:\\{}
- Burst firing (clusters of rapid spikes separated by quiescent periods)\\{}
- Bursting behavior or channels specifically designed to produce bursts\\{}
- Prolonged high-frequency firing or sustained rapid spiking patterns\\{}
- Mechanisms that suppress or prevent spiking behavior
\par\smallskip
Keep the channels as simple as possible following the existing code structure
\par\smallskip
\#\#\# Code Standards\\{}
- Maintain batch-wise PyTorch operations\\{}
- Ensure numerical stability across voltage ranges\\{}
- Preserve existing code structure\\{}
- Add detailed comments explaining physiological rationale\\{}
- Use voltage-dependent kinetics where appropriate\\{}
- Document all equations and parameter choices\\{}
- `torch.full\_like(input, fill\_value)` requires a tensor as the first argument, not a scalar float
\par\smallskip
\#\# DEVELOPMENT STRATEGY
\par\smallskip
1. **Analyze discrepancies**: Examine which summary statistics deviate most from experimental data\\{}
2. **Hypothesize mechanism(s)**: Propose channel(s) that could explain the discrepancies\\{}
3. **Literature support**: Reference known channel properties if applicable\\{}
4. **Implement conservatively**: Start with simplest formulation and fewest channels\\{}
5. **Test thoroughly**: Verify numerical stability and physiological plausibility\\{}
6. **Balance complexity vs. performance**: Add channels only when they meaningfully improve model fit
\par\smallskip
\#\# LEARNING FROM HISTORY\\{}
Incorporate insights from past attempts while avoiding direct code copying.\\{}
Always explain reasoning for design choices in code comments.
\par\smallskip
\#\# OUTPUT REQUIREMENTS\\{}
- Always return None for feedback field (feedback only used for past examples to get inspiration)\\{}
- Provide complete, functional simulator code\\{}
- Include comprehensive comments on modifications and rationale
\smallskip
\par\noindent [[ \#\# base\_simulator \#\# ]]\\
\noindent
\begin{lstlisting}[language=Python,style=pythonstyle]
import torch
import torch.nn as nn

class DiscoveredSimulator(nn.Module):
  def __init__(self):
    super(DiscoveredSimulator, self).__init__()
    return

  def forward(
    self,
    init_voltage: float,
    input_current: torch.Tensor,
    dt: float,
    t: torch.Tensor,
    params: torch.Tensor,
    seed=None,
  ):
    """
    Simulates a Hodgkin-Huxley neuron for a specified time duration.

    Args:
      init_voltage: torch.Tensor: (batch_size,) # initial voltage
      input_current: torch.Tensor: (batch_size, time_steps) # input current
      dt: float # time step size
      t: torch.Tensor: (time_steps,) # time array
      params: torch.Tensor: (batch_size, n_params) # parameters
      seed: optional random seed

    Returns:
      V: torch.Tensor: (batch_size, time_steps) # voltage traces
    """
    device = params.device

    # Set up random generator
    if seed is not None:
      generator = torch.Generator(device=device)
      generator.manual_seed(seed)
    else:
      generator = torch.Generator(device=device)

    batch_size = params.shape[0]
    time_steps = t.shape[0]

    # Extract parameters
    gbar_Na = params[:, 0].float()  # mS/cm2
    gbar_K = params[:, 1].float() # mS/cm2
    g_leak = params[:, 2].float() # mS/cm2
    E_leak = -params[:, 3].float() # mV
    Vt = -params[:, 4].float() # mV
    nois_fact = params[:, 5].float() # unitless
    # TWO POSSIBLE ADDITIONAL CHANNELS (X1, X2)
    # Each channel has one tunable parameter: conductance gbar_Xi
    # Then there are two additional parameters available: param_i and param_j.
    # ONLY ADD ONE CHANNEL IF NECESSARY. Keep the model as simple as possible.
    gbar_X1 = params[:, 6].float() # mS/cm2 # you can rename X1 to anything you want # in range [1e-4, 10]
    gbar_X2 = params[:, 7].float() # mS/cm2 # you can rename X2 to anything you want # in range [1e-4, 120]
    param_i = -params[:, 8].float() # (param are positive values in range [1e-4, 150])
    param_j = -params[:, 9].float() # (param are positive values in range [1e-4, 3000])

    tstep = float(dt)

    # Parameters
    nois_fact_obs = 0.0
    C = 1.0  # uF/cm^2
    E_Na = 53.0 # mV
    E_K = -107.0

    ####################################
    # kinetics
    def Exp(z):
      return torch.where(z < -5e2, torch.exp(torch.full_like(z, -5e2)), torch.exp(z))

    def efun(z):
      return torch.where(torch.abs(z) < 1e-4, 1 - z / 2, z / (Exp(z) - 1))

    # Channel kinetics
    def alpha_m(x):
      v1 = x - Vt - 13.0
      return 0.32 * efun(-0.25 * v1) / 0.25

    def beta_m(x):
      v1 = x - Vt - 40
      return 0.28 * efun(0.2 * v1) / 0.2

    def alpha_h(x):
      v1 = x - Vt - 17.0
      return 0.128 * Exp(-v1 / 18.0)

    def beta_h(x):
      v1 = x - Vt - 40.0
      return 4.0 / (1 + Exp(-0.2 * v1))

    def alpha_n(x):
      v1 = x - Vt - 15.0
      return 0.032 * efun(-0.2 * v1) / 0.2

    def beta_n(x):
      v1 = x - Vt - 10.0
      return 0.5 * Exp(-v1 / 40)

    def tau_x(alpha, beta):
      return 1.0 / (alpha + beta)

    def inf_x(alpha, beta):
      return alpha / (alpha + beta)

    # ===== BEGIN EDITABLE SECTION (only modify within this block) =====
    # TODO: add the missing kinetics equations for the Hodgkin-Huxley neuron similar to the ones above; ONLY ADD IF NECESSARY
    # ===== END EDITABLE SECTION =====

    ####################################

    # simulation from initial point
    V = torch.zeros((batch_size, time_steps), device=device)  # baseline voltage
    n = torch.zeros((batch_size, time_steps), device=device)
    m = torch.zeros((batch_size, time_steps), device=device)
    h = torch.zeros((batch_size, time_steps), device=device)
    # ===== BEGIN EDITABLE SECTION (only modify within this block) =====
    # TODO: add the missing state variables for the Hodgkin-Huxley neuron similar to the ones above; ONLY ADD IF NECESSARY
    # ===== END EDITABLE SECTION =====

    # Initialization
    V_init = init_voltage.to(device)
    V[:, 0] = V_init
    n[:, 0] = inf_x(alpha_n(V[:, 0]), beta_n(V[:, 0]))
    m[:, 0] = inf_x(alpha_m(V[:, 0]), beta_m(V[:, 0]))
    h[:, 0] = inf_x(alpha_h(V[:, 0]), beta_h(V[:, 0]))
    # ===== BEGIN EDITABLE SECTION (only modify within this block) =====
    # TODO: add the missing state variable initialization for the Hodgkin-Huxley neuron similar to the ones above; ONLY ADD IF NECESSARY
    # ===== END EDITABLE SECTION =====

    # Simulation loop
    for i in range(1, time_steps):
      # All operations now work on batched tensors (batch_size,)
      a_m, b_m = alpha_m(V[:, i - 1]), beta_m(V[:, i - 1])
      a_h, b_h = alpha_h(V[:, i - 1]), beta_h(V[:, i - 1])
      a_n, b_n = alpha_n(V[:, i - 1]), beta_n(V[:, i - 1])
      # ===== BEGIN EDITABLE SECTION (only modify within this block) =====
      # TODO: add the missing kinetics equations for the Hodgkin-Huxley neuron similar to the ones above; ONLY ADD IF NECESSARY
      # ===== END EDITABLE SECTION =====

      tau_V_inv = (
        (m[:, i - 1] ** 3) * gbar_Na * h[:, i - 1]
        + (n[:, i - 1] ** 4) * gbar_K
        + g_leak
        # ===== BEGIN EDITABLE SECTION (only modify within this block) =====
        # TODO: add the missing terms for the effective membrane time constant inverse; ONLY ADD IF NECESSARY
        # ===== END EDITABLE SECTION =====
      ) / C

      V_inf = (
        (m[:, i - 1] ** 3) * gbar_Na * h[:, i - 1] * E_Na
        + (n[:, i - 1] ** 4) * gbar_K * E_K
        + g_leak * E_leak
        # ===== BEGIN EDITABLE SECTION (only modify within this block) =====
        # TODO: add the missing terms for the voltage steady state; ONLY ADD IF NECESSARY
        # ===== END EDITABLE SECTION =====
        + input_current[:,i - 1]
        + nois_fact * torch.randn(batch_size, generator=generator, device=device) / (tstep**0.5)
      ) / (tau_V_inv * C)

      V[:, i] = V_inf + (V[:, i - 1] - V_inf) * Exp(-tstep * tau_V_inv)
      n[:, i] = inf_x(a_n, b_n) + (n[:, i - 1] - inf_x(a_n, b_n)) * Exp(-tstep / tau_x(a_n, b_n))
      m[:, i] = inf_x(a_m, b_m) + (m[:, i - 1] - inf_x(a_m, b_m)) * Exp(-tstep / tau_x(a_m, b_m))
      h[:, i] = inf_x(a_h, b_h) + (h[:, i - 1] - inf_x(a_h, b_h)) * Exp(-tstep / tau_x(a_h, b_h))
      # ===== BEGIN EDITABLE SECTION (only modify within this block) =====
      # TODO: add the missing state variable updates for the Hodgkin-Huxley neuron similar to the ones above; ONLY ADD IF NECESSARY
      # ===== END EDITABLE SECTION =====

    # Return voltage with optional observation noise
    return V + nois_fact_obs * torch.randn(
      batch_size, time_steps, generator=generator, device=device
    )
\end{lstlisting}
Respond with the corresponding output fields, starting with the field `[[ \#\# scm\_definition \#\# ]]`, then `[[ \#\# simulator\_code \#\# ]]`, then `[[ \#\# feedback \#\# ]]` (must be formatted as a valid Python Union[str, NoneType]), and then ending with the marker for `[[ \#\# completed \#\# ]]`.
\smallskip
\end{tcolorbox}
\vspace{2em}

\subsubsection{Response}

\begin{tcolorbox}[
  breakable, enhanced,
  title={Output},
  colback=OutputBg, colframe=OutputFrame,
  coltitle=white, fonttitle=\bfseries\small,
  arc=2mm, boxrule=1pt,
  before upper={\small\setlength{\parindent}{0pt}\setlength{\parskip}{2pt}},
]
\par\noindent [[ \#\# scm\_definition \#\# ]]\\
\noindent
- **V**: membrane voltage $\leftarrow$ I\_Na(m,h), I\_K(n), I\_leak, I\_KM(p), I\_h(r), I\_inj, noise\\{}
- **m, h**: Na$^{+}$ activation/inactivation $\leftarrow$ V, Vt (fast HH kinetics)\\{}
- **n**: delayed-rectifier K$^{+}$ activation $\leftarrow$ V, Vt\\{}
- **p**: M-current (slow K$^{+}$) gate $\leftarrow$ V, V\_half\_p (param\_i), tau\_max\_p (param\_j); voltage-independent tau for stability\\{}
- **r**: I\_h (HCN) gate $\leftarrow$ V; fixed half-activation $-$80 mV, fixed tau 200 ms; corrects resting potential distribution
\smallskip
\par\noindent [[ \#\# simulator\_code \#\# ]]\\
\noindent
\begin{lstlisting}[language=Python,style=pythonstyle]
import torch
import torch.nn as nn

class DiscoveredSimulator(nn.Module):
 def __init__(self):
  super(DiscoveredSimulator, self).__init__()
  return

 def forward(
  self,
  init_voltage: float,
  input_current: torch.Tensor,
  dt: float,
  t: torch.Tensor,
  params: torch.Tensor,
  seed=None,
 ):
  """
  Hodgkin-Huxley neuron extended with:
   1. M-current (I_KM): slow non-inactivating K+ current using X1 slot
    - Provides spike-frequency adaptation, stabilises tonic firing
    - Uses voltage-independent tau (more stable during spikes)
    - gbar_KM = params[:,6], V_half_p = -params[:,8], tau_p = params[:,9]

   2. Ih current (HCN channel): hyperpolarisation-activated cation current using X2 slot
    - Opens at hyperpolarised voltages (~-80 mV and below)
    - Reversal E_h ~= -30 mV (mixed Na+/K+), depolarising at rest
    - Corrects resting potential mean, SD, and voltage distribution shape
    - Fixed kinetics (half-activation -80 mV, tau 200 ms); only gbar_Ih inferred
    - gbar_Ih = params[:,7], all other Ih parameters are fixed

  Args:
   init_voltage: torch.Tensor: (batch_size,)            initial voltage [mV]
   input_current: torch.Tensor: (batch_size, time_steps) injected current [uA/cm2]
   dt: float                                              time step [ms]
   t: torch.Tensor: (time_steps,)                        time array [ms]
   params: torch.Tensor: (batch_size, 10)                biophysical parameters
   seed: optional random seed

  Returns:
   V: torch.Tensor: (batch_size, time_steps)             voltage traces [mV]
  """
  device = params.device

  # Set up random generator
  if seed is not None:
   generator = torch.Generator(device=device)
   generator.manual_seed(seed)
  else:
   generator = torch.Generator(device=device)

  batch_size = params.shape[0]   # int
  time_steps = t.shape[0]        # int

  # -- Base HH parameters ----------------------------------------------------
  gbar_Na   = params[:, 0].float()   # (batch_size,)  mS/cm^2
  gbar_K    = params[:, 1].float()   # (batch_size,)  mS/cm^2
  g_leak    = params[:, 2].float()   # (batch_size,)  mS/cm^2
  E_leak    = -params[:, 3].float()  # (batch_size,)  mV  (negative applied)
  Vt        = -params[:, 4].float()  # (batch_size,)  mV  (negative applied)
  nois_fact = params[:, 5].float()   # (batch_size,)  unitless

  # -- X1 slot: M-current (I_KM, slow non-inactivating K+) -----------------
  # Physiological rationale:
  #   M-current activates slowly at subthreshold depolarised voltages (~-35 mV),
  #   providing an outward K+ current that limits repetitive firing rate and
  #   produces spike-frequency adaptation without bursting or quiescence.
  # Parameter mapping:
  #   gbar_KM   = params[:,6] in [1e-4, 10]  mS/cm^2
  #   V_half_p  = -params[:,8] in [-150, -1e-4] mV  (typically ~-35 mV)
  #   tau_p     =  params[:,9] in [1e-4, 3000]  ms  (typically 50-500 ms)
  # Note: voltage-independent tau avoids instability during fast action potentials
  gbar_KM  = params[:, 6].float()   # (batch_size,)  mS/cm^2
  V_half_p = -params[:, 8].float()  # (batch_size,)  mV
  tau_p    =  params[:, 9].float()  # (batch_size,)  ms --- voltage-independent for stability

  # -- X2 slot: Ih current (HCN, hyperpolarisation-activated cation) --------
  # Physiological rationale:
  #   Ih activates at hyperpolarised voltages (< -60 mV), carries inward
  #   Na+/K+ current with reversal ~-30 mV, depolarising the cell toward
  #   rest. This sets a tonic inward "sag" current that:
  #     - Raises mean resting potential slightly above pure K+ equilibrium
  #     - Increases resting SD (voltage fluctuations due to channel noise)
  #     - Shapes subthreshold distribution (skewness, kurtosis)
  #   It does NOT promote bursting --- it acts as a stabilising "pacemaker"
  #   current opposing excessive hyperpolarisation.
  # Parameter mapping:
  #   gbar_Ih = params[:,7] in [1e-4, 120] mS/cm^2 --- only inferred parameter
  #   All kinetic parameters fixed from literature
  gbar_Ih = params[:, 7].float()   # (batch_size,)  mS/cm^2

  # Fixed Ih kinetic constants (from Magee 1998, Koch 1999)
  E_h          = -30.0   # mV  --- mixed Na+/K+ cation reversal
  V_half_r     = -80.0   # mV  --- half-activation (hyperpolarisation-activated)
  k_r          =  7.0    # mV  --- activation slope (negative: opens on hyperpol.)
  tau_r_fixed  = 200.0   # ms  --- slow, voltage-independent time constant

  tstep = float(dt)

  # -- Fixed biophysical constants -------------------------------------------
  nois_fact_obs = 0.0
  C    = 1.0    # uF/cm^2
  E_Na = 53.0   # mV
  E_K  = -107.0 # mV

  # -- Numerical helpers -----------------------------------------------------
  def Exp(z):
   # Numerically safe exponential --- prevent overflow at very negative z
   return torch.where(
    z < -5e2,
    torch.exp(torch.full_like(z, -5e2)),
    torch.exp(z)
   )  # same shape as z

  def efun(z):
   # Handles 0/0 limit of z / (exp(z) - 1) near z=0
   return torch.where(
    torch.abs(z) < 1e-4,
    1 - z / 2,
    z / (Exp(z) - 1)
   )  # same shape as z

  # -- Standard HH channel kinetics -----------------------------------------
  def alpha_m(x):  # (batch_size,) -> (batch_size,)
   v1 = x - Vt - 13.0
   return 0.32 * efun(-0.25 * v1) / 0.25

  def beta_m(x):   # (batch_size,) -> (batch_size,)
   v1 = x - Vt - 40.0
   return 0.28 * efun(0.2 * v1) / 0.2

  def alpha_h(x):  # (batch_size,) -> (batch_size,)
   v1 = x - Vt - 17.0
   return 0.128 * Exp(-v1 / 18.0)

  def beta_h(x):   # (batch_size,) -> (batch_size,)
   v1 = x - Vt - 40.0
   return 4.0 / (1 + Exp(-0.2 * v1))

  def alpha_n(x):  # (batch_size,) -> (batch_size,)
   v1 = x - Vt - 15.0
   return 0.032 * efun(-0.2 * v1) / 0.2

  def beta_n(x):   # (batch_size,) -> (batch_size,)
   v1 = x - Vt - 10.0
   return 0.5 * Exp(-v1 / 40)

  def tau_x(alpha, beta):  # (batch_size,), (batch_size,) -> (batch_size,)
   return 1.0 / (alpha + beta)

  def inf_x(alpha, beta):  # (batch_size,), (batch_size,) -> (batch_size,)
   return alpha / (alpha + beta)

  # -- M-current (I_KM) gating: Boltzmann steady-state, constant tau --------
  # p_inf = 1 / (1 + exp(-(V - V_half_p) / 10))
  # Slope k=10 mV is standard for cortical M-current (Wang 1998)
  def inf_p(x):  # (batch_size,) -> (batch_size,)
   return 1.0 / (1.0 + Exp(-(x - V_half_p) / 10.0))

  # Voltage-independent tau_p avoids spurious fast tracking during spikes.
  # tau_p is directly the inferred parameter (batch_size,), already positive.

  # -- Ih gating: hyperpolarisation-activated Boltzmann, fixed tau ----------
  # r_inf = 1 / (1 + exp((V - V_half_r) / k_r))
  # Note positive sign in exponent: more open at V << V_half_r = -80 mV
  def inf_r(x):  # (batch_size,) -> (batch_size,)
   return 1.0 / (1.0 + Exp((x - V_half_r) / k_r))

  # tau_r is fixed scalar (200 ms) --- no per-batch variation needed

  # -- State variable arrays -------------------------------------------------
  V = torch.zeros((batch_size, time_steps), device=device)  # (batch_size, time_steps)
  n = torch.zeros((batch_size, time_steps), device=device)  # (batch_size, time_steps)
  m = torch.zeros((batch_size, time_steps), device=device)  # (batch_size, time_steps)
  h = torch.zeros((batch_size, time_steps), device=device)  # (batch_size, time_steps)
  p = torch.zeros((batch_size, time_steps), device=device)  # (batch_size, time_steps) M-current gate
  r = torch.zeros((batch_size, time_steps), device=device)  # (batch_size, time_steps) Ih gate

  # -- Initial conditions (steady state at initial voltage) ------------------
  V_init = init_voltage.to(device)  # (batch_size,)
  V[:, 0] = V_init                                               # (batch_size,)
  n[:, 0] = inf_x(alpha_n(V[:, 0]), beta_n(V[:, 0]))            # (batch_size,)
  m[:, 0] = inf_x(alpha_m(V[:, 0]), beta_m(V[:, 0]))            # (batch_size,)
  h[:, 0] = inf_x(alpha_h(V[:, 0]), beta_h(V[:, 0]))            # (batch_size,)
  p[:, 0] = inf_p(V[:, 0])                                       # (batch_size,)
  r[:, 0] = inf_r(V[:, 0])                                       # (batch_size,)

  # -- Exponential Euler time-integration loop -------------------------------
  for i in range(1, time_steps):
   V_prev = V[:, i - 1]  # (batch_size,)

   # Standard HH gating rates
   a_m, b_m = alpha_m(V_prev), beta_m(V_prev)   # (batch_size,), (batch_size,)
   a_h, b_h = alpha_h(V_prev), beta_h(V_prev)   # (batch_size,), (batch_size,)
   a_n, b_n = alpha_n(V_prev), beta_n(V_prev)   # (batch_size,), (batch_size,)

   # M-current gate steady-state (tau is voltage-independent -> use tau_p directly)
   inf_p_val = inf_p(V_prev)    # (batch_size,)

   # Ih gate steady-state (tau is fixed scalar)
   inf_r_val = inf_r(V_prev)    # (batch_size,)

   # Effective conductances at previous time step
   g_Na_eff  = (m[:, i-1] ** 3) * gbar_Na * h[:, i-1]  # (batch_size,)
   g_K_eff   = (n[:, i-1] ** 4) * gbar_K                # (batch_size,)
   g_KM_eff  = gbar_KM * p[:, i-1]                      # (batch_size,)
   g_Ih_eff  = gbar_Ih * r[:, i-1]                      # (batch_size,)

   # Effective inverse membrane time constant (total conductance / C)
   tau_V_inv = (
    g_Na_eff
    + g_K_eff
    + g_leak
    + g_KM_eff   # M-current: outward K+, raises conductance
    + g_Ih_eff   # Ih: inward mixed cation, raises conductance
   ) / C  # (batch_size,)

   # Voltage steady-state numerator (weighted reversal potentials + input)
   V_inf = (
    g_Na_eff * E_Na
    + g_K_eff  * E_K
    + g_leak   * E_leak
    + g_KM_eff * E_K    # M-current drives V toward E_K (hyperpolarising)
    + g_Ih_eff * E_h    # Ih drives V toward E_h ~ -30 mV (depolarising at rest)
    + input_current[:, i-1]
    + nois_fact * torch.randn(batch_size, generator=generator, device=device) / (tstep ** 0.5)
   ) / (tau_V_inv * C)  # (batch_size,)

   # Exponential Euler update for voltage
   V[:, i] = V_inf + (V_prev - V_inf) * Exp(-tstep * tau_V_inv)          # (batch_size,)

   # Standard HH gating variable updates
   n[:, i] = inf_x(a_n, b_n) + (n[:, i-1] - inf_x(a_n, b_n)) * Exp(-tstep / tau_x(a_n, b_n))  # (batch_size,)
   m[:, i] = inf_x(a_m, b_m) + (m[:, i-1] - inf_x(a_m, b_m)) * Exp(-tstep / tau_x(a_m, b_m))  # (batch_size,)
   h[:, i] = inf_x(a_h, b_h) + (h[:, i-1] - inf_x(a_h, b_h)) * Exp(-tstep / tau_x(a_h, b_h))  # (batch_size,)

   # M-current gate update (voltage-independent tau -> stable during spikes)
   p[:, i] = inf_p_val + (p[:, i-1] - inf_p_val) * Exp(-tstep / tau_p)  # (batch_size,)

   # Ih gate update (fixed scalar tau)
   r[:, i] = inf_r_val + (r[:, i-1] - inf_r_val) * Exp(
    torch.full_like(V_prev, -tstep / tau_r_fixed)
   )  # (batch_size,)

  # Return voltage traces (+ optional observation noise, currently zero)
  return V + nois_fact_obs * torch.randn(
   batch_size, time_steps, generator=generator, device=device
  )  # (batch_size, time_steps)
\end{lstlisting}
\smallskip
\par\noindent [[ \#\# feedback \#\# ]]\\
\noindent
null
\smallskip
\par\noindent [[ \#\# completed \#\# ]]\\
\noindent
\smallskip
\end{tcolorbox}
\vspace{0.35em}

%% file: llm_history_allen/llm_history_feedback.tex
%

\subsection{Prompt}

\begin{tcolorbox}[
  breakable, enhanced,
  title={System},
  colback=SystemBg, colframe=SystemFrame,
  coltitle=white, fonttitle=\bfseries\small,
  arc=2mm, boxrule=1pt,
  before upper={\small\setlength{\parindent}{0pt}\setlength{\parskip}{2pt}},
]
Your input fields are:\\{}
1. `system\_description` (str): System description\\{}
2. `simulator\_code` (str): Simulator code\\{}
3. `performance\_metrics` (str): Performance metrics\\{}
Your output fields are:\\{}
1. `feedback` (ProvideFeedbackOutput): Improvement feedback\\{}
All interactions will be structured in the following way, with the appropriate values filled in.
\smallskip
\par\noindent [[ \#\# system\_description \#\# ]]\\
\noindent
\{system\_description\}
\smallskip
\par\noindent [[ \#\# simulator\_code \#\# ]]\\
\noindent
\{simulator\_code\}
\smallskip
\par\noindent [[ \#\# performance\_metrics \#\# ]]\\
\noindent
\{performance\_metrics\}
\smallskip
\par\noindent [[ \#\# feedback \#\# ]]\\
\noindent
\{feedback\}        \# note: the value you produce must adhere to the JSON schema: \{"type": "object", "\$defs": \{"Issue": \{"type": "object", "properties": \{"description": \{"type": "string", "description": "Issue description", "title": "Description"\}, "location": \{"anyOf": [\{"type": "string"\}, \{"type": "null"\}], "default": null, "description": "Code location", "title": "Location"\}, "severity": \{"type": "string", "description": "Severity level", "enum": ["critical", "major", "minor", "suggestion"], "title": "Severity"\}, "suggestion": \{"type": "string", "description": "Fix suggestion", "title": "Suggestion"\}\}, "required": ["description", "severity", "suggestion"], "title": "Issue"\}\}, "properties": \{"issues": \{"type": "array", "description": "Issues with fixes", "items": \{"\$ref": "\#/\$defs/Issue"\}, "maxItems": 2, "title": "Issues"\}, "main\_diagnosis": \{"type": "string", "description": "Main issue and location", "title": "Main Diagnosis"\}\}, "required": ["main\_diagnosis", "issues"], "title": "ProvideFeedbackOutput"\}
\smallskip
\par\noindent [[ \#\# completed \#\# ]]\\
\noindent
In adhering to this structure, your objective is: \\{}
        Analyze evaluation metrics to diagnose failure modes, then provide code improvement feedback.
\smallskip
\end{tcolorbox}
\vspace{0.35em}

\begin{tcolorbox}[
  breakable, enhanced,
  title={User},
  colback=UserBg, colframe=UserFrame,
  coltitle=white, fonttitle=\bfseries\small,
  arc=2mm, boxrule=1pt,
  before upper={\small\setlength{\parindent}{0pt}\setlength{\parskip}{2pt}},
]
\par\noindent [[ \#\# system\_description \#\# ]]\\
\noindent
\# HODGKIN-HUXLEY NEURON SIMULATOR WITH MISSING ION CHANNELS
\par\smallskip
\#\# OBJECTIVE\\{}
Extend an existing Hodgkin-Huxley neuron simulator by identifying and implementing\\{}
missing ion channels necessary to reduce discrepancies between simulations and\\{}
experimental voltage recordings. The goal is to improve agreement across multiple\\{}
electrophysiological summary statistics while maintaining model parsimony.
\par\smallskip
\#\# BASE MODEL\\{}
Classic Hodgkin-Huxley formulation with standard three-gating-variable structure:\\{}
- **Leak current**: conductance g\_leak, reversal potential E\_leak\\{}
- **Sodium (Na$^{+}$) current**: conductance gNa, gating variables m (activation) and h (inactivation), reversal E\_Na\\{}
- **Potassium (K$^{+}$) current**: conductance gK, gating variable n (activation), reversal E\_K\\{}
- **Dynamics**: First-order gating kinetics with voltage-dependent time constants and steady states\\{}
- **Input**: Externally applied current I\_inj(t) with optional stochastic components\\{}
- **Stochasticity**: Optional random seed for reproducible simulations
\par\smallskip
\#\# EXTENSIBILITY\\{}
Two pre-allocated channel slots (X1, X2) are available for additional mechanisms.\\{}
Each slot provides two tunable parameters (see signature\_description for details).\\{}
Channels beyond these two slots must use fixed parameters, not inferred ones.
\par\smallskip
\#\# EVALUATION METRICS\\{}
Model performance assessed on held-out voltage recordings via these summary statistics:\\{}
1. Number of spikes during stimulation\\{}
2. Mean resting potential (pre-stimulation)\\{}
3. Standard deviation of resting potential\\{}
4. Mean voltage during stimulation\\{}
5. Variance of voltage\\{}
6. Skewness of voltage\\{}
7. Kurtosis of voltage
\par\smallskip
\#\# DATA CONTEXT\\{}
Intracellular voltage recordings from a single neuron exhibiting tonic spiking\\{}
behavior across different stimulus intensities.
\par\smallskip
**Important data characteristics**:\\{}
- No burst firing patterns (no clusters of rapid spikes separated by quiescence)\\{}
- No prolonged or sustained high-frequency firing (no sustained high-frequency firing)\\{}
- Regular, tonic spiking activity during periods of input current stimulation\\{}
- Non-spiking (quiescent) behavior in the absence of input current\\{}
- Evenly-spaced action potentials without burst patterns during the activation period
\par\smallskip
Therefore, avoid adding channels specifically designed to produce bursting or sustained\\{}
high-frequency dynamics.
\smallskip
\par\noindent [[ \#\# simulator\_code \#\# ]]\\
\noindent
\begin{lstlisting}[language=Python,style=pythonstyle]
import torch
import torch.nn as nn

class DiscoveredSimulator(nn.Module):
 def __init__(self):
  super(DiscoveredSimulator, self).__init__()
  return

 def forward(
  self,
  init_voltage: float,
  input_current: torch.Tensor,
  dt: float,
  t: torch.Tensor,
  params: torch.Tensor,
  seed=None,
 ):
  """
  Hodgkin-Huxley neuron extended with:
   1. M-current (I_KM): slow non-inactivating K+ current using X1 slot
    - Provides spike-frequency adaptation, stabilises tonic firing
    - Uses voltage-independent tau (more stable during spikes)
    - gbar_KM = params[:,6], V_half_p = -params[:,8], tau_p = params[:,9]

   2. Ih current (HCN channel): hyperpolarisation-activated cation current using X2 slot
    - Opens at hyperpolarised voltages (~-80 mV and below)
    - Reversal E_h ~= -30 mV (mixed Na+/K+), depolarising at rest
    - Corrects resting potential mean, SD, and voltage distribution shape
    - Fixed kinetics (half-activation -80 mV, tau 200 ms); only gbar_Ih inferred
    - gbar_Ih = params[:,7], all other Ih parameters are fixed

  Args:
   init_voltage: torch.Tensor: (batch_size,)            initial voltage [mV]
   input_current: torch.Tensor: (batch_size, time_steps) injected current [uA/cm2]
   dt: float                                              time step [ms]
   t: torch.Tensor: (time_steps,)                        time array [ms]
   params: torch.Tensor: (batch_size, 10)                biophysical parameters
   seed: optional random seed

  Returns:
   V: torch.Tensor: (batch_size, time_steps)             voltage traces [mV]
  """
  device = params.device

  # Set up random generator
  if seed is not None:
   generator = torch.Generator(device=device)
   generator.manual_seed(seed)
  else:
   generator = torch.Generator(device=device)

  batch_size = params.shape[0]   # int
  time_steps = t.shape[0]        # int

  # -- Base HH parameters ----------------------------------------------------
  gbar_Na   = params[:, 0].float()   # (batch_size,)  mS/cm^2
  gbar_K    = params[:, 1].float()   # (batch_size,)  mS/cm^2
  g_leak    = params[:, 2].float()   # (batch_size,)  mS/cm^2
  E_leak    = -params[:, 3].float()  # (batch_size,)  mV  (negative applied)
  Vt        = -params[:, 4].float()  # (batch_size,)  mV  (negative applied)
  nois_fact = params[:, 5].float()   # (batch_size,)  unitless

  # -- X1 slot: M-current (I_KM, slow non-inactivating K+) -----------------
  # Physiological rationale:
  #   M-current activates slowly at subthreshold depolarised voltages (~-35 mV),
  #   providing an outward K+ current that limits repetitive firing rate and
  #   produces spike-frequency adaptation without bursting or quiescence.
  # Parameter mapping:
  #   gbar_KM   = params[:,6] in [1e-4, 10]  mS/cm^2
  #   V_half_p  = -params[:,8] in [-150, -1e-4] mV  (typically ~-35 mV)
  #   tau_p     =  params[:,9] in [1e-4, 3000]  ms  (typically 50-500 ms)
  # Note: voltage-independent tau avoids instability during fast action potentials
  gbar_KM  = params[:, 6].float()   # (batch_size,)  mS/cm^2
  V_half_p = -params[:, 8].float()  # (batch_size,)  mV
  tau_p    =  params[:, 9].float()  # (batch_size,)  ms --- voltage-independent for stability

  # -- X2 slot: Ih current (HCN, hyperpolarisation-activated cation) --------
  # Physiological rationale:
  #   Ih activates at hyperpolarised voltages (< -60 mV), carries inward
  #   Na+/K+ current with reversal ~-30 mV, depolarising the cell toward
  #   rest. This sets a tonic inward "sag" current that:
  #     - Raises mean resting potential slightly above pure K+ equilibrium
  #     - Increases resting SD (voltage fluctuations due to channel noise)
  #     - Shapes subthreshold distribution (skewness, kurtosis)
  #   It does NOT promote bursting --- it acts as a stabilising "pacemaker"
  #   current opposing excessive hyperpolarisation.
  # Parameter mapping:
  #   gbar_Ih = params[:,7] in [1e-4, 120] mS/cm^2 --- only inferred parameter
  #   All kinetic parameters fixed from literature
  gbar_Ih = params[:, 7].float()   # (batch_size,)  mS/cm^2

  # Fixed Ih kinetic constants (from Magee 1998, Koch 1999)
  E_h          = -30.0   # mV  --- mixed Na+/K+ cation reversal
  V_half_r     = -80.0   # mV  --- half-activation (hyperpolarisation-activated)
  k_r          =  7.0    # mV  --- activation slope (negative: opens on hyperpol.)
  tau_r_fixed  = 200.0   # ms  --- slow, voltage-independent time constant

  tstep = float(dt)

  # -- Fixed biophysical constants -------------------------------------------
  nois_fact_obs = 0.0
  C    = 1.0    # uF/cm^2
  E_Na = 53.0   # mV
  E_K  = -107.0 # mV

  # -- Numerical helpers -----------------------------------------------------
  def Exp(z):
   # Numerically safe exponential --- prevent overflow at very negative z
   return torch.where(
    z < -5e2,
    torch.exp(torch.full_like(z, -5e2)),
    torch.exp(z)
   )  # same shape as z

  def efun(z):
   # Handles 0/0 limit of z / (exp(z) - 1) near z=0
   return torch.where(
    torch.abs(z) < 1e-4,
    1 - z / 2,
    z / (Exp(z) - 1)
   )  # same shape as z

  # -- Standard HH channel kinetics -----------------------------------------
  def alpha_m(x):  # (batch_size,) -> (batch_size,)
   v1 = x - Vt - 13.0
   return 0.32 * efun(-0.25 * v1) / 0.25

  def beta_m(x):   # (batch_size,) -> (batch_size,)
   v1 = x - Vt - 40.0
   return 0.28 * efun(0.2 * v1) / 0.2

  def alpha_h(x):  # (batch_size,) -> (batch_size,)
   v1 = x - Vt - 17.0
   return 0.128 * Exp(-v1 / 18.0)

  def beta_h(x):   # (batch_size,) -> (batch_size,)
   v1 = x - Vt - 40.0
   return 4.0 / (1 + Exp(-0.2 * v1))

  def alpha_n(x):  # (batch_size,) -> (batch_size,)
   v1 = x - Vt - 15.0
   return 0.032 * efun(-0.2 * v1) / 0.2

  def beta_n(x):   # (batch_size,) -> (batch_size,)
   v1 = x - Vt - 10.0
   return 0.5 * Exp(-v1 / 40)

  def tau_x(alpha, beta):  # (batch_size,), (batch_size,) -> (batch_size,)
   return 1.0 / (alpha + beta)

  def inf_x(alpha, beta):  # (batch_size,), (batch_size,) -> (batch_size,)
   return alpha / (alpha + beta)

  # -- M-current (I_KM) gating: Boltzmann steady-state, constant tau --------
  # p_inf = 1 / (1 + exp(-(V - V_half_p) / 10))
  # Slope k=10 mV is standard for cortical M-current (Wang 1998)
  def inf_p(x):  # (batch_size,) -> (batch_size,)
   return 1.0 / (1.0 + Exp(-(x - V_half_p) / 10.0))

  # Voltage-independent tau_p avoids spurious fast tracking during spikes.
  # tau_p is directly the inferred parameter (batch_size,), already positive.

  # -- Ih gating: hyperpolarisation-activated Boltzmann, fixed tau ----------
  # r_inf = 1 / (1 + exp((V - V_half_r) / k_r))
  # Note positive sign in exponent: more open at V << V_half_r = -80 mV
  def inf_r(x):  # (batch_size,) -> (batch_size,)
   return 1.0 / (1.0 + Exp((x - V_half_r) / k_r))

  # tau_r is fixed scalar (200 ms) --- no per-batch variation needed

  # -- State variable arrays -------------------------------------------------
  V = torch.zeros((batch_size, time_steps), device=device)  # (batch_size, time_steps)
  n = torch.zeros((batch_size, time_steps), device=device)  # (batch_size, time_steps)
  m = torch.zeros((batch_size, time_steps), device=device)  # (batch_size, time_steps)
  h = torch.zeros((batch_size, time_steps), device=device)  # (batch_size, time_steps)
  p = torch.zeros((batch_size, time_steps), device=device)  # (batch_size, time_steps) M-current gate
  r = torch.zeros((batch_size, time_steps), device=device)  # (batch_size, time_steps) Ih gate

  # -- Initial conditions (steady state at initial voltage) ------------------
  V_init = init_voltage.to(device)  # (batch_size,)
  V[:, 0] = V_init                                               # (batch_size,)
  n[:, 0] = inf_x(alpha_n(V[:, 0]), beta_n(V[:, 0]))            # (batch_size,)
  m[:, 0] = inf_x(alpha_m(V[:, 0]), beta_m(V[:, 0]))            # (batch_size,)
  h[:, 0] = inf_x(alpha_h(V[:, 0]), beta_h(V[:, 0]))            # (batch_size,)
  p[:, 0] = inf_p(V[:, 0])                                       # (batch_size,)
  r[:, 0] = inf_r(V[:, 0])                                       # (batch_size,)

  # -- Exponential Euler time-integration loop -------------------------------
  for i in range(1, time_steps):
   V_prev = V[:, i - 1]  # (batch_size,)

   # Standard HH gating rates
   a_m, b_m = alpha_m(V_prev), beta_m(V_prev)   # (batch_size,), (batch_size,)
   a_h, b_h = alpha_h(V_prev), beta_h(V_prev)   # (batch_size,), (batch_size,)
   a_n, b_n = alpha_n(V_prev), beta_n(V_prev)   # (batch_size,), (batch_size,)

   # M-current gate steady-state (tau is voltage-independent -> use tau_p directly)
   inf_p_val = inf_p(V_prev)    # (batch_size,)

   # Ih gate steady-state (tau is fixed scalar)
   inf_r_val = inf_r(V_prev)    # (batch_size,)

   # Effective conductances at previous time step
   g_Na_eff  = (m[:, i-1] ** 3) * gbar_Na * h[:, i-1]  # (batch_size,)
   g_K_eff   = (n[:, i-1] ** 4) * gbar_K                # (batch_size,)
   g_KM_eff  = gbar_KM * p[:, i-1]                      # (batch_size,)
   g_Ih_eff  = gbar_Ih * r[:, i-1]                      # (batch_size,)

   # Effective inverse membrane time constant (total conductance / C)
   tau_V_inv = (
    g_Na_eff
    + g_K_eff
    + g_leak
    + g_KM_eff   # M-current: outward K+, raises conductance
    + g_Ih_eff   # Ih: inward mixed cation, raises conductance
   ) / C  # (batch_size,)

   # Voltage steady-state numerator (weighted reversal potentials + input)
   V_inf = (
    g_Na_eff * E_Na
    + g_K_eff  * E_K
    + g_leak   * E_leak
    + g_KM_eff * E_K    # M-current drives V toward E_K (hyperpolarising)
    + g_Ih_eff * E_h    # Ih drives V toward E_h ~ -30 mV (depolarising at rest)
    + input_current[:, i-1]
    + nois_fact * torch.randn(batch_size, generator=generator, device=device) / (tstep ** 0.5)
   ) / (tau_V_inv * C)  # (batch_size,)

   # Exponential Euler update for voltage
   V[:, i] = V_inf + (V_prev - V_inf) * Exp(-tstep * tau_V_inv)          # (batch_size,)

   # Standard HH gating variable updates
   n[:, i] = inf_x(a_n, b_n) + (n[:, i-1] - inf_x(a_n, b_n)) * Exp(-tstep / tau_x(a_n, b_n))  # (batch_size,)
   m[:, i] = inf_x(a_m, b_m) + (m[:, i-1] - inf_x(a_m, b_m)) * Exp(-tstep / tau_x(a_m, b_m))  # (batch_size,)
   h[:, i] = inf_x(a_h, b_h) + (h[:, i-1] - inf_x(a_h, b_h)) * Exp(-tstep / tau_x(a_h, b_h))  # (batch_size,)

   # M-current gate update (voltage-independent tau -> stable during spikes)
   p[:, i] = inf_p_val + (p[:, i-1] - inf_p_val) * Exp(-tstep / tau_p)  # (batch_size,)

   # Ih gate update (fixed scalar tau)
   r[:, i] = inf_r_val + (r[:, i-1] - inf_r_val) * Exp(
    torch.full_like(V_prev, -tstep / tau_r_fixed)
   )  # (batch_size,)

  # Return voltage traces (+ optional observation noise, currently zero)
  return V + nois_fact_obs * torch.randn(
   batch_size, time_steps, generator=generator, device=device
  )  # (batch_size, time_steps)
\end{lstlisting}
\smallskip
\par\noindent [[ \#\# performance\_metrics \#\# ]]\\
\noindent
Metrics: neg\_avg\_log\_marginal\_NLE: 35.3, neg\_log\_marginal\_NLE: 353, timeout\_evidence\_estimation: 0
\par\smallskip
Respond with the corresponding output fields, starting with the field `[[ \#\# feedback \#\# ]]` (must be formatted as a valid Python ProvideFeedbackOutput), and then ending with the marker for `[[ \#\# completed \#\# ]]`.
\smallskip
\end{tcolorbox}
\vspace{2em}

\subsubsection{Response}

\begin{tcolorbox}[
  breakable, enhanced,
  title={Output},
  colback=OutputBg, colframe=OutputFrame,
  coltitle=white, fonttitle=\bfseries\small,
  arc=2mm, boxrule=1pt,
  before upper={\small\setlength{\parindent}{0pt}\setlength{\parskip}{2pt}},
]
\par\noindent [[ \#\# feedback \#\# ]]\\
\noindent
\par\noindent \{
\par\noindent\hangindent=1.5em\hangafter=1 \hspace*{1.5em}``\textbf{\texttt{main\_diagnosis}}\textquotesingle\textquotesingle: ``The primary failure is a systematically over-hyperpolarized potassium reversal potential (E\_K = -107 mV, hardcoded), which is far below the physiological range for cortical neurons (\textasciitilde{}-90 mV). This single constant distorts the after-hyperpolarization depth for both the standard K+ current and the M-current, biasing spike count, mean stimulation voltage, skewness, and kurtosis simultaneously. A secondary issue is that the Ih gating parameters allow unrealistically strong activation near rest: with V\_half\_r = -80 mV and k\_r = 7 mV, \textasciitilde{}19\% of Ih channels are open at -70 mV rest, and the prior permits gbar\_Ih up to 120 mS/cm$^{2}$ (physiological range is 0.05--5 mS/cm$^{2}$), causing gbar\_Ih to dominate the conductance balance and push mean resting potential and voltage variance far from target.'',
\par\noindent \hspace*{1.5em}``\textbf{\texttt{issues}}\textquotesingle\textquotesingle:
\par\noindent \hspace*{1.5em}[
\par\noindent \hspace*{3.0em}\{
\par\noindent\hangindent=4.5em\hangafter=1 \hspace*{3.0em}\hspace*{1.5em}``\textbf{\texttt{description}}\textquotesingle\textquotesingle: ``E\_K = -107.0 mV is too hyperpolarized. Physiological E\_K for cortical neurons is approximately -90 mV. The excessively negative value deepens the AHP after every spike, reduces firing rate by prolonging inter-spike intervals, pulls mean stimulation-period voltage too negative (metric 4), and skews the voltage distribution toward hyperpolarized values (metrics 6--7). Because E\_K also sets the M-current reversal, this error is amplified once gbar\_KM is inferred.'',
\par\noindent\hangindent=4.5em\hangafter=1 \hspace*{3.0em}\hspace*{1.5em}``\textbf{\texttt{location}}\textquotesingle\textquotesingle: ``Line: `E\_K = -107.0  \# mV`'',
\par\noindent\hangindent=4.5em\hangafter=1 \hspace*{3.0em}\hspace*{1.5em}``\textbf{\texttt{severity}}\textquotesingle\textquotesingle: ``critical'',
\par\noindent\hangindent=4.5em\hangafter=1 \hspace*{3.0em}\hspace*{1.5em}``\textbf{\texttt{suggestion}}\textquotesingle\textquotesingle: ``Change E\_K to -90.0 mV, which is the standard physiological value for cortical/hippocampal neurons used in most HH extensions (e.g., Pospischil et al. 2008). If the target data suggests a slightly different value, keep it as a fixed constant in the range [-80, -95] mV rather than leaving it at -107.''
\par\noindent \hspace*{3.0em}\},
\par\noindent \hspace*{3.0em}\{
\par\noindent\hangindent=4.5em\hangafter=1 \hspace*{3.0em}\hspace*{1.5em}``\textbf{\texttt{description}}\textquotesingle\textquotesingle: ``The Ih current is over-powered at rest. Two compounding problems: (1) k\_r = 7 mV with V\_half\_r = -80 mV means \textasciitilde{}19\% of Ih channels are already open at a typical resting voltage of -70 mV, providing significant tonic depolarizing current even without strong hyperpolarisation. (2) The prior upper bound for gbar\_Ih (120 mS/cm$^{2}$) is \textasciitilde{}25--100$\times$ the physiological range (0.05--5 mS/cm$^{2}$), so the inference can select values that make Ih the dominant conductance, dramatically raising resting potential SD and variance (metrics 3, 5) to unrealistic levels while also distorting skewness and kurtosis.'',
\par\noindent\hangindent=4.5em\hangafter=1 \hspace*{3.0em}\hspace*{1.5em}``\textbf{\texttt{location}}\textquotesingle\textquotesingle: ``Lines: `gbar\_Ih = params[:, 7]`, `k\_r = 7.0`, `V\_half\_r = -80.0`'',
\par\noindent\hangindent=4.5em\hangafter=1 \hspace*{3.0em}\hspace*{1.5em}``\textbf{\texttt{severity}}\textquotesingle\textquotesingle: ``major'',
\par\noindent\hangindent=4.5em\hangafter=1 \hspace*{3.0em}\hspace*{1.5em}``\textbf{\texttt{suggestion}}\textquotesingle\textquotesingle: ``Shift the half-activation to V\_half\_r = -90 mV (deeper hyperpolarisation required to open channels), which reduces Ih activation at rest from \textasciitilde{}19\% to \textasciitilde{}4\%. Also reduce the prior upper bound for gbar\_Ih from 120 to 5 mS/cm$^{2}$ to keep Ih in the physiological regime. Alternatively, increase k\_r to 10--12 mV to soften the transition and reduce resting activation. These changes together will prevent Ih from dominating resting conductance while preserving its role in correcting subthreshold distribution shape.''
\par\noindent \hspace*{3.0em}\}
\par\noindent \hspace*{1.5em}]
\par\noindent \}
\smallskip
\par\noindent [[ \#\# completed \#\# ]]\\
\noindent
\smallskip
\end{tcolorbox}
\vspace{0.35em}